\documentclass{article}
\usepackage{amsthm, amsmath, amsfonts, amssymb}
\numberwithin{equation}{section}
\usepackage{mathtools}
\usepackage{mathrsfs}
\usepackage{enumitem}
\usepackage{hyperref}
\hypersetup{
    colorlinks,
    linkcolor={blue!30!black},
    citecolor={blue!50!black},
    urlcolor={blue!80!black}
}

\usepackage{natbib}
\usepackage{color}
\usepackage{xcolor}
\usepackage[margin=3cm]{geometry}
\usepackage{algorithm}
\usepackage{algorithmic}
\usepackage[normalem]{ulem}
\usepackage{ifthen}
\usepackage{subcaption}
\usepackage{pgf}
\usepackage{tikz}
\usepackage{authblk}

\newtheorem{theorem}{Theorem}[section]
\newtheorem{lemma}{Lemma}[section]
\newtheorem{proposition}{Proposition}[section]
\newtheorem{corollary}{Corollary}[section]
\newtheorem{assumption}{Assumption}

\theoremstyle{remark}

\newtheorem{example}{Example}[section]
\newtheorem{remark}{Remark}[section]

\usepackage{xspace}

\newcommand\ie{\emph{i.e.}\xspace}

\newcommand\iid{\ensuremath{\mathit{i.i.d.}}\xspace }

\newcommand{\N}{\mathbb{N}}

\newcommand{\R}{\mathbb{R}}
\newcommand{\rset}{\R}

\newcommand{\Xset}{\mathcal{X}}

\newcommand{\ind}{\mathbf{1}} 
\newcommand{\un}{\ind}
\newcommand{\point}{\,\cdot\,}

\usepackage{xifthen}
\newcommand{\PP}[1][]{\ifthenelse{\equal{#1}{}}{\ensuremath{\mathbb{P}}}{\ensuremath{\mathbb{P}\left( #1 \right)}}}
\newcommand{\EE}[1][]{\ifthenelse{\equal{#1}{}}{\ensuremath{\mathbb E}}{\ensuremath{{\mathbb E}\left[ #1 \right]}}}
\newcommand{\ud}{\,\mathrm{d}} 

\newcommand{\given}[1][{}]{\;\middle\vert\;{#1} }

\newcommand{\tto}{\xrightarrow[t\to\infty]{}}

\newcommand{\law}{\mathcal{L}}
\newcommand{\Var}{\mathrm{Var}}

\newcommand{\sphere}{\mathbb{S}}
\newcommand{\ball}{\mathbb{B}}

\newcommand{\Plim}{P_\infty}

\newcommand{\argmin}{\textup{argmin}}

\newcommand{\pinbl}[1][]{%
  \ifthenelse{\equal{#1}{}}{\ell_\tau}{\ell_{#1}}%
}

\newcommand{\risk}[1][]{
  \ifthenelse{\equal{#1}{}}{\mathcal{R}}{\mathcal{R}_{ #1 }}%
}

\newcommand{\emprisk}[1][]{%
  \ifthenelse{\equal{#1}{}}{\widehat{\mathcal{R}}}{\widehat{\mathcal{R}}_{ #1 }}%
}

\newcommand{\riskPen}[1][]{%
  \ifthenelse{\equal{#1}{}}{\mathcal{R}^{(\lambda)}}{\mathcal{R}^{(\lambda)}_{ #1 }}%
}

\newcommand{\empriskPen}[1][]{%
  \ifthenelse{\equal{#1}{}}{\widehat{\mathcal{R}}^{(\lambda)}}{\widehat{\mathcal{R}}_{#1}^{(\lambda)}}%
}

\newcommand{\riskinter}[1][]{
  \ifthenelse{\equal{#1}{}}{\widetilde{\risk}_k}{\widetilde{\risk}_{ #1 }}%
}

\newcommand{\riskinterPen}[1][]{%
  \ifthenelse{\equal{#1}{}}{\widetilde{\mathcal{R}}^{(\lambda)}}{\widetilde{\mathcal{R}}_{#1}^{(\lambda)}}%
}

\newcommand{\quantiletau}{q}
\newcommand{\quantileinf}{q_{\Plim}}

\newcommand{\rkhs}{\mathcal{H}}
\newcommand{\normrkhs}[1][]{\ensuremath{\lVert#1\rVert_\rkhs}}
\newcommand{\tnk}{t_{n,k}}
\newcommand{\hhlambda}{\widehat{h}_{\lambda,k}}
\newcommand{\htildelambda}{\widetilde{h}_{\lambda,\tnk}}

\newcommand{\fh}{\widehat{f}}
\newcommand{\ERMfamily}{\mathcal{F}}

\newcommand{\condK}{\mathcal{K}}
\newcommand{\Rad}{\mathrm{Rad}}

\newcommand{\representer}{w}

\newcommand{\approxerror}{\mathcal{A}}

\newcommand{\hhat}[1][]{
  \ifthenelse{\equal{#1}{}}{\widehat{h}}{\widehat{h}_{#1}}%
}
\newcommand{\hhatPen}[1][]{
  \ifthenelse{\equal{#1}{}}{\widehat{h}}{\widehat{h}_{#1}^{(\lambda)}}%
}

\newcommand{\ktrain}{k_{\textrm{train}}}

\newcommand{\ktest}{k_{\textrm{test}}}
\newcommand{\ntest}{n_{\textrm{test}}}

\title{Out-of-Distribution generalization of quantile regression with heavy tailed inputs: an SVM approach}
\author[1]{Baptiste Leroux}
\author[2]{Clément Dombry}
\author[3]{Anne Sabourin}
\affil[1]{Université Paris Cité, CNRS, MAP5, F-75006 Paris, France, 
baptiste.leroux@u-paris.fr}
\affil[2]{Université Marie et Louis Pasteur, CNRS, LmB (UMR 6623), F-25000 Besançon, France}
\affil[3]{Université Paris Cité, Université Paris Saclay, ENS Paris Saclay, CNRS, SSA, INSERM, Centre Borelli, F-75006, Paris, France}
\begin{document}

\maketitle

\begin{abstract}

We study quantile regression in an extrapolation regime where the covariate takes unusually large values. Under regular variation assumptions, extreme observations can be effectively characterized through their angular components, enabling learning strategies that focus on the angle of the most extreme observations. This approach is formalized through the minimization of an asymptotic conditional risk that localizes learning in the tail of the covariate distribution.

We propose a novel Support Vector Machine (SVM) framework for extreme quantile regression, leveraging reproducing kernel Hilbert spaces to handle high-dimensional and nonlinear settings. Our method also accommodates unbounded response variables and avoids restrictive transformations. We establish finite-sample learning guarantees under mild regularity assumptions.

The proposed framework unifies ideas from statistical learning and multivariate extremes, providing a tractable and theoretically grounded approach to extrapolation. We complement our theoretical findings with an empirical study on river flow data from the Danube, demonstrating the practical relevance of our methods.



\end{abstract}
\section{Introduction}\label{sec:intro}

Let $(X,Y) \in \R^d \times \R$ be a pair of random variables with unknown distribution $P$. Supervised regression aims at learning a rule from \iid observations $((X_i,Y_i))_{1\leq i\leq n}$ to estimate some target functional, typically the conditional expectation or another summary of the conditional distribution $\law\left(Y \given X=x\right)$. In this context,
Empirical Risk Minimization (ERM) algorithms provide a principled approach to this problem and enjoy well-established theoretical guarantees under combinatorial complexity assumptions, notably through Vapnik–Chervonenkis (VC) theory \citep{devroye2013probabilistic,lugosi2002pattern}. 
Among other state-of-the-art learning strategies, Support Vector Machine (SVM) methods \citep{christmann2008support,berlinet2011reproducing} have proven particularly successful, both in classification \citep{blanchard2008statistical} and regression settings. By leveraging reproducing kernel Hilbert spaces (RKHS), SVMs provide flexible nonlinear models while maintaining strong theoretical control and computational tractability. A substantial body of literature has investigated their statistical properties, including achievable learning rates under various regularity assumptions \citep{christmann2008support,steinwart2011estimating,caponnetto2007optimal}.

These methods often rely on the assumption of bounded covariates and only offer guarantees on the average performance. As a result, they may fail to provide reliable predictions in rare regimes, when covariates take values potentially outside the support of the observed domain. Maintaining a model's predictive capabilities outside the support of its training data constitutes extrapolation, a notoriously difficult task in machine learning. A closely related framework is domain adaptation \citep{ben2010theory,sugiyama2007covariate}, which specifically addresses settings where the test distribution shifts from the training distribution. However, domain adaptation techniques still rely on target domain samples, hence they are inherently unsuited for regions devoid of data. In contrast, extrapolation tackles the more challenging objective of generalizing to regions where no data has ever been observed. This challenge is especially relevant in climate science, where ML models exhibit limited extrapolation capacities for the prediction of extreme events \citep{pasche2025validating}.

\paragraph{Related works} This work contributes to a recent strain of research dedicated to developing prediction functions that exhibit high performance within the extreme regions of the covariate space. Contrary to the literature on extreme quantile regression, which focuses on extreme values of the response variable $Y$ \citep{Chernozhukov2017,Daouia2013}, our emphasis is on extremes on the covariate $\lVert X \rVert$. The objective is to provide reliable predictions for the response $Y$ in an extrapolation regime, \emph{i.e.} when $\lVert X \rVert$ takes an unusually, potentially unprecedentedly large value. Extreme Value Theory (EVT) provides a natural framework to address this problem \citep{clemenccon2026weak,engelke26Extrap}. Extant literature in this field spans both methodological developments and practical applications, which can be broadly categorized into two main paradigms: the construction of flexible probabilistic models evaluated via experimental simulations, and the study of constrained, stylized frameworks designed to establish non-asymptotic statistical guarantees. For instance, \cite{buritica2024progression} propose flexible models that accommodate unbounded targets through quantile regression; however, their framework is restricted to a unidimensional covariate and lacks formal theoretical guarantees. Similarly, \cite{de2022regression} provide results for quantile regression in a multivariate framework, but their approach is limited to parametric models under max-stable assumptions.
In parallel, recent studies have explored extrapolation learning tasks under regular variation assumptions. 
This line of work has led to the development of machine learning algorithms tailored for extrapolation, for both binary classification \citep{jalalzai2018binary} and least-squares regression  \citep{huet2023regression}. These methods are supported by non-asymptotic statistical learning guarantees regarding the expected error of the prediction function generated by the algorithm, conditional on  the covariate being asymptotically large. A key insight from this literature is that extrapolation can be achieved in this setting by learning a prediction function that depends solely on the angular components associated to the most extreme observations. Following this strategy, a penalized LASSO approach is explored by \cite{clemenccon2026weak} in a linear setting. Applications of such strategies include text representation in natural language processing \citep{jalalzai2020heavy} and environmental modeling such as sea-level analysis \citep{huet2025multi}.

The aforementioned works introduce a conditioned version of the risk functional to the event where the norm of the covariate $X$ exceeds some large threshold $t>0$,
\begin{equation*}
    \risk_t(f) = \EE[\ell(Y,f(X)) \given \lVert X \rVert \geq t].
\end{equation*}
Instead of minimizing the non conditioned version of the risk $\risk(f) = \risk_0(f)$, which accounts essentially for non-extreme regions on the input space, the objective shifts to minimizing an asymptotic version of the conditional risk
\begin{equation*}
    \risk_\infty(f) = \limsup_{t \to \infty} \risk_t(f).
\end{equation*}

Existing approaches to this minimization problem face several limitations. First, the finite-sample guarantees established in \cite{huet2023regression} and \cite{jalalzai2018binary} rely heavily on the restrictive assumption that the predictor class has a finite VC dimension. In practice, this constrains the hypothesis class to linear or low-complexity models. This limitation, coupled with the absence of a regularization term, make their approach poorly suited to high dimensional settings. Moreover, the boundedness assumption on the target made in \cite{huet2023regression} simplifies the theoretical analysis but is overly restrictive in many applications, particularly in extreme value analysis where responses are typically unbounded. Although boundedness can be artificially enforced via nonlinear transformations of the target \citep{huet2023regression,clemenccon2026weak}, such transformations interact poorly with the conditional expectation. Consequently, the guarantees obtained in \cite{huet2023regression} hold only for the transformed target, with no obvious extension for the original untransformed variable. 

\paragraph{Purpose and contributions} It this work, we establish a methodological and theoretical framework for extrapolation that delivers strong practical performance under realistic assumptions. Crucially, our approach addresses a combination of challenges left unresolved by the current state of the art: it accommodates unbounded target variables and operates outside the restrictive VC-dimension framework, all while maintaining rigorous theoretical guarantees. 
Specifically, we leverage quantile regression \citep{KoenkerBassett1978} to bypass the boundedness assumption of the response variable. We first propose an ERM algorithm that extends the work carried in \cite{huet2023regression}, deriving guarantees for the excess asymptotic risk under VC-dimension assumptions. Our main contribution, however, is the introduction of a novel Support Vector Machine (SVM) algorithm tailored for extreme quantile regression. SVM approaches for quantile regression offer a rich class of predictors that do not rely on VC-dimension bounds. While they have been extensively studied in the literature in standard setting, providing consistency and learning rate guarantees  \citep{JMLR:v7:takeuchi06a,steinwart2011estimating,eberts2012optimal}, their capacity for extrapolation remains unexplored. Combining quantile regression with the SVM approach,  we lift the boundedness assumption on the response variable and establish rigorous, tractable high-probability bounds for the excess asymptotic risk. We derive finite-sample learning guarantees in both the ERM and SVM settings under regular variation assumptions on the covariate distribution.
Finally, we illustrate the practical relevance of our methodology on a real-world dataset of river flow measurements from the Danube, where accurate prediction in extreme regimes is of practical importance.

The remainder of this paper is organized as follows. The theoretical framework and proposed algorithms are introduced in Section~\ref{sec:framework}, alongside examples demonstrating the validity of our assumptions in several standard Extreme Value Analysis (EVA) settings. Statistical guarantees are established in Section~\ref{sec:stats}. Section~\ref{sec:expes} presents the empirical study on the Danube river flow data.
All proofs are deferred to the Appendix.

\section{Framework}\label{sec:framework}
Here and throughout, $\lVert \cdot \rVert$ represents any norm on $\R^d$. The associated sphere is defined as $\sphere = \{x \in \R^d \colon \lVert x \rVert = 1\}$ and the angular component of a vector $x \in \R^d \setminus \{0\}$ is denoted by $\theta(x) = x/ \lVert x \Vert$. Given a probability space $(\Omega, \mathcal{A}, \PP)$, and an event $A \in \mathcal{A}$, we let $\un_{A}$ be the associated indicator function. For any random variable $Z$, we write $\mathcal{\law}(Z)$ its law, and denote weak convergence by $\xrightarrow[]{w}$.
\subsection{Quantile regression}

Let $(X,Y)$ be a pair of random variables with joint distribution $P$, where $X$ takes values in $\Xset \subset \R^d$ and $Y$ is real-valued. Fix $\tau \in (0,1)$. The goal of quantile regression is to estimate the \emph{conditional quantile function}, defined for $x \in \Xset$ by
\begin{equation*}
    q(x) = \inf \{t \in \R : \PP(Y \leq t \mid X = x) \geq \tau\}.
\end{equation*}
A common strategy for regression problems consists in introducing a risk function whose minimizer corresponds to the target quantity. For quantile regression, the appropriate loss is the \emph{pinball loss} \citep{KoenkerBassett1978,Koenker2020}, defined by
\begin{equation*}
    \pinbl(y,z) = (1 - \tau)(z - y)\un_{y < z} + \tau(y - z)\un_{y \geq z}.
\end{equation*}
It is known \citep[see \emph{e.g.}][]{steinwart2011estimating} that the conditional quantile function minimizes the pinball risk $\EE[\ell(Y, f(X))]$ over all measurable functions $f:\Xset\to \rset$, whenever the distribution of the target $Y$ has bounded support. This result extends to targets with unbounded support by considering the \emph{modified pinball loss}, defined as
\begin{equation*}
    \pinbl'(y,z) = \pinbl(y,z) - \pinbl(y,0).
\end{equation*}
Although this modified loss is not necessarily positive, it has the advantage of remaining bounded as a function of $y$. Moreover, the pinball loss is $\max(\tau,1-\tau)$-Lipschitz with respect to each variable and satisfies, for any $z \in \R$,
\begin{align*}
    \lvert \pinbl'(y,z) \rvert 
    &\leq \max(\tau,1-\tau) \lvert z \rvert.
\end{align*}
For any measurable function $f \colon \R^d \to \R$ such that $\EE[\lvert f(X)\rvert] < \infty$, the associated risk is defined by
\[
\risk(f) = \EE[\pinbl'(Y,f(X))].
\]
The assumption $\EE[\lvert f(X)\rvert] < \infty$ ensures that the modified pinball risk is well defined, given that the bound
\begin{equation*}
    \lvert \pinbl'(y,f(x)) \rvert \leq \max(\tau,1-\tau) \lvert f(x) \rvert
\end{equation*}
guarantees the integrability of the loss function. This integrability condition on $f$ is essential to secure a finite risk. The corresponding Bayes risk is defined by
\begin{equation*}
    \risk^* = \inf \{\risk(f) : f \text{ measurable, } \EE[\lvert f(X)\rvert] < \infty\}.
\end{equation*}
The following result  shows that the conditional quantile function $q$ minimizes the modified pinball even when the target is unbounded, provided that the conditional quantile $q$ satisfies the integrability condition $\EE[\lvert q(X) \rvert]<\infty$. This extends the work of \cite{steinwart2011estimating} which establishes the identifiability of the quantile function under the strict assumption of a bounded support for $Y$. Here, we relax this constraint, replacing it with this mild integrability condition on the quantile function itself to provide a significantly broader framework.

\begin{proposition}\label{lemme quantile}
Suppose the conditional quantile function satisfies $\EE[|q(X)|]<\infty$. 
Then the conditional quantile function minimizes the modified pinball risk $\risk(f)$ 
over all measurable functions $f$ such that $\EE[|f(X)|] < \infty$, that is,
\begin{equation*}
    \risk(q) = \risk^*.
\end{equation*}
\end{proposition}
\noindent As an illustrative example, consider a heavy-tailed random variable $Y$ that is independent of $X$. While the support of $Y$ is explicitly unbounded, its conditional quantile function is constant and is therefore integrable, demonstrating the utility of our relaxed distributional assumptions. Proposition~\ref{lemme quantile} provides the theoretical foundation of quantile regression within the empirical risk minimization framework. Given a training sample, the goal is to find a function $\widehat{q}$ whose risk is close to the Bayes risk. A natural approach is to minimize the \emph{empirical risk} over a suitable class of functions, possibly complemented by an appropriate regularization term. The aim of this paper is to establish statistical guarantees for such strategies when attention is restricted to extreme regions of the covariate space. This \emph{extreme covariate setting} is introduced in the next subsection.



\subsection{The extreme covariate setting} \label{sec:extreme_setting}

The challenge of regression in extreme regions stems from the fact that estimating quantiles via the direct minimization of the (possibly penalized) modified pinball risk often yields poor performance when $\lVert X\rVert$ takes unusually large values. Indeed, by the law of total probability, the risk can be decomposed into a weighted sum of its behavior in typical and extreme regimes
\begin{align*}
    \risk(f) 
    &= \PP(\lVert X \rVert < t)\, \EE[\pinbl'(Y,f(X)) \mid \lVert X \rVert < t] 
     + \PP(\lVert X \rVert \geq t)\, \EE[\pinbl'(Y,f(X)) \mid \lVert X \rVert \geq t].
\end{align*}
As $t \to \infty$, the contribution of the second term vanishes, causing standard machine learning algorithms to prioritize the bulk of the distribution at the expense of accuracy in the tail regions of the covariate. Following \citet{huet2023regression}, we therefore focus on estimating the conditional quantiles of $Y$ when the norm of $X$ exceeds a large threshold $t$, and consider the \emph{conditional pinball risk}
\begin{equation*}
    \risk[t](f) = \EE[\pinbl'(Y,f(X)) \mid \lVert X \rVert \geq t].
\end{equation*}
If the conditional quantile function satisfies $\EE[|q(X)|] < \infty$, applying Proposition \ref{lemme quantile} to the conditional distribution $\mathcal{L}\left(X,Y \given \lVert X \rVert \geq t\right)$ implies that it also minimizes the conditional pinball risk. It follows that for any $t \geq 0$, the corresponding conditional Bayes risk is given by
\begin{equation}
    \risk_t^* = \inf\{\risk_t(f) : f \text{ measurable, } \EE[|f(X)|] < \infty\} = \risk_t(q).
\end{equation}
Since we are interested in the regime where $t$ becomes large, we define the \emph{asymptotic conditional risk} as
\begin{equation}\label{eq:asymptotic_risk}
     \risk[\infty](f) = \limsup_{t\to \infty} \risk[t](f) 
     = \limsup_{t\to \infty}\EE[\pinbl'(Y,f(X)) \mid \lVert X \rVert \geq t],
\end{equation}
and the corresponding \emph{asymptotic conditional Bayes risk} as
\begin{equation*}
    \risk_{\infty}^* = \inf \{\risk_\infty(f) \colon f \text{ measurable, } \EE[\lvert f(X)\rvert] < \infty\}.
\end{equation*}
Since $\risk_t^* = \risk_t(q)$ for any $t\geq0$, for any $f$ measurable such that $\EE[\lvert f(X) \rvert]<\infty$, we have
\begin{equation*}
    \risk_\infty(f) = \limsup_{t \to \infty} \risk_t(f) \leq \limsup_{t \to \infty} \risk_t(q) = \risk_\infty(q).
\end{equation*}
Consequently, when $\EE[\lvert q(X) \rvert]<\infty$, the conditional quantile function also minimizes the asymptotic risk, therefore,
\begin{equation*}
    \risk[\infty]^* = \risk[\infty](q).
\end{equation*}
Alternatively, the joint behavior of $(X,Y)$ when $\lVert X \rVert$ is exceptionally large can be analyzed by imposing a regular variation condition on the joint distribution of $(X,Y)$. 
Although the  framework considered in this paper is similar to the one in \citet{huet2023regression}, we relax the boundedness assumption on $Y$, which we believe extends significantly the scope of application in a context where the covariate has heavy tails. Recall that a function $b$ is said to be \emph{regularly varying with tail index $\alpha>0$} if 
$b(tx)/b(t) \to x^\alpha$ as $t\to\infty$.

\begin{assumption}[First-component regular variation]\label{assumption regular variation} 
The pair $(X,Y)$ is \textbf{first-component regularly varying} (or regularly varying with respect to the covariate), \ie there exists a nonzero Borel measure $\mu$ on $\R^d \times \R$, finite on sets bounded away from $\{0\} \times \R$, and a regularly varying function $b$ such that
\begin{equation*}
    b(t)\,\PP[(t^{-1}X, Y)\in A\times C] \tto \mu(A\times C)
\end{equation*}
for all Borel sets $A,C$, with $A$ bounded away from $0$ and $\mu(\partial(A\times C)) = 0$.
\end{assumption}
\noindent
Under the latter assumption, we can define the \emph{joint angular measure}
\begin{equation*}
    \Phi(B \times C) = \mu\{(x,y) \in \R^d \times \R : \lVert x \rVert \geq 1, \theta(x) \in B, y \in C\}.
\end{equation*}
\citet{huet2023regression} showed that, this assumption is equivalent to the convergence in distribution
\begin{equation*}
    \big( X/t, Y \mid \lVert X \rVert \geq t \big) \overset{w}{\tto} (X_\infty, Y_\infty),
\end{equation*}
where $(X_\infty, Y_\infty)$ has distribution $\Plim$ defined by
\begin{equation}\label{eq:Plim}
    \Plim\{(x,y)\in \R^d \times \R : \lVert x\rVert \geq r, \theta(x)\in B, y\in C\}
    = r^{-\alpha} \Phi(B\times C).
\end{equation}
Intuitively, the quantile regression of $Y$ on $X$ in the extreme region $\lVert X \rVert \geq t$, with $t \to \infty$, can be understood through the quantile regression of $Y_\infty$ on $X_\infty$. 
Let $\quantileinf(x)$ denote the conditional quantile of $Y_\infty$ given $X_\infty = x$. 
An important consequence of~\eqref{eq:Plim} is that the radial component $\lVert X_\infty\rVert$ is independent of the pair $(\Theta_\infty, Y_\infty)$, where $\Theta_\infty=\theta(X_\infty)$ denotes the angular component of $X_\infty$. 
This implies the following structural property of the extreme conditional quantile. 

\begin{lemma}[Angular structure of the extreme conditional quantile]\label{lemma:angularity-extreme quantile}
The extreme conditional quantile function can be written as $\quantileinf = h_{\Plim} \circ \theta$, where $h_{\Plim} \colon \mathbb{S} \to \mathbb{R}$ is continuous. 
\end{lemma}
\noindent It turns out that the asymptotic risk $\risk[\infty]$ defined by \eqref{eq:asymptotic_risk} is closely related to the risk associated with $(X_\infty,Y_\infty)$. Define the \emph{extreme risk} as
\begin{equation*}
    \risk[\Plim](f) = \EE[\pinbl'(Y_\infty, f(X_\infty))],
\end{equation*}
which is well defined for any function $f$ such that $\EE[|f(X_\infty)|] < \infty$. 
The corresponding \emph{extreme Bayes risk} is defined as
\begin{equation*}
    \risk[\Plim]^* 
    = \inf \{\risk[\Plim](f) : f \text{ measurable, } \EE[|f(X_\infty)|] < \infty\}.
\end{equation*}
If the extreme conditional quantile is measurable and satisfies $\EE[|q_{\Plim}(X_\infty)|] < \infty$, 
then by Proposition~\ref{lemme quantile},
\begin{equation*}
    \risk[\Plim]^* = \risk[\Plim](q_{\Plim}).
\end{equation*}
Parallel to the framework of \citet{huet2023regression}, but adapted to conditional quantiles rather than conditional expectations, we assume the following regularity property.

\begin{assumption}[Regularity and stability assumption]\label{assumption regularity}
The extreme conditional quantile function $\quantileinf$ is continuous on $\R^d \setminus \{0\}$, and
\begin{equation}\label{eq:assumption_2}
\EE\!\left[|\quantiletau(X) - \quantileinf(X)| \mid \lVert X \rVert \geq t\right] \tto 0.
\end{equation}
\end{assumption}
\noindent Under this assumption, the function $q_{\Plim}$ is continuous on the compact sphere $\mathbb{S}$ and is therefore bounded. 
Consequently, $\EE[|q_{\Plim}(X_\infty)|] \leq \|q_{\Plim}\|_\infty < \infty$, and the minimizer of the extreme risk $\risk[\Plim]$ coincides with the extreme conditional quantile function $q_{\Plim}$. 
The precise relationship between the asymptotic conditional risk $\risk[\infty]$ and the extreme risk $\risk[\Plim]$ is stated in the following theorem.

\begin{theorem}\label{Main theorem proba}
Under Assumptions~\ref{assumption regular variation} and~\ref{assumption regularity}, we have:
\begin{enumerate}
    \item For any function of the form $f = h \circ \theta$, where $h\colon \mathbb{S} \to \R$ is continuous, $\risk_t(f) \tto \risk_{P_\infty}(f)$. 
    In particular, $\risk_\infty(f) = \limsup_{t \to \infty} \risk_t(f) = \risk_{P_\infty}(f)$.
    \item The conditional Bayes risk converges to the extreme Bayes risk as $t \to \infty$, i.e. $\risk_t^* \tto \risk_{P_\infty}^*$.
    \item The asymptotic conditional Bayes risk equals the extreme Bayes risk, i.e. $\risk^*_{\infty} = \risk^*_{P_\infty}$.
    \item The asymptotic conditional Bayes risk equals the asymptotic risk of the extreme conditional quantile function, i.e. $\risk^*_{\infty} = \risk_{\infty}(\quantileinf)$.
\end{enumerate}
\end{theorem}
Lemma~\ref{lemma:angularity-extreme quantile} and Theorem~\ref{Main theorem proba} justify our approach to quantile regression in the extreme covariate setting: we estimate an angular function $\widehat{h} \colon \mathbb{S} \to \R$ that minimizes an empirical counterpart, potentially penalized, of the extreme risk $\risk[\Plim]$,  and set $\widehat{q} = \widehat{h} \circ \theta$.

\subsection{Examples} \label{sec:examples}
In this subsection we describe several examples related to standard EVA framework for which Assumptions \ref{assumption regular variation} and \ref{assumption regularity} are satisfied, thereby illustrating the practical relevance of the setting presented in the preceding section. All proofs are deferred to the appendix, section \ref{sec:appendix_proofs_examples}.
\begin{example}[Additive noise model with regularly varying design]\label{example:noisy_target}
    Assume $X$ is regularly varying and that the target $Y$ has the form
    \begin{equation*}
        Y = g(X) + \varepsilon
    \end{equation*}
    for some noise $\varepsilon$ independent from $X$, and some continuous function $g$ satisfying $\EE[\lvert g(X) \rvert] < \infty$. Assume also that there exists some continuous function $g_\theta \colon \sphere \to \R$ satisfying
    \begin{equation} \label{assumption_example-noisy-target}
        \sup_{\lVert x \rVert \geq t} \lvert g(x) - g_\theta(\theta(x)) \rvert \tto 0.
    \end{equation}
    Then, Assumptions \ref{assumption regular variation} and \ref{assumption regularity} hold.
\end{example}
\noindent The next example concerns the prediction of the missing component of some regularly varying vector.

\begin{example}[Predicting a missing component of a regularly varying vector]\label{example:prediction_missing_component}
Let $(X,Z) \in \R^d \times \R$ be a random pair, with $\PP(X=0)=0$. We equip $\R^{d+1}$ with the norm $\lVert \cdot \rVert = \lVert \cdot \rVert_p$ for some $p \in [1,\infty]$. Assume that the pair $(X,Z)$ has positive continuous density $\pi$ on $\R^{d+1} \setminus \{0\}$. Assume also that there exists some regularly varying function $b(t)$ with tail index $\alpha > 0$ and some positive, continuous density $\pi_\infty$ such that
\begin{equation} \label{ex:regular_variation_density}
    \sup_{\lVert(x,z) \rVert > 1} \lvert t^{d+1} b(t) \pi(tx,tz) - \pi_\infty(x,z) \rvert \tto 0.
\end{equation}
For $x \in \R^d$, Denote $\pi_\infty(x) = \int_{\R} \pi_\infty(x,z) \ud z$, and assume that this marginal density has a lower bound on the sphere, \ie there exists $c>0$ such that
\begin{equation} \label{ex:lower_bound_marginal}
    \inf_{\omega \in \sphere} \pi_\infty(\omega) \geq c.
\end{equation}
Define $Y = \frac{Z}{\lVert X \rVert}$. Then, it is shown in Proposition \ref{thm:ex_missing_component} that the pair $(X,Y)$ satisfies both Assumption~\ref{assumption regular variation} and Assumption~\ref{assumption regularity}.
\end{example}

\noindent The framework introduced in Example~\ref{example:prediction_missing_component} resembles the one presented in \cite{huet2023regression} (Example 2.2), adapted to the context of quantile regression. A key advantage of our setting is that the missing component $Z$ is rescaled only by $\lVert X \rVert$, whereas in \cite{huet2023regression} is it rescaled by the norm of the whole vector $\lVert (X,Z) \rVert$. This modification yields a much simpler expression for the prediction of the missing component $Z = Y \lVert X \rVert$. This rescaling also allows us to consider the case $\lVert \cdot \rVert = \lVert \cdot \rVert_\infty$, which is standard in extreme value analysis.
The same normalization strategy can be found in \cite{clemenccon2026weak} (Proposition 4.1), however, their framework still relies on a boundedness assumption on the target $Y$.
The following proposition states that Example \ref{example:prediction_missing_component} falls into the Extreme covariate setting introduced in the previous section.

\begin{proposition} \label{thm:ex_missing_component}
    In the framework of Example~\ref{example:prediction_missing_component}, Assumption~\ref{assumption regular variation} and Assumption~\ref{assumption regularity} are satisfied.
    \end{proposition}
    \subsection{Notations and algorithms}
    Theorem \ref{Main theorem proba} ensures that the extreme conditional quantile function $q_{\Plim}$ is a minimizer of the asymptotic conditional risk $\risk[\infty]$. Moreover, from Lemma \ref{lemma:angularity-extreme quantile}, $q_{\Plim}$ is of angular type. Thus, the asymptotic Bayes risk $\risk_{\infty}^*$ satisfies 
    \begin{equation*}
    \risk_{\infty}^* = \risk_{\infty}(q_{\Plim}) = \inf \{\risk_\infty(f \circ \theta)\colon \text{$f$ measurable, } \EE \lvert f(\theta(X))\rvert<\infty\}.
    \end{equation*}
    Thus, in order to estimate the extreme quantile function $q_{\Plim}$,  one can aim at minimizing $\risk_{\infty}$, restricting the  search to angular predictors. Because the asymptotic risk is not directly available, we instead aim at minimizing $\risk_t(f \circ \theta)$ for some large threshold $t$. Let $(X_1,Y_1), \dots, (X_n, Y_n)$ be an \iid sample with $\PP(X_1 = 0) = 0$, ensuring that, almost surely, $\theta(X_i)$ is well-defined.  One natural strategy for this problem consists in minimizing the empirical version of
\begin{equation*}
    \risk_t(f \circ \theta) = \frac{\EE[\pinbl'(Y,f \circ \theta(X)) \un_{\lVert X \rVert \geq t}]}{\PP(\lVert X \rVert \geq t)}.
\end{equation*}
If $t = t_{n,k}$ is the $(1-k/n)$-quantile of the law of $\lVert X \rVert$, we obtain the  \emph{intermediate risk}
\begin{equation} \label{def:intermediate_risk}
\begin{aligned} 
\riskinter(f \circ \theta) &= \frac{1}{n \PP(\lVert X \rVert \geq \tnk)} \sum_{i=1}^n \pinbl'(Y_i,f\circ \theta(X_i)) \un_{\lVert X_i \rVert \geq \tnk}\\
&= \frac{1}{k} \sum_{i=1}^n \pinbl'(Y_i,f\circ \theta(X_i)) \un_{\lVert X_i \rVert \geq \tnk}.
\end{aligned}
\end{equation}
However, the quantile $\tnk$ is unknown in practice, hence we instead aim at minimizing the \emph{empirical risk}, defined by
\begin{equation*}
    \emprisk[k](f \circ \theta) = \frac{1}{k}\sum_{i=1}^k \pinbl'(Y_{(i)},f\circ \theta(X_{(i)}))
\end{equation*}
where $(X_{(i)}, Y_{(i)})_{i \in \{1,\dots,n\}}$ are the sorted observations, by decreasing order of magnitude of the norm of the covariate $X$. For a given class $\ERMfamily$ of functions $f \colon \sphere \to \R$, we obtain the corresponding ERM algorithm, and the estimator of the extreme quantile function $q_{\Plim}$ is obtained by solving the optimization problem
\begin{equation*}
    \min_{f \in \ERMfamily} \emprisk[k](f \circ \theta).
\end{equation*}
This approach, leading to Algorithm \ref{algo:ERM}, is a natural extension of the algorithm proposed in \cite{huet2023regression} to the context of quantile regression. 
\begin{algorithm}
    \caption{Quantile regression on extremes with unpenalized ERM}
    \label{algo:ERM}
    \hspace*{\algorithmicindent} \textbf{Input:} Training dataset $\mathcal{D}_n = \{(X_1,Y_1),\dots,(X_n,Y_n)\}$, family $\ERMfamily$ of functions $f \colon \sphere \longrightarrow \R$, number $k\leq n$ of extreme observations, $\tau \in (0,1)$. New observation $x$.
    
    \hspace*{\algorithmicindent} \textbf{Training step:}
    Sort the training sample by decreasing order of magnitude of the norm to form an extreme training sample
    \[((X_{(1)},Y_{(1)}), \dots, (X_{(k)},Y_{(k)}))\]
    where $\lVert X_{(1)} \rVert \geq\dots\geq \lVert X_{(k)} \rVert$. Compute the extreme angles $\theta_1 = \frac{X_{(1)}}{\lVert X_{(1)} \rVert}, \dots, \theta_k = \frac{X_{(k)}}{\lVert X_{(k)} \rVert}$.
    
    Solve the optimization problem
    \[\min_{f \in \ERMfamily} \frac{1}{k} \sum_{i=1}^k \pinbl'(Y_{(i)},f(\theta_{(i)})).\]
    Output $\fh$ the minimizer.
    \hspace*{\algorithmicindent} 
    
    \textbf{Prediction step:} If $x \leq \lVert X_{(k)}\rVert$, use some off-the-shelf method to estimate the quantile of the law of $Y$ given $X=x$.
    Else, output $\fh(\theta(x))$.
\end{algorithm}

\begin{algorithm}
    \caption{SVM Quantile regression on extremes}
    \label{algo:SVM}
    \hspace*{\algorithmicindent} \textbf{Input:} Training dataset $\mathcal{D}_n = \{(X_1,Y_1),\dots,(X_n,Y_n)\}$, RKHS $\rkhs$ of functions $h \colon \R^d \longrightarrow \R$, number $k\leq n$ of extreme observations, parameter $\lambda > 0$ and $\tau \in (0,1)$. New observation $x$.
    
    \hspace*{\algorithmicindent} \textbf{Training step:}
    Sort the training sample by decreasing order of magnitude of the norm to form an extreme training sample
    \[((X_{(1)},Y_{(1)}), \dots, (X_{(k)},Y_{(k)})\]
    where $\lVert X_{(1)} \rVert \geq\dots\geq \lVert X_{(k)} \rVert$. Compute the extreme angles $\theta_{(1)} = \frac{X_{(1)}}{\lVert X_{(1)} \rVert}, \dots, \theta_{(k)} = \frac{X_{(k)}}{\lVert X_{(k)} \rVert}$.
    
    Solve the optimization problem
    \[\min_{h \in \rkhs} \frac{1}{k} \sum_{i=1}^k \pinbl'(Y_{(i)},h(\theta_{(i)})) + \lambda \normrkhs[h]^2.\]
    Output $\hhlambda$ the minimizer.
    \hspace*{\algorithmicindent} 
    
    \textbf{Prediction step:} If $x \leq \lVert X_{(k)}\rVert$, use some off-the-shelf method to estimate the quantile of the law of $Y$ given $X=x$.
    Else, output $\hhlambda(\theta(x))$.
\end{algorithm}
However, the lack of a penalization term in the ERM approach described in Algorithm \ref{algo:ERM} makes it highly vulnerable to the challenges of high dimensionality. Moreover, the statistical guarantees obtained in this framework typically rely on VC-dimension arguments, which severely restrict the capacity of the predictor class over which optimization can be performed. These observations motivate our SVM-based approach, described in Algorithm \ref{algo:SVM}. Given a parameter $\lambda>0$ and a RKHS $\rkhs$, the penalized empirical risk is defined by
\begin{equation*}
    \empriskPen_k(h \circ \theta) = \emprisk[k](h\circ \theta) + \lambda \normrkhs[h]^2.
\end{equation*}
The estimator $\hhlambda \circ \theta$ of $q_{\Plim}$ is obtained by solving the optimization problem
\begin{equation*}
    \min_{h \in \rkhs} \empriskPen_k(h \circ \theta).
\end{equation*}

Note that this optimization step can be performed working with the pinball loss or its modified version, and produces the same output because the two corresponding risk functionals only differ by a finite constant.
In this paper, the RKHS $\rkhs$ consists of functions $h \colon \R^d \to \R$. We will also always assume that the reproducing kernel $K$ associated to the RKHS $\rkhs$ is continuous, bounded by $1$. All the results we provide extend to the more general case where $K$ is bounded by some constant $M_K$.

In many RKHS frameworks, the compactness of the input space for the covariate $X$ is essential. Here, this assumption is violated since the input space is $\R^d$, moreover, analyzing extremes inherently prevents restricting the domain to a compact set. However, Theorem \ref{Main theorem proba} implies that we only need to consider the behavior of any $h \in \mathcal{H}$ on the sphere $\sphere$. This motivates the introduction of the space $\rkhs_\sphere$ of functions from $\rkhs$ restricted to $\sphere$,
\begin{equation}\label{def:rkhs'}
    \rkhs_{\sphere} = \{h \lvert_\sphere \colon h \in \rkhs\}.
\end{equation}
By applying the results of \cite{aronszajn1950theory} regarding the restriction of reproducing kernels, we establish the following proposition, which ensures that $\rkhs_\sphere$ inherits the RKHS structure of $\rkhs$.

\begin{proposition}[RKHS structure on $\rkhs_\sphere$]
    Let $\rkhs_\sphere$ be the space of functions of $\rkhs$ restricted to the sphere, defined by \eqref{def:rkhs'}. Let $K_\sphere$ be the restriction of $K$ to the sphere, \ie the kernel defined on $\sphere \times \sphere$ by
    \begin{equation*}
        K_\sphere(\theta_1,\theta_2) = K(\theta_1,\theta_2).
    \end{equation*}
    Then, $\rkhs_\sphere$ equipped with the norm
    \begin{equation*}
        \lVert h_\sphere \rVert_{\rkhs_\sphere} = \min \{ \normrkhs[h] \colon h\lvert_\sphere = h_\sphere\}
    \end{equation*}
    is a RKHS, with reproducing kernel $K_\sphere$.
\end{proposition}
    \noindent In the following, we will denote by $\ball_{\rkhs}$ the closed unit ball of $\rkhs$, and $\ball_{\rkhs_\sphere}$ the closed unit ball of $\rkhs_\sphere$.

\section{Statistical guarantees} \label{sec:stats}

We now present nonasymptotic results for both approaches. The statistical guarantees regarding the ERM approach, described in Algorithm \ref{algo:ERM}, are described in the next subsection. 
Our main contributions, however, lie in the statistical guarantees for our penalized approach, described in Algorithm~\ref{algo:SVM}, which is the focus of Section \ref{sec:stat_SVM}. For both algorithms, when $\widehat{f} \circ \theta$ is the issued estimator of the extreme quantile function $q_{\Plim}$, we provide a high probability bound for the \emph{excess asymptotic pinball risk}, that is, $\risk_\infty(\widehat{f} \circ \theta) - \risk_\infty^*$. 

\subsection{Statistical guarantees in the ERM framework} \label{sec:stat_ERM}

This section serves a preparatory role, highlighting the limitations of the ERM approach studied in \cite{huet2023regression}. We first demonstrate that their statistical guarantees can be adapted to the quantile regression framework, yielding explicit bounds on the excess risk for the estimator produced by Algorithm \ref{algo:ERM}. These guarantees require several strong assumptions on the class $\ERMfamily$. Namely, we assume that the functions $f \in \ERMfamily$ are uniformly bounded by a constant $M$. 
\begin{assumption}[Uniform boundedness of the family $\ERMfamily$]\label{assumptionERM:uniform_bound}
    There exists $M>0$ such that
    \begin{equation*}
        \sup_{f \in \ERMfamily} \lVert f \rVert_\infty \leq M.
    \end{equation*}
\end{assumption}
This assumption ensures, in particular, that the risk functionals under consideration are well-defined. With this assumption, the exact same decomposition of the excess asymptotic risk $\risk_\infty(\fh) - \risk_\infty^*$ as in \cite{huet2023regression} can be recovered, their result not being specific to the context of least-square setting but rather to the ERM framework. The proof is omitted for brevity, as it is straightforward to verify that the arguments used in \cite{huet2023regression} (Theorem 3.3) remain valid for any arbitrary loss function. We thus obtain the following proposition.
\begin{proposition}[Bias-variance decomposition of the excess asymptotic risk in the ERM framework] \label{thm:erm_bias_variance_dec}
    Suppose Assumptions \ref{assumption regular variation}, \ref{assumption regularity} and \ref{assumptionERM:uniform_bound} hold and let $\tnk$ be the $(1-k/n)$ quantile of $\lVert X \rVert$. Let $\fh$ be the prediction function issued by Algorithm \ref{algo:ERM}. Then,
    \begin{equation*}
        \risk_\infty(\fh \circ \theta) - \risk_\infty^* \leq B_1(\tnk) + B_2(\ERMfamily) + D_k
    \end{equation*}
    where $B_1(t)$ is the threshold bias, defined by
    \begin{equation*}
        B_1(t) = 2\sup_{f \in \ERMfamily}\lvert \risk_\infty(f \circ \theta) - \risk_t(f \circ \theta) \rvert,
    \end{equation*}
    $B_2(\rkhs)$ is the class bias, defined by
    \begin{equation*}
        B_2(\ERMfamily) = \inf_{f \in \ERMfamily} \risk_\infty(f \circ \theta) - \risk_\infty(q_{\Plim} \circ \theta),
    \end{equation*}
    and $D_k$ is a stochastic deviation term, defined by
    \begin{equation*}
        D_k = 2 \sup_{f \in \ERMfamily}\lvert \risk_{\tnk}(f\circ \theta) - \emprisk_k(f \circ \theta) \rvert.
    \end{equation*}
\end{proposition}
\noindent Recall that a metric space $(E,d)$ is totally bounded if, for any $\varepsilon>0$, it can be covered by a finite number of balls of radius $\varepsilon$ with centers in $E$.
The following proposition shows that, whenever the class $\ERMfamily$ is totally bounded, the bias term vanishes as $t \to \infty$.
\begin{proposition} [Control of the threshold bias]\label{prop:thershold_bias_vanishes}
    Suppose Assumptions \ref{assumption regular variation}, \ref{assumption regularity} and \ref{assumptionERM:uniform_bound} hold, and that the family $\mathcal{F}$ is totally bounded in the space $(\mathcal{C}(\sphere), \lVert \cdot \rVert_\infty)$. Then, as $t \to \infty$,
    \begin{equation*}
        B_1(t) = \sup_{f \in \mathcal{F}} \big\lvert \risk_t(f \circ \theta) - \risk_\infty(f \circ \theta) \big\rvert \tto 0.
    \end{equation*}
\end{proposition}
\noindent The proof of this proposition is deferred to the appendix in Section \ref{appendix:proof_threshold_bias}.
To obtain statistical guarantees on the deviation term $D_k$, we make an additional VC-dimension assumption on the class $\ERMfamily$.
\begin{assumption}[Finite VC-dimension of $\ERMfamily$]\label{assumptionERM:VCdim}
    The class $\ERMfamily$ has finite VC-dimension, denoted by $V_\ERMfamily$, and it is pointwise measurable, \ie there exists $\ERMfamily_0 \subset \ERMfamily$ which is countable and such that for any $\theta \in \sphere$ and $f \in \ERMfamily$, there exists a sequence $(f_i)_{i \in \N}$ of functions in $\ERMfamily_0$ with
    \begin{equation*}
        f_i(\theta) \xrightarrow[i \to \infty]{} f(\theta).
    \end{equation*}
\end{assumption}
\noindent In this setting, a high-probability bound can be derived for the deviation term $D_k$.
\begin{proposition}[Control on the deviation term] \label{prop:statistical_guarantee_ERM}
    Suppose Assumptions \ref{assumptionERM:uniform_bound} and \ref{assumptionERM:VCdim} hold. Then, there exists a universal constant $C$ such that, with probability no less than $1-\delta$,
    \begin{equation*}
    \begin{aligned}
        \sup_{f \in \ERMfamily} \lvert \risk_{\tnk}(f\circ \theta) - \emprisk[k](f \circ \theta) \rvert \leq \frac{M\max(\tau,1-\tau)}{\sqrt{k}} \Big(C\sqrt{V_\ERMfamily} + 2\sqrt{2\log(1/\delta)}
        +\frac{3 \log(1/\delta)}{\sqrt{k}} + \frac{C V_\ERMfamily}{\sqrt{k}}\big).
    \end{aligned}
    \end{equation*}
\end{proposition}
\noindent We do not claim originality in the technique of proof employed for Proposition \ref{prop:statistical_guarantee_ERM}. For completeness, we provide the proof in the Appendix, Section \ref{sec:appendix_proof_stat_guarantee_ERM}.

\noindent We now examine the results of Proposition \ref{thm:erm_bias_variance_dec} and the assumptions required to establish the statistical guarantees in this subsection. First, note that Assumption \ref{assumptionERM:uniform_bound} is not necessarily unrealistic as the Bayes asymptotic predictor is the extreme quantile function $q_{\Plim}$, which we assumed to be continuous on the sphere. This guarantees that it is, in particular, bounded. However, $\lVert q_{\Plim} \rVert$ is unobservable in practice.  This dependency on an unknown uniform bound introduces a operational challenge regarding the calibration of $M$.  Selecting $M$ incorrectly may drastically inflate the approximation error $B_2(\ERMfamily)$, or the deviation term $D_k$, yielding an implicit bias-variance tradeoff for the choice of $M$. Additionally, control over the approximation bias $B_2(\ERMfamily)$ remains theoretically elusive. Finally, note that the VC-type assumption on the class $\ERMfamily$ severely limits the capacity of the hypothesis spaces available for optimization. Collectively, these drawbacks, alongside the well-known vulnerability of ERM methods to the curse of dimensionality underscore the necessity of a regularized approach. This is the object of the next subsection, investigating the penalized approach described in Algorithm \ref{algo:SVM}.


\subsection{Statistical guarantees in the SVM framework} \label{sec:stat_SVM}
In this section, we propose to derive theoretical guarantees for the SVM approach, described in Algorithm \ref{algo:SVM}. We derive a complete error decomposition for the excess asymptotic risk and bound each resulting term, thereby establishing tractable guarantees under relatively mild assumptions.
One advantage of the SVM approach is that it yields bounds on $\normrkhs[\hhlambda]$. In the case where the kernel of the RKHS $\rkhs$ is bounded, this property ensures that $\hhlambda$ is uniformly bounded as well. This ultimately allows us to completely drop any boundedness assumption on the target $Y$, or any boundedness assumption similar to Assumption \ref{assumptionERM:uniform_bound}. Likewise, we will drop Assumption \ref{assumptionERM:VCdim}, at the cost of replacing the bias term $B_2(\ERMfamily)$ by another error term, the \emph{asymptotic approximation error} of the RKHS $\rkhs$.


\subsubsection{Preliminary results on SVM solutions}
\label{sec:SVM_preliminary}
Recall that our predictor  $\hhlambda$ is defined as a minimizer  of the penalized empirical risk
\begin{equation*}
    \empriskPen[k](h \circ \theta) = \emprisk[k](h \circ \theta) + \lambda \normrkhs[h]^2, \quad h\in\mathcal{H},
\end{equation*}
where $\emprisk_k(h \circ \theta) = \frac{1}{k}\sum_{i=1}^k\pinbl'(Y_{(i)}, h \circ \theta(X_{(i)}))$. We  say that $\hhlambda$ is the \emph{empirical SVM solution}. 

In the sequel, we introduce three other notions  of SVM solution that will be useful in our analysis and correspond to the different notions of risk considered in Section~\ref{sec:framework}. The existence and uniqueness of these minimizers are proved in Corollary \ref{cor:Bound_norm} below. For a threshold $t>0$, we  define  the \emph{conditional objective SVM solution} $h^*_{\lambda,t}$ as a minimizer of the penalized objective conditional risk
\begin{equation*}
    \riskPen_t(h \circ \theta) = \risk_t(h \circ \theta) + \lambda \normrkhs[h]^2, \quad h\in\mathcal{H},
\end{equation*}
where $\risk_t(h \circ \theta)= \EE[\pinbl'(Y,f(X)) \mid \lVert X \rVert \geq t]$.
We also introduce the \emph{asymptotic SVM solution} $h^*_{\lambda,\infty}$  defined as a minimizer of
\begin{equation*}
    \riskPen_{\Plim}(h \circ \theta) = \risk_{\Plim}(h \circ \theta) + \lambda \normrkhs[h]^2, \quad h\in\mathcal{H},
\end{equation*}
where $\risk_{\Plim}(h \circ \theta)= \EE[\pinbl'(Y_\infty,h\circ\theta(X_\infty)) ]$. Finally, we  introduce a final SVM solution that we call the \emph{intermediate SVM solution}. Let $\tnk$ be the $(1-k/n)$-quantile of $\lVert X \rVert$ and define $\widetilde{h}_{\lambda,\tnk}$ as a minimizer of the penalized pseudo-empirical risk
\begin{equation*}
    \widetilde{\risk}_k^{(\lambda)}(h \circ \theta) = \riskinter(h \circ \theta) + \lambda \normrkhs[h]^2,\quad h\in\mathcal{H},
\end{equation*}
where $\riskinter(h\circ \theta)$ was defined by \eqref{def:intermediate_risk}.

\noindent In this section, we state some preliminary results on these SVM solutions.  The first step is to show all of these SVM solutions are well defined. 
\begin{lemma}\label{lem:Existence-uniqueness-minimizer}
    Assume the kernel $K$ is bounded with $\sup_{x \in \R^d}K(x,x)=1$. For any $\lambda > 0$ and   any distribution $P$ on $\R^d \times \R$ such that $P(0 \times \R) = 0$, there exists a unique $h^*_{\lambda,P} \in \rkhs$  minimizing the penalized conditional risk 
    \[
    \riskPen_P(h \circ \theta) = \risk_P(h \circ \theta) + \lambda \normrkhs[h]^2,\quad h\in\mathcal{H},
    \]
    where  $\risk_P(h \circ \theta) = \EE_P[\pinbl'(Y,h \circ \theta(X))]$ denotes the modified pinball risk under $P$. Moreover, $h_{\lambda,P}^*$ satisfies
    \begin{equation*}
        \normrkhs[h_{\lambda,P}^*] \leq \frac{\max(\tau,1-\tau)}{\lambda}.
    \end{equation*}
\end{lemma}
 \noindent The proof of Lemma~\ref{lem:Existence-uniqueness-minimizer} follows the lines of \cite{christmann2008support}, Lemma 5.1 and Theorem 5.2, with an adaptation for the existence. This adaptation is necessary  because the modified pinball loss may take negative values and consequently, we do not have in general that $\lambda \normrkhs[h]^2 \leq \lambda \normrkhs[h]^2 + \risk_t(h \circ \theta) $, which is a crucial inequality in the proof of \cite{christmann2008support}. However, using the fact that the modified pinball loss is Lipschitz and the reproducing property of the RKHS $\rkhs$, we still manage to bypass this difficulty, and remove the integrability assumption on the target $Y$. The details for this proofs and other proofs for this section are deferred to the Appendix, Section \ref{appendix:proof_Existence_uniqueness_minimizer}. By ensuring the existence and uniqueness of the SVM solution under general conditions, this lemma confirms that the previously defined SVM solutions are well-posed. This is formally stated in the following corollary.


    \begin{corollary}[The SVM solutions are well-defined]\label{cor:Bound_norm}
        Assume the kernel $K$ is bounded with $\sup_{x \in \R^d} K(x,x) = 1$ and let $t>0$. Then, the SVM solutions $\hhlambda, h_{\lambda,t}^*, h_{\lambda,\infty}^*, \htildelambda$ are well defined. Moreover,
        \begin{equation*}
            \max\left(\normrkhs[\hhlambda], \normrkhs[h^*_{\lambda,t}],\normrkhs[h^*_{\lambda,\infty}]\right) 
            \leq \frac{\max(\tau,1-\tau)}{\lambda}.
        \end{equation*}
    \end{corollary}
\noindent Note that Proposition~\ref{def:rkhs'} implies that the restrictions of $\hhlambda, h^*_{\lambda,t}, h^*_{\lambda,\infty}$ to the sphere $\sphere$ satisfy the same inequalities in the RKHS $\rkhs_{\sphere}$, i.e.
    \begin{equation*}
            \max\left(\lVert h_{\lambda\lvert\sphere} \rVert_{\rkhs_\sphere},
            \lVert h^*_{\lambda,t\lvert\sphere} \rVert_{\rkhs_\sphere},
            \lVert h^*_{\lambda,\infty\lvert\sphere} \rVert_{\rkhs_\sphere} \right) 
            \leq \frac{\max(\tau,1-\tau)}{\lambda}.
        \end{equation*}
This is a straightforward consequence of the inequality $\lVert h_{\vert\sphere}\rVert_{\rkhs_\sphere}\leq \lVert h\rVert_{\rkhs}$, for all $h\in\rkhs$.
    Finally, the following proposition justifies the use of a RKHS $\rkhs$ of functions in the whole space $\R^d$ instead of a RKHS $\rkhs_\sphere$ of functions on the sphere $\sphere$, as the two approaches yield the same SVM solutions.
    \begin{proposition}[Equivalence of the SVM problems in $\rkhs$ and $\rkhs_\sphere$]
    \label{prop:restriction_RKHS_SVM_solution}
    Let $P$ be any probability distribution on $\R^d \times \R$ with $P(0 \times \R) = 0$. Consider the conditional objective SVM solution $h^*_{\lambda,P}$, obtained by minimizing the penalized conditional risk over $\mathcal{H}$. The restriction $h^*_{\lambda,P\vert\sphere}$  of $h^*_{\lambda,P}$ to the sphere is the unique minimizer in $\rkhs_\sphere$  of 
    \begin{equation*}
        \risk_{P}^{(\lambda)}(h \circ \theta) = \risk_{P}(h\circ \theta) + \lambda \lVert h \rVert_{\rkhs_\sphere}^2,\quad h\in \rkhs_\sphere.
    \end{equation*}
    \end{proposition}
    
\subsubsection{Decomposition of the excess asymptotic risk in the SVM framework} \label{sec:stat_svm_debut}
Equipped with the preliminary results established in the previous section, we now provide a non asymptotic analysis for the excess risk of the estimator $\hhlambda$. More precisely, we provide a high probability bound for the \emph{penalized excess asymptotic risk}, that is,
\begin{equation*}
    \risk_\infty(\hhlambda \circ \theta) + \lambda \normrkhs[\hhlambda]^2 - \risk_\infty^*.
\end{equation*}
We derive a decomposition for the excess asymptotic penalized risk analogous to the one presented in Proposition \ref{thm:erm_bias_variance_dec}, which is formally stated below.

\begin{theorem}[Bias-variance decomposition for the excess asymptotic risk in the SVM framework]\label{thm:SVM_bias_variance_dec}
    Suppose Assumptions \ref{assumption regular variation} and \ref{assumption regularity} hold. Let $\lambda>0$, and assume that the RKHS $\rkhs$ has a continuous kernel $K$ that is bounded by $1$. Consider $\hhlambda$ issued by Algorithm \ref{algo:SVM}. Then, 
    \begin{equation}\label{eq:decomposition_risk}
        \risk_\infty(\hhlambda \circ \theta) + \lambda \normrkhs[\hhlambda]^2 - \risk_\infty^* \leq B(n,k,\lambda) + \approxerror_2^\infty(\lambda) + D(n,k,\lambda),
    \end{equation}
    where 
    \begin{equation} \label{eq:threshold_bias}
    B(n,k,\lambda) = 2 \sup_{\normrkhs[h] \leq \max(\tau,1-\tau) \lambda^{-1}} \big \lvert \risk_{\tnk}(h \circ \theta) - \risk_\infty(h \circ \theta) \big \rvert 
\end{equation}
is the  \textbf{threshold bias},
\begin{equation}\label{asymptotic_approx_error}
    \begin{aligned}
    \approxerror_2^\infty(\lambda) &=  \risk_\infty(h^*_{\lambda,\infty} \circ \theta) + \lambda \normrkhs[h^*_{\lambda,\infty}]^2 - \risk_\infty^*\\
    &= \inf_{h \in \rkhs} \Big( \risk_\infty(h \circ \theta) + \lambda \normrkhs[h]^2 \Big) - \risk_\infty^*.
\end{aligned}
\end{equation}
is the \textbf{asymptotic approximation error} of the RKHS $\rkhs$, and

\begin{equation}\label{eq:variance_term}
    D(n,k,\lambda) = \riskPen_{\tnk}(\hhlambda) - \riskPen_{\tnk}(h^*_{\lambda,\tnk})
\end{equation}
is a \textbf{variance term}.
\end{theorem}
\noindent In the following, we analyze the different terms appearing in the decomposition \eqref{eq:decomposition_risk}. We start with the threshold bias $B(n,k,\lambda)$ defined by Equation \eqref{eq:threshold_bias}. The following proposition shows that, analogous to Proposition \ref{prop:thershold_bias_vanishes}, the bias term $B(n,k,\lambda)$ vanishes as $n \to \infty$ for any fixed $\lambda>0$. Notably, this result is obtained without imposing any additional assumptions on the RKHS $\rkhs$, whereas the corresponding ERM guarantee required the hypothesis space $\ERMfamily$ to be totally bounded.

\begin{proposition}[Bias control]\label{prop:BiasControl}
    Let Assumption \ref{assumption regular variation} hold. Then, for any $\lambda > 0$, 
    \begin{equation*}
    \sup_{\lVert h \rVert_{\rkhs} \leq \max(\tau,1-\tau)\lambda^{-1}} \big\lvert \risk_{t}(h \circ \theta) - \risk_\infty(h \circ \theta) \big\rvert \xrightarrow[t \to \infty]{} 0.
\end{equation*}
\end{proposition}
\noindent As a corollary to Proposition~\ref{prop:BiasControl} we obtain that, for any sequence $k(n)$  such that $k(n)/n\to 0$ and $k(n)/n$ is non-increasing,  there exists a vanishing sequence $\lambda(n)$ for which the   bias term  vanishes  in the asymptotic regime as $n\to\infty$.
\begin{corollary} \label{cor:threshold_bias_vanishes_SVM}
  Let $k=k(n)$ depend on $n$ in a way such that 
  $k(n)/n\to 0$ as $n\to\infty$. There exists a
  sequence $\lambda(n)\to 0$ such that the bias term
  $B(n, k(n), \lambda(n))$ defined in Proposition~\ref{prop:BiasControl}, converges to $0$ as $ n \to \infty$.
  
\end{corollary}

\subsubsection{Bounding the variance term}
\label{sec:stat_bound_distance}
We now turn to the variance term $D(n,k,\lambda)$ defined by equation \eqref{eq:variance_term}. 
A high-probability bound for $D(n,k,\lambda)$ can be obtained by exploiting the fact that $\hhlambda$ minimizes the penalized empirical risk. This implies
\begin{equation*}
    \empriskPen[k](\hhlambda \circ \theta) \leq \empriskPen[k](h^*_{\lambda,\tnk} \circ \theta).
\end{equation*}
Thus,
\begin{equation*}
\begin{aligned}
    \riskPen_t(\hhlambda \circ \theta) &= \empriskPen[k](\hhlambda \circ \theta) + \left(\risk_t(\hhlambda \circ \theta) - \emprisk[k](\hhlambda \circ \theta) \right)\\
    &\leq \empriskPen[k](h^*_{\lambda,\tnk} \circ \theta) + \left(\risk_t(\hhlambda \circ \theta) - \emprisk[k](\hhlambda \circ \theta) \right)\\
    &= \riskPen_t(h^*_{\lambda,\tnk} \circ \theta) + \left(\risk_t(\hhlambda \circ \theta) - \emprisk[k](\hhlambda \circ \theta) \right) - \left(\risk_t(h^*_{\lambda,\tnk} \circ \theta) - \emprisk[k](h^*_{\lambda,\tnk} \circ \theta)\right).
\end{aligned}
\end{equation*}
Leveraging the bound on $\normrkhs[\hhlambda]$ and $\lVert h^*_{\lambda,\tnk} \rVert_\rkhs$ provided by Corollary \ref{cor:Bound_norm}, we obtain
\begin{equation} \label{eq:bound_diff_risk_sup_1}
    \left(\risk_t(\hhlambda \circ \theta) - \emprisk[k](\hhlambda \circ \theta) \right) - \left(\risk_t(h^*_{\lambda,t} \circ \theta) - \emprisk[k](h^*_{\lambda,t} \circ \theta)\right)
    \leq 2 \sup_{h \in \frac{\max(\tau,1-\tau)}{\lambda} \ball_{\rkhs}} \lvert \risk_t(h \circ \theta) - \emprisk[k](h \circ \theta) \rvert.
\end{equation}
While this latter quantity could be bounded with high probability by adapting the techniques established in \cite{huet2023regression}, such an approach yields a suboptimal rate of order $(\lambda \sqrt{k})^{-1}$. To secure faster convergence rates, standard literature leverages the method of localization \citep{boucheron2005theory,bartlett2005local,lecue2012general}, which measures the capacity of a localized subset of functions near the optimum rather than the global hypothesis space. A key technical difficulty in our context stems from the random nature of the sample size exceeding the threshold $\tnk$, coupled with the small number of observations $k$ retained for training. Addressing this requires a careful analysis via tailored Bernstein-type concentration inequalities. Moreover, the classical localized bounds require both a variance-control condition and a light-tailed (bounded or sub-exponential) target variable $Y$. We prefer to circumvent these requirements in order to accommodate unbounded, potentially heavy-tailed targets. 
The method we propose instead leverages the stability of SVM \citep{christmann2008support,bousquet2002stability} and prove that $\hhlambda$ is close to $h^*_{\lambda,t}$ with high probability. Consequently, we can restrict the supremum in \eqref{eq:bound_diff_risk_sup_1} to a localized subset of the form
\begin{equation*}
    \sup_{h \in \ball_{\rkhs}(h^*_{\lambda,t},\rho(\lambda,k,\delta))} \Big\lvert \left(\risk_t(h \circ \theta) - \emprisk[k](h \circ \theta)\right) - \left(\risk_t(h^*_{\lambda,t} \circ \theta) - \emprisk[k](h^*_{\lambda,t} \circ \theta)\right) \Big\rvert
\end{equation*}
for some small radius $\rho(\lambda,k,\delta)$. 
To obtain a sharp estimate on $\lVert \hhlambda - h^*_{\lambda,\tnk} \rVert_\mathcal{H}$, we first bound the distance of $\hat{h}_\lambda$ to the intermediate SVM solution $\tilde{h}_\lambda$. The following lemma demonstrates that the norm $\lVert \hat{h}_\lambda - \tilde{h}_\lambda \rVert_\mathcal{H}$ can be controlled by a term capturing how closely the $(1-k/n)$-quantile of $\lVert X \rVert$ is approximated by the order statistic $\lVert X_{(k)} \rVert$.
\begin{lemma}\label{lem:distance_hh_htilde}
    Let $t = t_{n,k}$ be the $1-\frac{k}{n}$-quantile of $\lVert X \rVert$. Assume the kernel $K$ is bounded by $1$. Then, 
    \begin{equation}\label{eq:distance hh htilde}
        \lVert \hhlambda - \htildelambda \rVert_\rkhs \leq \frac{\max(\tau,1-\tau)}{2\lambda k} \Big\lvert\sum_{i=1}^n\un_{\lVert X_i \rVert \geq \tnk} -k \Big\rvert.
    \end{equation}
\end{lemma}
\noindent Applying Bernstein's inequality to the right-hand side of $\eqref{eq:distance hh htilde}$ yields a high probability bound on $\normrkhs[\hhlambda - \htildelambda]$.
\begin{corollary}\label{cor:distance_htilde_hh}
    If the kernel $K$ is bounded by $1$, with probability no less than $1-\delta$,
    \begin{equation} \label{eq:bound_hh_htilde}
        \lVert \hhlambda - \htildelambda \rVert_\rkhs \leq \frac{\max(\tau,1-\tau)}{2\lambda} \Big(\sqrt{\frac{2\log(2/\delta)}{k}} + \frac{\log(2/\delta)}{3k}\Big).
    \end{equation}

\end{corollary}
\noindent With the bound on $\normrkhs[\hhlambda - \htildelambda]$ established, it remains to control $\normrkhs[\htildelambda - h^*_{\lambda,t} ]$. We leverage a variant of the representer theorem, utilizing Theorem 5.8 and Corollary 5.10 in \cite{christmann2008support}. When paired with a Hilbert-space adaptation of Bernstein's inequality, this approach delivers a tight probabilistic control on $\normrkhs[\hhlambda-\htildelambda]$.

\begin{proposition}[High probability bound on $\lVert \htildelambda - h^*_{\lambda,\tnk} \rVert_\rkhs$]
\label{prop:distance_htilde_hstar}
With probability  no less than  $1-\delta$,
\begin{equation} \label{eq:bound_htilde_h^*}
    \lVert \htildelambda - h^*_{\lambda,\tnk} \rVert_\rkhs \leq \frac{ \max(\tau,1-\tau)}{\lambda} \left(\sqrt{\frac{2 \log(1/\delta)}{k}} + \frac{1}{\sqrt{k}} + \frac{4\log(1/\delta)}{3k}\right).
\end{equation}
\end{proposition}
\noindent The details for the proof of the previous claims are in the Appendix, Section \ref{appendix:proof_bernstein_hh_h^*}. Combining the two high probability bounds \eqref{eq:bound_hh_htilde} and \eqref{eq:bound_htilde_h^*} yields a high probability bound on the distance between the two SVM solutions $\hhlambda$ and $h^*_{\lambda,\tnk}$.

\begin{theorem}[high probability bound on $\lVert \hhlambda - h^*_{\lambda,\tnk}\rVert_\rkhs$]\label{distance_hh_h^*}
    Assume the kernel $K$ is bounded by 1. Then, with probability no less than $1-\delta$,
    \begin{equation}
        \lVert \hhlambda - h^*_{\lambda,\tnk} \rVert_\rkhs \leq  
        \frac{\max(\tau,1-\tau)}{\lambda \sqrt{k}}\left(1 + 3\sqrt{\frac{\log(3/\delta)}{2}} + \frac{3\log(3/\delta)}{2\sqrt{k}}\right).
    \end{equation}
\end{theorem}
\noindent Based on this high-probability control, we are now in a position to establish a non-asymptotic bound for the variance term $D(n,k,\lambda)$. The proof is deferred to the Appendix, Section \ref{appendix:proof_variance_bound}.
\begin{theorem}[Control of the variance term] \label{thm:variance_term_bound}
    Consider $\hhlambda$ issued by Algorithm \ref{algo:SVM}. Assume the kernel $K$ of the RKHS $\rkhs$ is continuous and bounded by $1$. Define
    \begin{equation*}
        C(k,\delta) =  \Big(1 + 3 \sqrt{\frac{ \log(1/\delta)}{2}} + \frac{3\log(1/\delta)}{2\sqrt{k}}\Big).
    \end{equation*}
    Then, it holds that, with probability no less than $1-4\delta$:
    \begin{equation}
        D(n,k,\lambda) \leq D_1(k,\lambda,\delta) + D_2(k,\lambda,\delta),
    \end{equation}
    where
    \begin{equation}
        D_1(k,\lambda,\delta) = \frac{\max(\tau,1-\tau)^2}{\lambda k} \Big(\sqrt{\frac{\log(1/\delta)}{2}} + \frac{\log(1/\delta)}{3\sqrt{k}}\Big)^2
    \end{equation}
    and 
    \begin{equation}
        D_2(k,\lambda,\delta) = \frac{\max(\tau,1-\tau)^2}{\lambda k} C(k,\delta) \Big( 3 + \sqrt{\frac{2\log(1/\delta)} {k}}
        + \frac{16 \log(1/\delta)}{3k}\Big). 
    \end{equation}

\end{theorem}

\subsubsection{Bounding the approximation error}
The remaining term to analyze in Theorem \ref{thm:SVM_bias_variance_dec} is the asymptotic approximation error $\approxerror_2^\infty(\lambda)$, defined by
\begin{equation}
    \approxerror_2^\infty(\lambda) = \inf_{h \in \rkhs} \Big(\risk_\infty(h \circ \theta) + \lambda \normrkhs[h]^2 \Big) - \risk_\infty^*.
\end{equation}
A standard assumption regarding this error is formulated as follows.
\begin{assumption}\label{assumption:approx_error}
    There exists $c>0$ and $\beta \in (0,1]$ such that, for any $\lambda > 0$,
    \begin{equation*}
        \approxerror_2^\infty(\lambda) \leq c \lambda^\beta.
    \end{equation*}
\end{assumption}
\noindent Combining Assumption \ref{assumption:approx_error} with the decomposition from Theorem \ref{thm:SVM_bias_variance_dec} and the variance bound established in Theorem \ref{thm:variance_term_bound}, and and choosing $\lambda = k^{-\frac{1}{\beta+1}}$, we find that the error term $\approxerror_2^{\infty}(\lambda) + D(\lambda,k,\delta)$ is of order $k^{-\frac{\beta}{\beta+1}}$. In this section, we discuss additional conditions that ensure Assumption~\ref{assumption:approx_error} is satisfied and we outline several directions for future research.
A particularly simple example  where Assumption~\ref{assumption:approx_error} holds true with $\beta=1$, is when the quantile function belongs to the angular RKHS $\rkhs_\sphere$, as stated formally below.
\begin{proposition}
    Write $q_{\Plim} = h_{\Plim} \circ \theta$ as in Lemma \ref{lemma:angularity-extreme quantile} and suppose that $h_{\Plim}$ belongs to $\rkhs_{\sphere}$. Then, $\approxerror_2^\infty(\lambda) \leq \lambda \lVert h_{\Plim}\rVert_{\rkhs_\sphere}^2$.
\end{proposition}
\begin{proof}
    This is a straightforward adaptation of the result of \cite{christmann2008support}, Corollary~5.18, with some necessary adaptation caused by the angular structure of the considered risk functionals and the extreme quantile function $q_{\Plim}$. Write $q_{\Plim} = h_{\Plim} \circ \theta$ as in Lemma \ref{lemma:angularity-extreme quantile}, and assume that $h_{\Plim} \in \rkhs_\sphere$. By Theorem \ref{Main theorem proba}, item 4, $\risk_\infty^* = \risk_\infty(q_{\Plim}) = \risk_\infty(h_{\Plim} \circ \theta)$. We obtain 
    \begin{align*}
        \approxerror_2^\infty(\lambda) &= \inf_{h \in \rkhs} \Big(\risk_\infty(h \circ \theta) + \lambda \normrkhs[h]^2 \Big) - \risk_\infty^*\\
        &\leq \risk_\infty(h_{\Plim} \circ \theta) + \lambda \normrkhs[h_{\Plim}]^2 - \risk_\infty(h_{\Plim} \circ \theta) = \lambda \lVert h_{\Plim} \rVert_{\rkhs_\sphere}^2.
    \end{align*}
\end{proof}
\noindent The assumption that $h_{\Plim}$ belongs to $\rkhs_{\sphere}$ is arguably very  stringent. 
Several relaxations of this condition have been proposed in the SVM literature. For instance, consider the case where $(X,Y) \sim P$ where $P$ be any distribution on $\R^d \times \R$. In the case of quantile regression in this setting where no extrapolation is involved, the approximation error is
\begin{equation*}
    \approxerror_2(\lambda) = \inf_{h \in \rkhs} \big(\EE_P[\pinbl(Y,h(X))] + \lambda \normrkhs[h]^2\big) - \EE_P[\pinbl(Y,q(X))].
\end{equation*}
\cite{christmann2008support} (section 5.6) show that if the corresponding quantile function $q$ belongs to some interpolation space \citep[see, e.g.,][]{smale2003estimating}, then Assumption \ref{assumption:approx_error}  is satisfied with explicit forms for $\beta$ and $c$. Alternatively, if $q$ belongs to a Besov-like space and a Gaussian kernel is employed, \cite{eberts2012optimal} and \cite{meister2016optimal} obtain explicit upper bounds on the approximation error by exploiting the regularity of $q$ and utilizing a convolution scheme to smooth the conditional quantile function.
We anticipate that these results could be extended to our context. The primary bottleneck would be modifying the working assumption in the cited reference, namely, that the input space is a subset of $\R^d$ with non-empty interior and a bounded Lipschitz domain,  to accommodate the case of interest here, where the kernel's domain  is the sphere $\sphere$.

\section{Case study on a real dataset}\label{sec:expes}

The aim of our experiments is threefold: (i) to compare the relative performance of ERM using an unpenalized linear model that satisfies the VC assumptions against, an SVM with a Gaussian kernel, thereby illustrating the bias-variance trade-off; (ii) to visualize the effect of marginal standardization on predicted and observed target values when covariates are not on a standard Pareto scale; and (iii) to compare the proposed SVM and ERM approaches against natural baselines, namely naive methods that do not specifically rely on the angular component of the covariate, or that apply a standard mean-deviation standardization instead of the probability integral transform.


Introduced by \cite{asadi2015extremes}, the dataset we consider constitutes a classic benchmark in the extreme value theory literature, having been further analyzed across multiple studies \citep{engelke2020graphical, mhalla2020causal}. It comprises river flow measurements from $31$ gauging stations in the upper Danube basin.
Our goal  is to predict a conditional quantile of the river flow at a target station, based on large measurements from the other stations. This prediction should be reliable in a crisis scenario \emph{e.g.} in case of a flood in the Danube basin, which translates in our setting to unusually high flow levels at one or more stations. 

This study focuses exclusively on learning the conditional median, but our methodology generalizes to arbitrary quantiles.
We treat the daily measurements as an \iid sample. While this assumption ignores potential temporal dependencies, we set this issue aside for the purpose of our analysis. In our experiment, we take $\lVert \cdot \rVert = \lVert \cdot \rVert_\infty$, so that the event $\lVert X \rVert \geq t$ occurs whenever one of the component of $X$ exceeds the threshold $t$.

\subsection{Workflow}
This section describes the proposed prediction procedure, which is applicable for extreme values of the covariate. We describe in detail the marginal standardization step, the choice and rescaling of the target station, and the cross-validation procedure employed to select all hyper-parameters.
\paragraph{Marginal standardization}
The observations are grouped into a vector $\widetilde{X} \in \R^{31}$.
We begin by a marginal  standardization step, aiming at transforming  the vector $\widetilde{X}$ to unit-Pareto margins. This standardization step is necessary in this case because the river flow is not homogeneous along the Danube river basin, thus the distribution of the observed river flow vary greatly from one station to another.
The standardization of the vector $\widetilde{X}$ considered here follows the approach of \cite{coles1994statistical} (see also \cite{coles1991modelling} or \cite{beirlant2006statistics}, Equation (9.65)). 
Denote by $F_j$ the cumulative distribution function of the marginal $\widetilde{X}_j$. For some high threshold $u_j$, the exceedance over the threshold $u_j$ is modeled by a generalized Pareto distribution, \ie, denoting $\zeta_j = \PP(X_j>u_j)$ for some parameters $\xi_j, \sigma_j$ and for any $x>u_j$,
\begin{equation*}
    F_j(x) = 1 - \zeta_j(1 + \xi_j \frac{x - u_j}{\sigma_j})^{-1/\xi_j}.
\end{equation*}
The transformation of $F_j$ in the bulk of the distribution, \ie in the domain $x<u_j$, is done with the marginal empirical distribution. Given an \iid sample $\widetilde{X}_1,\dots,\widetilde{X}_n$, for each station $j \in \{1,\dots,31\}$ and corresponding vector $\widetilde{X}_{:j}$, the estimation of $F_j(x)$ is
\begin{equation*}
\widehat{F}_j(x) = 
\begin{cases}
 1 - \zeta_j \left(1+ \hat{\xi}_j \frac{x - \hat{u}_j}{\hat{\sigma}_j}\right)^{-\frac{1}{\hat{\xi}_j}} \text{ if } x \geq \hat{u}_j\\
 \frac{1}{n}\sum_{k=1}^n \un_{\tilde{X}_{k,j}\leq x} \text{ otherwise.} 
\end{cases}
\end{equation*}
where $\hat{u}_j$ is an empirical $\zeta_j$-quantile for $F_j$, and $\hat{\xi}_j, \hat{\sigma}_j$ are estimated by maximum likelihood.
The vector of observations for station $j$, standardized to unit Pareto margins is defined by
\begin{equation*}
    \widehat{V}_j = \left(\frac{1}{1-\widehat{F}_j(\tilde{X}_{ij})}\right)_{i \in \{1, \dots, n\}}.
\end{equation*}
In the following, we work with the resulting standardized dataset $\widehat{V} = (\widehat{V}_{i,j})_{(i,j) \in \{1,\dots,n\} \times \{1,\dots,31\}}$. Without loss of generality we assume that the target station is the last one,  and we denote $(X_i,Z_i) = \widehat{V}_i$. In this work we do not investigate theoretically the impact of the standardization error, leaving this important question to further work. This question is a notoriously technical one in the literature of multivariate EVT involving pseudo-polar coordinates with non-linear transforms, although some asymptotic results exist regarding estimation of the angular measure in dimension $d=2$ using an empirical estimator of the marginal distribution functions \citep{EinmahlPiterbargHaan2001,einmahl2009maximum}, as well as non-asymptotic ones in arbitrary dimension \citep{ClemenconJalalzaiSabourinSegers2023}.

\paragraph{Prediction algorithm}
In light of Example \ref{example:prediction_missing_component}, we propose the following procedure for predicting the missing component of the (transformed) vector $(X,Z)$:
\begin{enumerate}
    \item For each measurement $(X_i,Z_i)$, compute $Y_i = \frac{Z_i}{\lVert X_i \rVert}$.
    \item Use Algorithm \ref{algo:ERM} or \ref{algo:SVM} to learn an estimate $\widehat q_Y(x)$ of the quantile function of $Y$ given $X=x$. 

The theoretical justification relies on Example~\ref{example:prediction_missing_component} which requires regular variation as well as a non degeneracy condition for the limiting distribution.

\item The estimated conditional quantile function for the Pareto-scaled target $Z$ given the transformed covariate $x$ is then 
$$
\widehat q_Z(x) = \|x\| \widehat q_Y(x). 
$$
\item In the end, when it comes to predicting the original target, based on a new observation $\tilde x_{1:30} = (\tilde x_1,\ldots, \tilde x_{30})$, we apply the Pareto transformation to the new covariate, and the inverse transformation to the estimated quantile function. The final output is the predicted conditional quantile of $\widetilde X_{31}$ given $\big\{(\widetilde X_1,\ldots \widetilde X_{30}) = \tilde x_{1:30}\}$, 
$$
\widehat q_{\widetilde X}(\tilde x_{1:30})= \widehat F_{31}^{\leftarrow}\Big(1 - \frac{1}{ \widehat q_Z(x)}\Big), 
$$
where $x=(x_1,\ldots, x_{30})$ with $x_j = \big(1 - \widehat F_j(\tilde x_j)\big)^{-1}$ and where $(\point)^\leftarrow$ denotes the generalized quantile function.

\end{enumerate}

\paragraph{Choice of the target station}
 Our proposed method is applicable to any target station $Z$ for which the normalized variable $Y = Z/\lVert X \rVert$ does not degenerate to $0$ as $\lVert X \rVert \to \infty$. Consequently, our theoretical guarantees apply to target stations where extreme values coincide with large covariate norms. Based on this criterion, stations 15 through 18 serve as suitable candidates for our analysis. As illustrated in Figure \ref{fig:map_danube}, these stations are situated consecutively in a downstream direction of the Danube.
 \begin{figure}
     \centering
     \includegraphics[width=0.7\linewidth]{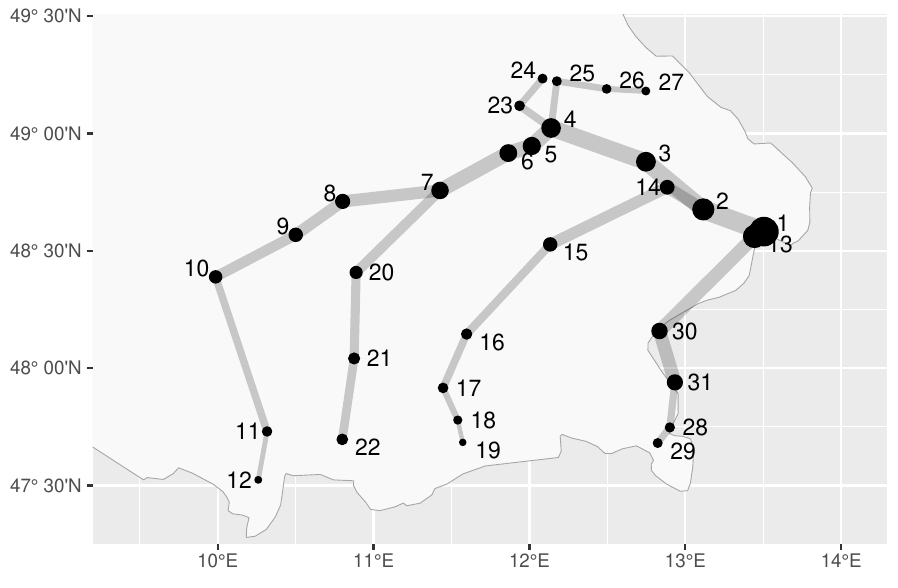}
     \caption{The geographical distribution of the gauging stations along the Danube in Bavaria, Germany. Line thickness is proportional to the average flow volume.}
     \label{fig:map_danube}
 \end{figure}
 Focusing on station 18, Figure~\ref{fig:scatter_plot} 
shows that large values for $Z$ and $\lVert X \rVert$ may occur simultaneously. The scatter plot of $Y = Z/\lVert X \rVert$ against $\lVert X \rVert$ illustrates that, conditionally on $\lVert X \rVert \geq t$, the distribution of $Y$ approaches a mixture of a point mass at zero and a non-degenerate component.
More precisely, Figure~\ref{fig:scatter_plot} suggests that the conditional distribution of $Y$ given $X=x$ degenerates to $\delta_0$ whenever $x$ lies in the region where the coordinates corresponding to stations 15, 16, and 17 are negligible relative to $\lVert x \rVert$, and remains non-degenerate otherwise. This empirical insight indicates that the framework established in Example~\ref{example:prediction_missing_component} may not strictly hold. Furthermore, due to the asymptotic dependencies among the coordinates of $\widetilde{X}$, the existence of a limiting density $\pi_\infty$ for the distribution of the limiting vector  as in Example \ref{example:prediction_missing_component} is not guaranteed. Nevertheless, despite these theoretical challenges, the empirical performance of our procedure remains robust. A rigorous treatment of these issues lies beyond the scope of this work and is deferred to future research.

The selection of the station to carry our experiments is consistent with the existing work on the Danube dataset. In particular, stations $15, 16, 17$ and $18$ are flow connected and, as demonstrated in \cite{engelke2020graphical}, they exhibit asymptotic dependence.

\paragraph{Selection of hyper-parameters}
In all our experiments, the hyper-parameters are selected via cross-validation. We investigate two learning frameworks: Empirical Risk Minimization (ERM) and Support Vector Machines (SVM) equipped with a Gaussian kernel. The hyperparameter configurations differ substantially between the two approaches. In the ERM framework, the optimization involves a single parameter, namely the number of extreme observations retained for training, denoted by $\ktrain$. Conversely, the SVM approach requires the simultaneous tuning of three hyper-parameters: the training sample size $\ktrain$, the regularization parameter $\lambda$, and the Gaussian kernel bandwidth $\gamma$. In the latter case, the hyper-parameters $\gamma$ and $\lambda$ are selected via a grid search already implemented in the \texttt{liquidSVM} package \citep{steinwart2017liquidsvm}. We now detail the specific procedure employed to select the remaining parameter, $\ktrain$.

To select the optimal number of training observations $k_{\text{train}}$, we employ a modified cross-validation scheme. The entire dataset is partitioned into $K$ disjoint folds. In each iteration, a single fold is reserved for training the model, while the remaining $K-1$ folds constitute the validation set. This layout inverts the standard $K$-fold cross-validation procedure, where $K-1$ folds are typically used for training. The rationale behind this design stems from the fact that the generalization error is evaluated exclusively on a subsample of the validation set containing the most extreme observations. Consequently, the validation set must be sufficiently large to ensure a statistically reliable evaluation of the performance metrics above high thresholds. In our empirical implementation, we set $K=5$. Thus, in each iteration, the training set comprises $20\%$ of the total observations, while the validation set contains the remaining $80\%$.
We evaluate the predictive performance by comparing the estimated conditional median $m(X)$ learned during training against the true values of $Y$ in the validation set. To do so, we employ an empirical approximation of the asymptotic pinball risk, computed exclusively on the subset of validation points exhibiting the largest covariate norms. Specifically, the asymptotic risk is estimated via
\begin{equation}
\frac{1}{\ktest}\sum_{i=1}^{\ktest} \ell_{1/2}(Y_{(i)}, m(X_{(i)})),
\end{equation}
where $(X_{(1)}, Y_{(1)}), \dots, (X_{(\ktest)}, Y_{(\ktest)})$ denote the observations corresponding to the $\ktest$ largest covariate norms $\lVert X \rVert$. This estimator provides a precise approximation of the asymptotic risk provided the fraction $\tau_{\textrm{test}} = \ktest/\ntest$ is sufficiently small. In our empirical analysis, we fix $\tau_{\textrm{test}} = 1\%$.

We perform these measurements over $10$ permutations of the whole transformed dataset, varying the number $\ktrain$ of extreme observations to learn from. We compare the results obtained using Algorithm \ref{algo:ERM} and Algorithm \ref{algo:SVM}. The ERM framework (Algorithm~\ref{algo:ERM}) is implemented using the \texttt{sklearn} quantile regressor, where the predictor class $\mathcal{F}$ is restricted to linear functionals $h(\theta) = \langle w, \theta \rangle$ for $w \in \R^d$. For the SVM alternative (Algorithm~\ref{algo:SVM}), the model operates over a Gaussian RKHS with an adjustable bandwidth parameter $\gamma$.

\subsection{Experimental Results}
\begin{figure} 
    \begin{subfigure}[b]{0.4\textwidth}
        \centering
        \includegraphics[width=1\textwidth]{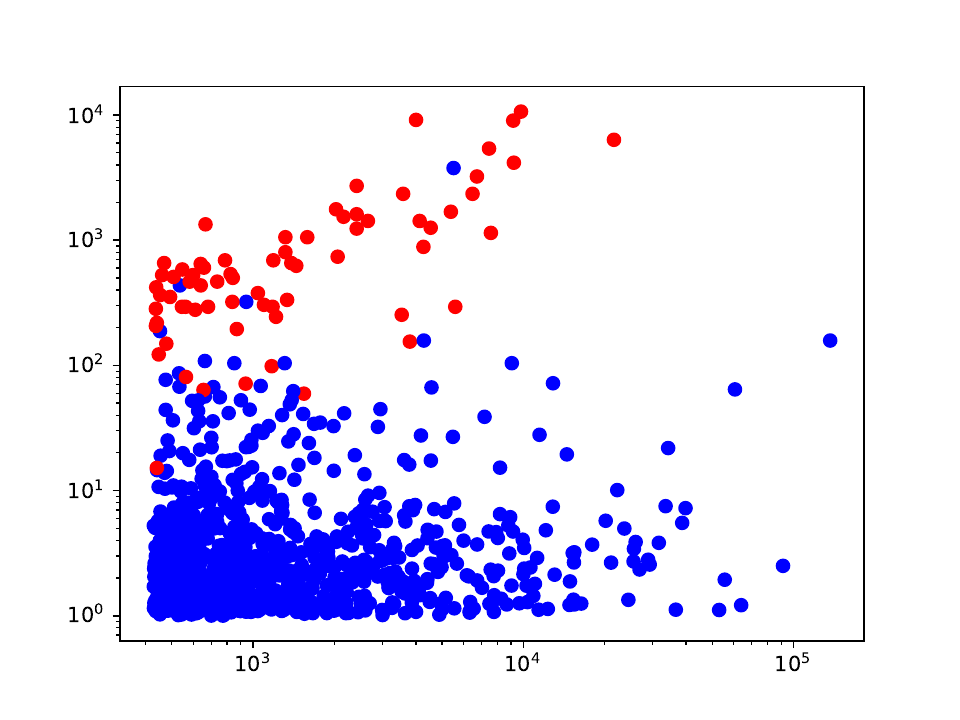}
    \end{subfigure}
    \centering
    \begin{subfigure}[b]{0.4\textwidth}
        \centering
        \includegraphics[width=1\textwidth]{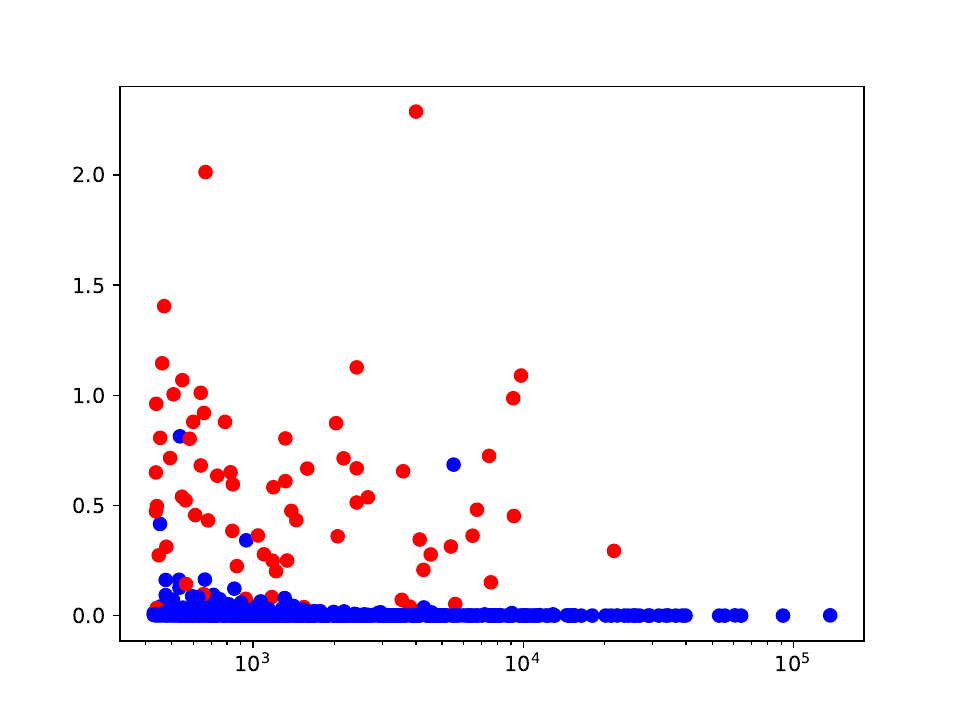}
    \end{subfigure}%
    \caption{Left panel: log-log Scatter plot of $Z$ with respect to $\lVert X \rVert$ for station $18$, on the $5\%$ most extreme observations. Right panel: Scatter plot of $Y = Z/\lVert X \rVert$ with respect to $\lVert X \rVert$ for station $18$, on the $5\%$ most extreme observations. Points are colored red if the maximum for the norm $\lVert X \rVert_\infty$ is attained in one of the three stations 16, 17 or 18, and blue otherwise.}
    \label{fig:scatter_plot}
\end{figure}

\paragraph{Bias-Variance trade-off and ERM versus SVM}
Figure~\ref{fig:tradeoff_bias_variance_Danube} displays the predictive pinball risks for the SVM method (Algorithm~\ref{algo:SVM} with a Gaussian kernel) and the ERM method (Algorithm~\ref{algo:ERM} with a linear predictor class) for varying values of $k_{\text{train}}$. These curves are estimated using the cross-validation procedure detailed above. The results reveal a clear bias-variance trade-off regarding $k_{\textrm{train}}$: the best performance for both SVM and ERM occurs at a $k_{\textrm{train}}/n_{\textrm{train}}$ ratio between $0.05$ and $0.1$. This ratio is five to ten times larger than that of the test set, meaning the test set's radial threshold is much higher than the optimal training threshold. This discrepancy confirms the effectiveness of our approach for extreme value extrapolation.
Additionally, the results reported in Figure~\ref{fig:tradeoff_bias_variance_Danube} illustrate that SVMs outperform ERM with a non-penalized linear model at lower sample sizes, specifically when the covariate dimension is large relative to the sample size and the curse of dimensionality becomes more apparent.


\begin{figure} 
    \centering
    \includegraphics[width=.5\linewidth]{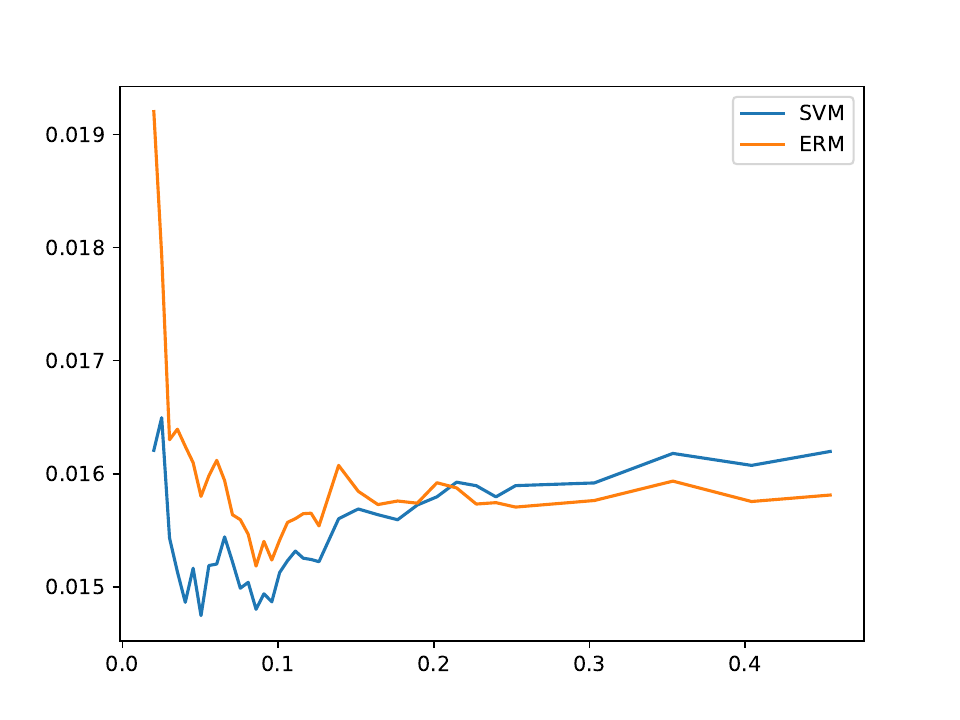}
    \caption{Mean approximated asymptotic pinball risk over 10 permutations of the whole dataset with the described $5$-fold cross-validation procedure, for different values $\ktrain$. The generalization error is estimated when using the ERM approach corresponding to Algorithm \ref{algo:ERM}, or the SVM approach corresponding to Algorithm \ref{algo:SVM}.}
    \label{fig:tradeoff_bias_variance_Danube}
\end{figure}

\paragraph{Effect of standardization}
Figure~\ref{fig:predictions_true_values_our_approach_extreme_values} displays the effect of standardization in  Algorithm~\ref{algo:SVM}. Specifically, it plots the predicted quantiles for $Y = Z/\lVert X \rVert$ against their corresponding true observed values on both the transformed and original scale. The left panel shows the predictions and observations on the Pareto scale, while the right panel displays both values back-transformed to the original scale.
 Based on the bias-variance tradeoff observed in Figure \ref{fig:tradeoff_bias_variance_Danube}, a proportion of $5\%$ of the most extreme observations were retained for training.

\paragraph{Comparison with baseline approaches}

We evaluate the proposed ERM and SVM methods by comparing them against naive baseline counterparts. For these baselines, each component of $\widetilde{X}$ is independently standardized using empirical means and standard deviations, yielding $\widehat{V}' = (\widehat{V}'_{i,j}) \in \mathbb{R}^{n \times 31}$. Let $(X'_i, Z'_i) = \widehat{V}'_i$. We then train a quantile ERM or SVM on the subset of observations with the largest infinity norms to estimate the conditional median $m(x)$ of $Z'$ given $X'=x$. The hyperparameters are determined via grid search using our cross-validation framework. A key distinction is that the baseline models use the raw covariate $X'$ directly, whereas our proposed approach relies on the angular component $\theta(X)$.
Figure~\ref{fig:comparison_approaches} displays the results, showing that the naive approach fails to provide accurate predictions for large station measurements. Conversely, our proposed methodology remains relevant in this regime.

\begin{figure} 
    \centering
    \begin{subfigure}[b]{0.4\textwidth}
        \centering
        \includegraphics[width=1\textwidth]{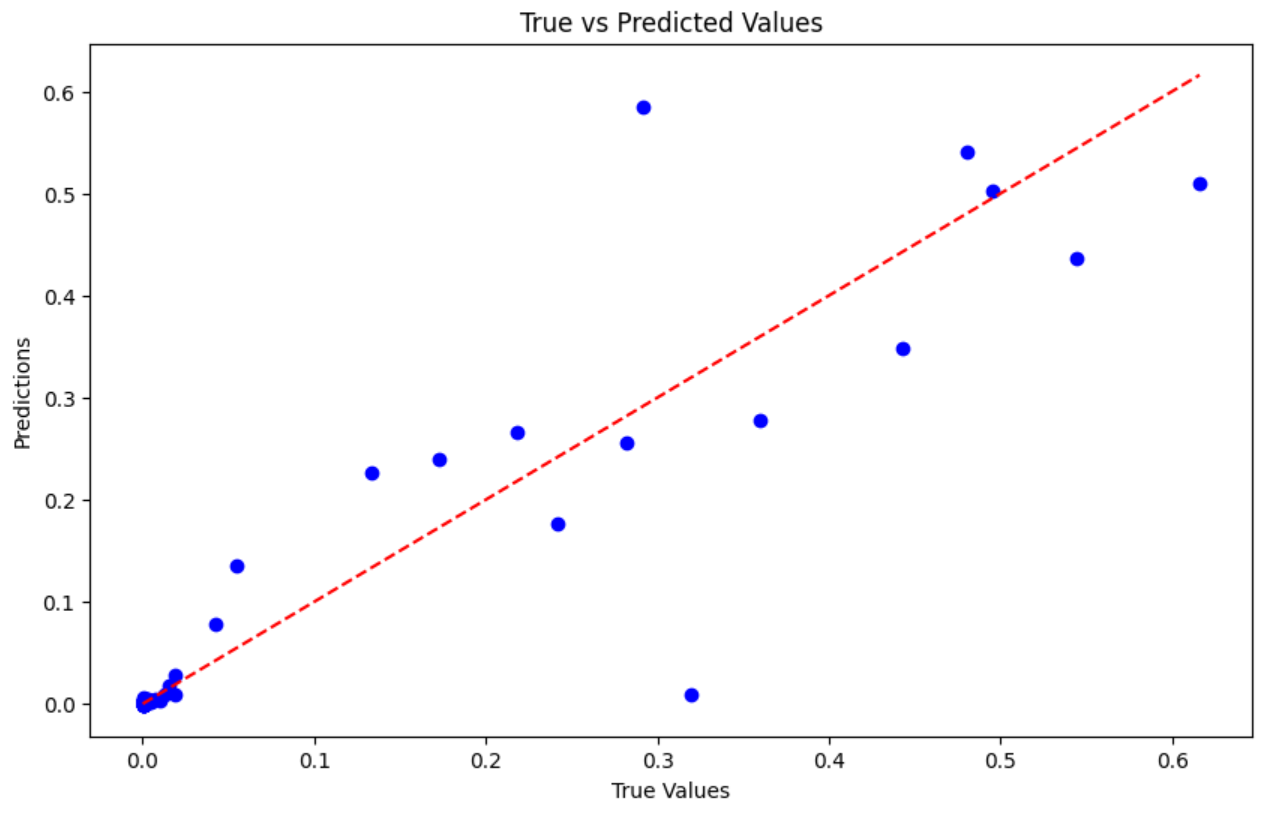}
    \end{subfigure}%
    ~ 
    \begin{subfigure}[b]{0.4\textwidth}
        \centering
        \includegraphics[width=1\textwidth]{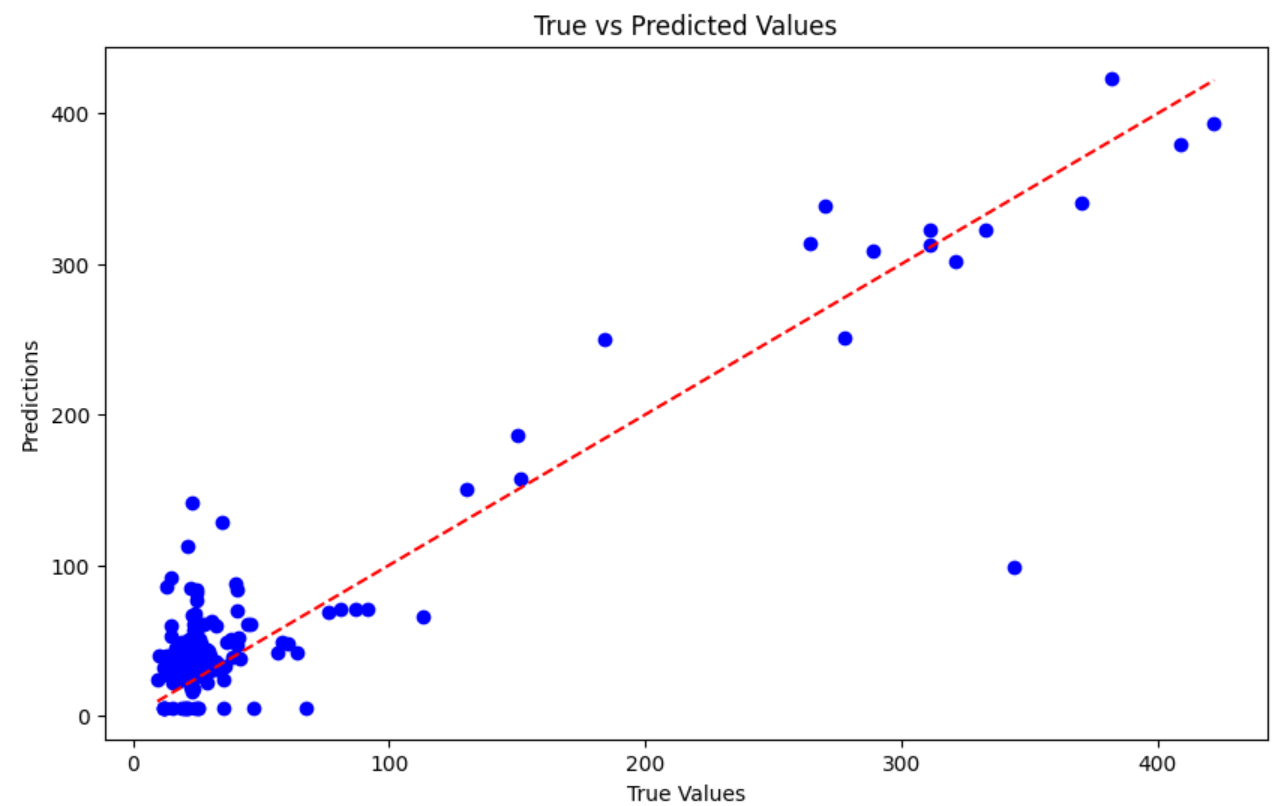}
    \end{subfigure}
    \caption{Left panel: Prediction versus true values for the prediction of $Y = Z/\lVert X \rVert$, where the coordinates have been standardized to unit Pareto margins. Right panel: Predictions versus true values for the prediction of $Z$ with our approach in its original, non transformed scale.}
    \label{fig:predictions_true_values_our_approach_extreme_values}
\end{figure}



\begin{figure}[t]
\begin{subfigure}[b]{.5\linewidth}
\includegraphics[width=1\textwidth]{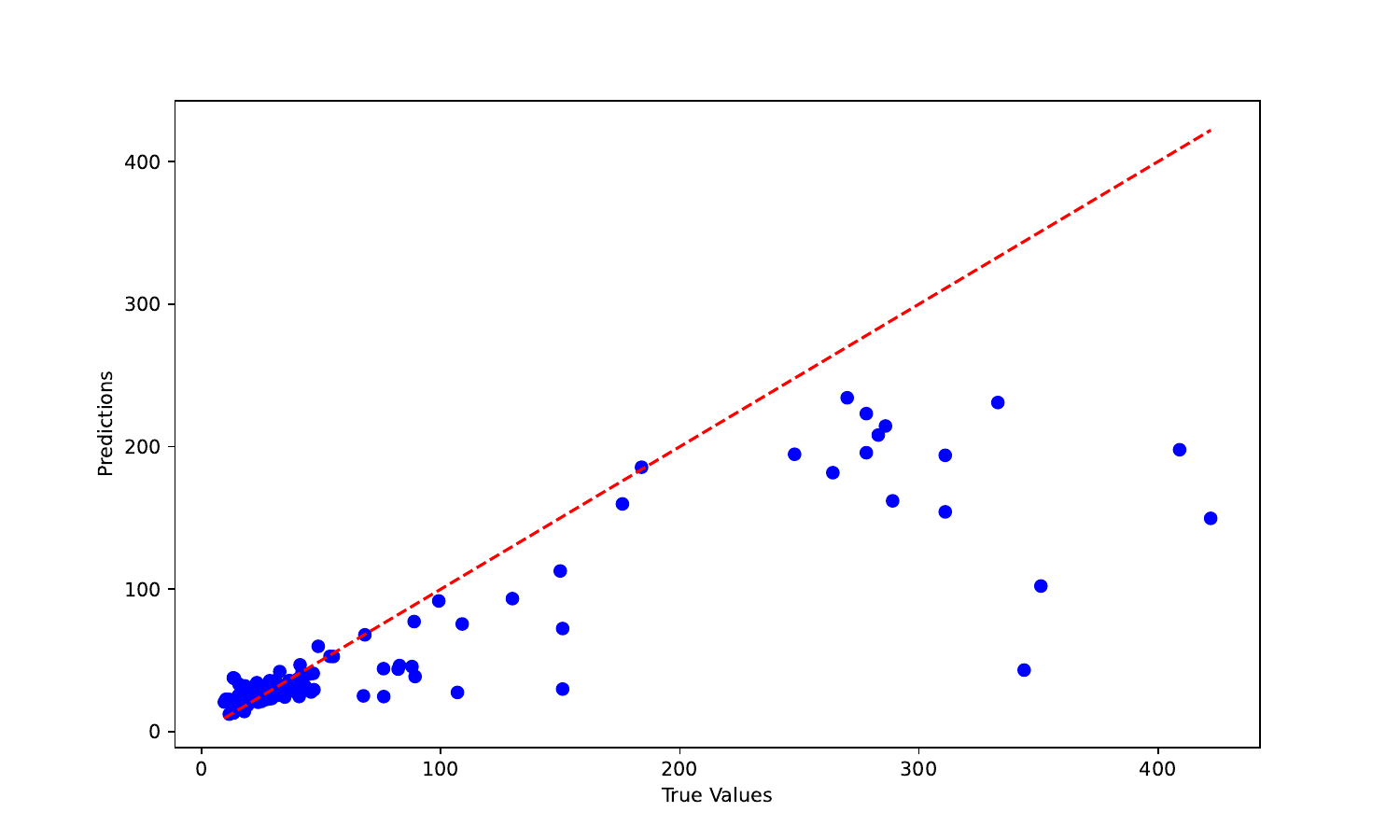}
\caption{Baseline approach, SVM algorithm}
\end{subfigure}
\begin{subfigure}[b]{.5\linewidth}
\includegraphics[width=1\textwidth]{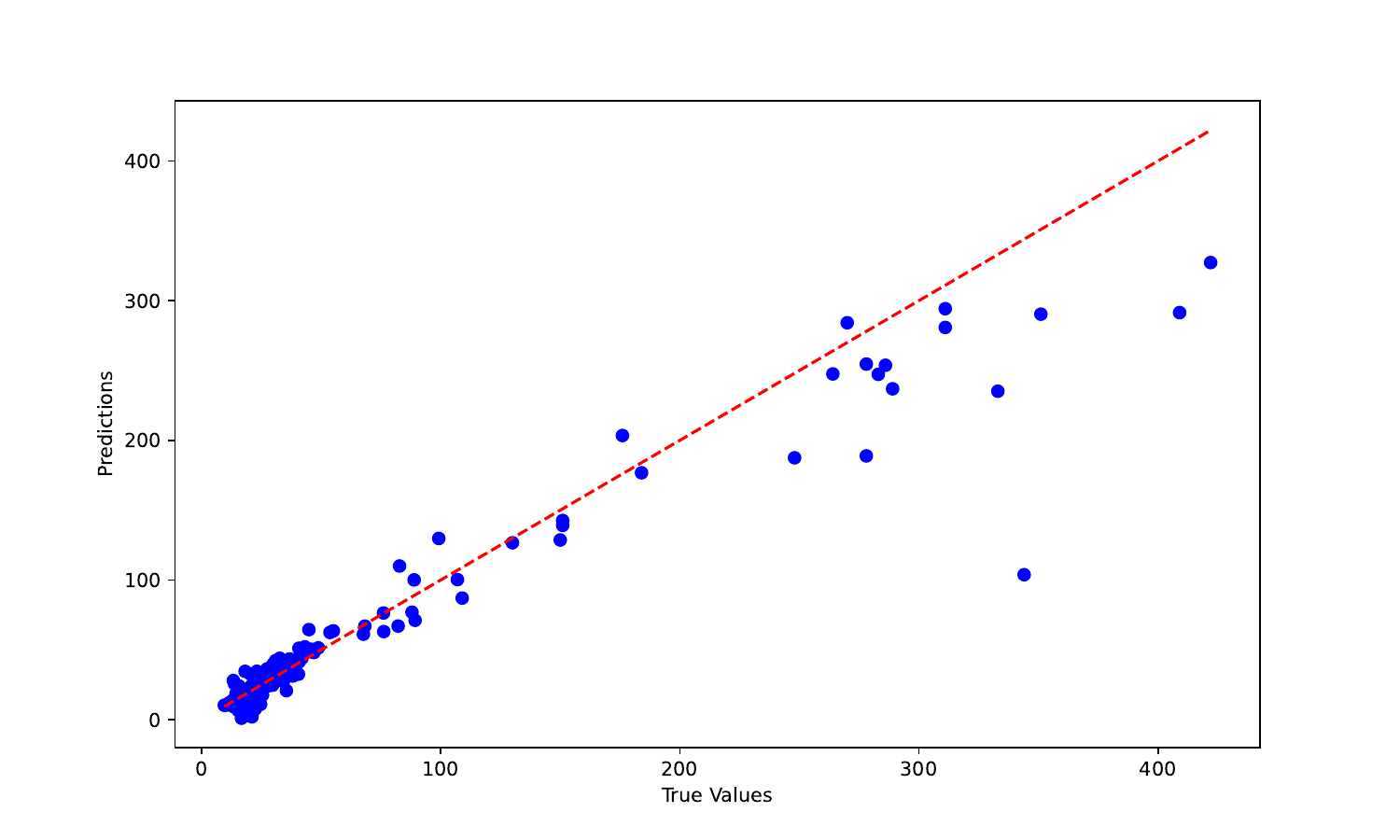}
\caption{Baseline approach, ERM algorithm}
\end{subfigure}
\newline
\begin{subfigure}[b]{.5\linewidth}
\includegraphics[width=1\textwidth]{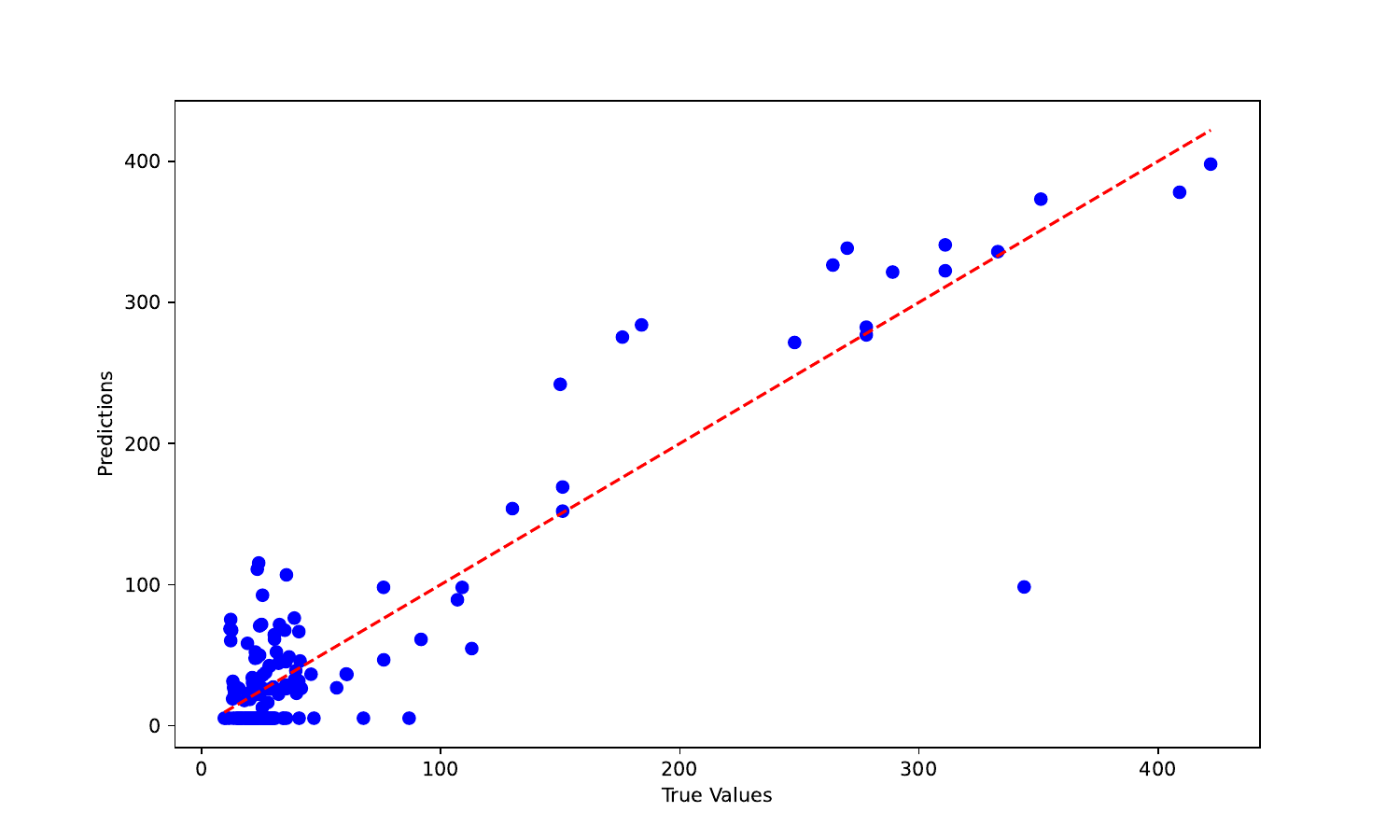}
\caption{Our approach, Algorithm \ref{algo:SVM} (SVM)}
\end{subfigure}
\begin{subfigure}[b]{.5\linewidth}
\includegraphics[width=1\textwidth]{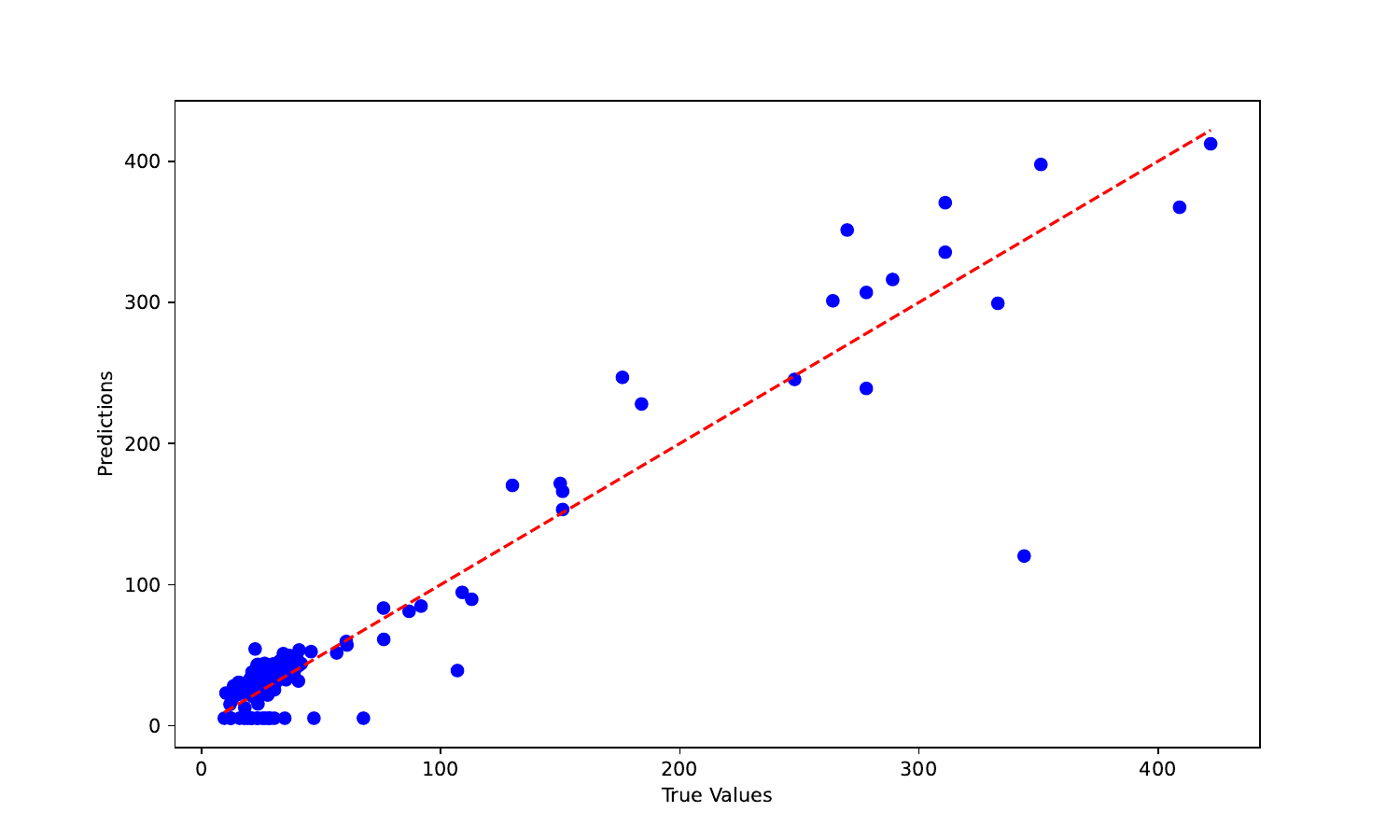}
\caption{Our approach, Algorithm \ref{algo:ERM} (ERM)}
\end{subfigure}
\caption{Predictions against true values for different approaches, in the original scale of the dataset, on the $1 \%$ most extreme observations in the test set.}
\label{fig:comparison_approaches}
\end{figure}

\appendix
\section{Proofs of the results in Section \ref{sec:framework}}
\subsection{Proof of Proposition \ref{lemme quantile}}

The proof of Proposition~\ref{lemme quantile} relies on the following lemma, which shows, in a univariate setting without covariates that the modified pinball loss elicits the quantile.
\begin{lemma} \label{lemma_1_quantile} 
    Let $Y$ be any real-valued random variable. The $\tau$-quantile $q$ of $Y$ minimizes the modified pinball risk, that is,
    \begin{equation*}
        \EE[\pinbl'(Y,q)] = \min_{z \in \R} \EE[\pinbl'(Y,z)]
    \end{equation*}
\end{lemma}
\begin{proof}
    The technique of proof is similar to that of Proposition 3.9 in \cite{christmann2008support}, except that we drop the assumption that $Y$ is integrable by considering the modified version of the pinball loss. We fix $z>0$, and compute the associated inner risk. If $y < 0$, we also have $y < z$ and by the definition of the modified pinball loss,
    \begin{align*}
        \pinbl'(y,z)& = \pinbl(y,z) - \pinbl(y,0)\\
        &=(1-\tau)(z-y) - (1-\tau)(-y)\\
        &= (1-\tau)z.
    \end{align*}
    By a similar calculation, we find that, if $0 \leq y < z$:
    \begin{equation*}
        \pinbl'(y,z) = (1-\tau)z - y
    \end{equation*}
    and, if $y\geq z$,
    \begin{equation*}
        \pinbl'(y,z) = -\tau z.
    \end{equation*}
    Denote $Q$ the law of $Y$. Since the pinball loss is $\max(\tau,1-\tau)$ with respect to its second variable, it holds that for any $(y,z) \in \R^2$,
    \begin{equation*}
        \lvert \pinbl'(y,z) \rvert = \lvert \pinbl(y,z) - \pinbl(y,0) \rvert \leq \max(\tau,1-\tau) \lvert z \rvert
    \end{equation*}
    which is integrable with respect to $Q$ the distribution of $Y$. Moreover, with the calculations above, we obtain, for any $z>0$,
    \begin{equation}
    \begin{aligned} \label{eq:lemma1_calculation1}
        \EE[\pinbl'(Y,z)]
        &=\int_{y < 0} (1-\tau)z \ud Q(y) 
        + \int_{0 \leq y < z} \big((1-\tau)z - y\big) \ud Q(y)
        + \int_{y \geq z} - \tau z \ud Q(y)\\
        &= (1-\tau)z Q((-\infty,0)) + (1-\tau)z Q([0,z)) 
        - \int_{0 \leq y < z} y \ud Q(y) 
        -\tau z Q([z,+\infty))\\
        &= (1-\tau)z Q((-\infty,z)) - \int_{0 \leq y < z} y \ud Q(y) - \tau z Q([z,+\infty))\\
        &= zQ((-\infty,z)) -\tau z - \int_{0 \leq y < z} y \ud Q(y)
    \end{aligned}
    \end{equation}
    where the last line follows follows from the fact that $Q([z,+\infty)) = 1- Q((-\infty,z))$. Now, notice that
    \begin{equation}
    \begin{aligned}\label{eq:lemma1_calculation2}
        \int_{0 \leq y < z} y \ud Q(y) &= \EE[\un_{Y \in [0,z)} Y]\\
        &= \int_{s = 0}^\infty \PP(Y \un _{Y \in [0,z)} \geq s) \ud s\\
        &= \int_{s \in [0,z)} Q([s,z))\ud s\\
        &= \int_{s \in [0,z)} Q((-\infty,z)) - Q((-\infty,s))\ud s\\
        &= z Q((-\infty,z)) - \int_{s \in [0,z)} Q((-\infty,s))\ud s.
    \end{aligned}
    \end{equation}
    Combining equations \eqref{eq:lemma1_calculation1} and  \eqref{eq:lemma1_calculation2}, we obtain
    \begin{align*}
        \EE[\pinbl'(Y,z)] &= \int_{s \in [0,z)} Q((-\infty,s)) \ud s - \tau z\\
        &= \int_{s \in [0,z)} (Q((-\infty,s)) - \tau) \ud s
    \end{align*}
    To address the case where $z \leq 0$, observe that $\EE[\pinbl'(Y,z)] = \EE[\ell_{1-\tau}'(-Y,-z)]$. Consequently, applying the previously established result by substituting $z$, $\tau$, and $Y$ with $-z$, $1-\tau$, and $-Y$, respectively, yields for any $z < 0$ that
    \begin{align*}
        \EE[\pinbl'(Y,z)] &= \int_{s \in [0,-z)} Q((-s,+\infty)) - (1-\tau) \ud s\\
        &= \int_{s \in [0,-z)} \tau - Q((-\infty,-s]) \ud s\\
        &= \int_{u \in (z,0]} \tau - Q((-\infty,u]) \ud u
    \end{align*}
    Now, assume that the generalized quantile $q$ of $Y$ is non-negative. Then, for any $z<0$,  and any $u \in (z,0]$, we have $u \leq q$, thus $Q((-\infty,u]) \leq \tau$ by the definition of  $q$ as the generalized quantile. It follows that, for any $z\leq 0$, 
    \begin{equation*}
        \EE[\pinbl'(Y,z)] = \int_{u \in (z,0]} \tau - Q((-\infty,u]) \ud u \geq 0.
    \end{equation*}
    Now, for any $s \in \R$, we have $Q((-\infty,s)) < \tau $ for any $s<q$ and $Q((-\infty,s)) \geq \tau$ for any $s \geq q$. It follows that, for any $z\geq 0$
    \begin{equation*}
        \EE[\pinbl'(Y,z)] = \int_{s \in [0,z)} (Q((-\infty,s)) - \tau) \ud s
    \end{equation*}
    has a minimum for $z=q$, and the corresponding risk is non-positive. Hence, in the case $q\geq 0$, the risk $\EE[\pinbl'(Y,z)]$ is minimized by the generalized quantile $q$. This reasoning extends analogously to the case where the generalized quantile $q$ is negative.
\end{proof}    
\begin{proof}[Proof of Proposition~\ref{lemme quantile}]
    For any function $f$ with $\EE[\lvert f(X)\rvert]<\infty$, we have
    \begin{equation*}
        \EE[\pinbl'(Y,f(X))] = \mathbb{E}[\EE[\pinbl'(Y,f(x)) \given X=x]]
    \end{equation*}
    Applying Lemma \ref{lemma_1_quantile} to the conditional law of $Y$ given $X=x$, we obtain that the inner risk $\EE[\pinbl'(Y,f(x)) \given X=x]$
    is minimized when $f(x) = q(x)$ where $q$ is the conditional quantile function. Moreover, since the coniditional function satisfies $\EE[\lvert q(X) \rvert]<\infty$, it follows that
    \begin{equation*}
        \EE[\pinbl'(Y,q(X))] = \inf \{\EE[\pinbl'(Y,f(X))] : f \text{ measurable, } \EE[\lvert f(X)\rvert] < \infty\}.
    \end{equation*}
\end{proof}

\subsection{Proofs of Lemma~\ref{lemma:angularity-extreme quantile} and Theorem~\ref{Main theorem proba} }
\begin{proof}[Proof of Lemma \ref{lemma:angularity-extreme quantile}]
    Recall that within the regular variation framework, the pair $(Y_\infty, \Theta_\infty)$ is independent from $\lVert X_\infty\rVert$. Therefore,
    \begin{equation*}
    \left(Y_{\infty} \given \Theta_\infty, \lVert X \rVert_\infty \right) \overset{\law}{=}  \left(Y_\infty \given  \Theta_\infty\right)
    \end{equation*}
    So that the conditional quantile of $Y_\infty$ given $X_\infty=x$ is the conditional quantile of $Y_\infty$ given $\Theta_\infty = \theta(x)$. Thus, the extreme conditional quantile $q_{\Plim}$ is of angular type and we may write $q _{\Plim} = h_{\Plim} \circ \theta$ for a certain measurable function $h_{\Plim} \colon \sphere \to \R$. Moreover, Assumption \ref{assumption regularity} ensures we can choose $h_{\Plim}$ to be continuous.
\end{proof}
\begin{proof}[Proof of Theorem \ref{Main theorem proba}]
    \begin{enumerate}
    \item Let $f = h \circ \theta$ be as in the theorem. $h$ is continuous on the compact set $\mathbb{S}$ so it is bounded. In particular, $f$ is bounded and its risk is well defined. Moreover, for any $(\theta,y) \in \sphere \times \R$,
    \begin{align*}
        \lvert \pinbl'(y,h(\theta)) \rvert 
        &= \lvert \pinbl(y,h(\theta)) - \pinbl(y,0) \rvert\\
        &\leq \max(\tau,1-\tau) \lVert h \rVert_\infty.
    \end{align*}
    Consequently, the function $(\theta,y) \mapsto \pinbl'(y,h(\theta))$ is bounded and continuous.
    Now, Assumption \ref{assumption regular variation} implies the convergence in law 
   \[\left(Y,\theta(X) \; \middle| \; \lVert X \rVert \geq t\right) \overset{\law} {\tto}(Y_\infty,\Theta_\infty).\]

    Thus, we have the convergence
    \[\risk_t(f) = \EE[\pinbl'(Y,h\circ\theta(X)) \; \middle| \; \lVert X \rVert \geq t] \tto \EE[\pinbl'(Y_\infty,h(\Theta_\infty))] = R_{P_\infty}(f).\]

    \item First, note that
    \begin{equation*}
        \EE[\lvert q(X) \rvert \given \lVert X \rVert \geq t] \leq \EE[\lvert \quantiletau(X) - \quantileinf(X) \rvert \given\lVert X \rVert \geq t] + \EE[\lvert \quantileinf(X) \rvert \given \lVert X \rVert \geq t]
    \end{equation*}
    Lemma \ref{lemma:angularity-extreme quantile} ensures that $\quantileinf$ is bounded, and Assumption \ref{assumption regularity} ensures that 
    \begin{equation*}
        \EE[\lvert \quantiletau(X) - \quantileinf(X) \rvert \given\lVert X \rVert \geq t] \tto 0.
    \end{equation*}
    Hence, for $t$ sufficiently large, we have $\EE[\lvert \quantiletau(X) - \quantileinf(X) \rvert \given\lVert X \rVert \geq t] < \infty$ and
    \begin{equation*}
         \EE[\lvert q(X) \rvert \given \lVert X \rVert \geq t] \leq \EE[\lvert \quantiletau(X) - \quantileinf(X) \rvert \given\lVert X \rVert \geq t] + \EE[\lvert q_{\Plim}(X) \rvert \given \lVert X \rVert \geq t] < \infty.
    \end{equation*}
    It follows that the conditional Bayes risk is $\risk_t^*=\risk_t(q)$ for $t$ large enough. Since $q_{\Plim}$ is bounded, we also have that the extreme Bayes risk is $\risk_{\Plim}^* = \risk_{\Plim} (q_{\Plim})$.
    Now, for $t$ large enough such that $\EE[\lvert q(X) \rvert \given \lVert X \rVert \geq t] < \infty$, we have
    \begin{align*}
        \risk_t^* &= \EE\left[\pinbl'(Y,q(X)) \; \middle| \; \lVert X \rVert \geq t\right]\\
        &= \EE\left[\pinbl'(Y,\quantileinf(X))\; \middle| \; \lVert X \rVert \geq t\right] + \EE\left[\pinbl'(Y,q(X)) - \pinbl'(Y,\quantileinf(X))\; \middle| \; \lVert X \rVert \geq t\right]\\
        &= \risk_t(\quantileinf) + \EE\left[\pinbl'(Y,q(X)) - \pinbl'(Y,\quantileinf(X))\; \middle| \; \lVert X \rVert \geq t\right].
    \end{align*}

    To deal with the first term, recall that from Lemma $\ref{lemma:angularity-extreme quantile}$, we know that the function $\quantileinf$ is angular. Property 1 thus implies $\risk_t(\quantileinf) \tto \risk_{P_\infty}(\quantileinf)$.

    To show that the second term vanishes, we use the fact that the function $\pinbl'$ is $\max(\tau,1-\tau)$-Lipschitz in each variable, which yields
    \begin{align*}
        \left\lvert \EE[\pinbl'(Y,\quantiletau(X)) - \pinbl'(Y,\quantileinf(X)) \given \lVert X \rVert \geq t]\right\rvert &\leq 
        \EE[ \lvert \pinbl'(Y,\quantiletau(X)) - \pinbl'(Y,\quantileinf(X)) \rvert \given \lVert X \rVert \geq t]\\
        &\leq \max(\tau,1-\tau) \EE[ \lvert \quantiletau(X) - \quantileinf(X) \rvert \given \lVert X \rVert \geq t].
    \end{align*}
    Assumption \ref{assumption regularity} ensures that $\EE\left[ \lvert \quantiletau(X) - \quantileinf(X) \rvert\; \middle| \; \lVert X \rVert \geq t\right] \tto 0$.

    \item By Property 2, we have $\risk^*_t \tto \risk_{P_\infty}^*$, and, since $\risk_\infty^* = \risk_\infty(q) = \limsup_{t \to \infty} \risk_t^*$, we obtain
    \begin{equation*}
        \risk_\infty^* = \risk_{P_\infty}^*.
    \end{equation*}

    \item Property 3 ensures that $\risk_\infty^* =\risk_{P_\infty}(\quantileinf)$.
    Now, since $\quantileinf$ is angular, we deduce from Property 1 that $\risk_\infty(\quantileinf) = \underset{t \to \infty}{\limsup}~\risk_t(\quantileinf) = \risk_{P_\infty}(\quantileinf) = \risk_{\Plim}^*$.
    \end{enumerate}
\end{proof}

\subsection{Proofs for Section \ref{sec:examples}} \label{sec:appendix_proofs_examples}
\begin{proof}[Proof for Example~\ref{example:noisy_target}]
    The proof is analogous to the proof of Example 2.A in \cite{huet2023regression}. We begin by proving that Assumption \ref{assumption regular variation} is satisfied. Let $r \colon \R^d \times \R \to \R$ be a bounded and Lipschitz function, and $t> 0$. We have
    \begin{align*}
        \EE[r(X/t, Y) \given \lVert X \rVert \geq t] &= \EE[r(X/t, g_\theta(\theta(X)) + \varepsilon) \given \lVert X \rVert \geq t]\\
        &+  \EE[r(X/t, g_\theta(\theta(X)) + \varepsilon) \given \lVert X \rVert \geq t] - \EE[r(X/t,g(X) + \varepsilon) \given \lVert X \rVert \geq t]\\
    \end{align*}
    Because $X$ is regularly varying and $\varepsilon$ is independent from $X$, we have
    \begin{align*}
        \EE[r(X/t, g_\theta(\theta(X)) + \varepsilon) \given \lVert X \rVert \geq t] \tto \EE[r(X_\infty, g_\theta(\Theta_\infty) + \varepsilon)].
    \end{align*}
    And, since $r$ is $C$-Lipschitz,
    \begin{align*}
        &\big\lvert  \EE[r(X/t,g(X) + \varepsilon) \given \lVert X \rVert \geq t] - \EE[r(t/tX, g_\theta(\theta(X)) + \varepsilon) \given \lVert X \rVert \geq t]\big\rvert\\
        \leq \ &\EE[\lvert  r(X/t,g(X) + \varepsilon) - r(X/t, g_\theta(\theta(X)) + \varepsilon) \rvert \given \lVert X \rVert \geq t]\\
        \leq \ & C \EE[\lvert g(X) - g_\theta(\theta(X))\rvert \given \lVert X \rVert \geq t]\\
        \leq \ & C \sup_{\lVert x \rVert \geq t} \lvert g(x) - g_\theta(\theta(x)) \rvert
    \end{align*}
    which vanishes as $t$ tends to infinity because of the condition \eqref{assumption_example-noisy-target}. We obtain that, for any bounded Lipschitz function $r$,
    \begin{equation*}
        \EE[r(X/t, Y) \given \lVert X \rVert \geq t] \xrightarrow[t \to \infty]{} \EE[r(X_\infty, g_\theta(\Theta_\infty) + \varepsilon)].
    \end{equation*}
    Combined with Portemanteau's theorem, this shows the convergence in distribution
    \begin{equation*}
        \left(X/t,g_\theta(\theta(X)) + \varepsilon \given \lVert X \rVert \geq t\right) \xrightarrow[t\to \infty]{w} \left(X_\infty, g_\theta(\Theta_\infty)\right)
    \end{equation*}
    and Assumption \ref{assumption regular variation} is satisfied.
    To prove that Assumption \ref{assumption regularity} also holds, denote by $q_\varepsilon$ the $\tau$-quantile of $\varepsilon$. By independence of the variables $X$ and $\varepsilon$, the following equality holds in distribution:
    \begin{equation*}
        \left(g(X) + \varepsilon \given X=x \right) \overset{(d)}{=} g(x) + \varepsilon
    \end{equation*}
    Thus, the conditional $\tau$-quantile of $Y= g(X) + \varepsilon$ given $X=x$ is $q(x) = g(x) + q_\varepsilon$. Similarly, the $\tau$-conditional of $Y_\infty = g_\theta(\Theta_\infty) + \varepsilon$ given $X_\infty = x$ is $q_{\Plim} = g_\theta(\theta(x)) + q_\varepsilon$. It follows from the continuity of $g_\theta$ that $q_{\Plim}$ is continuous. Now, we have, for any $x \in \R^d$,
    \begin{equation*}
        \lvert q(x) - q_{\Plim}(x) \rvert = \lvert g(x) - g_\theta(\theta(x)) \rvert
    \end{equation*}
    which implies
    \begin{align*}
        \EE[\lvert \quantiletau(X) - \quantileinf(X) \rvert \;\middle|\; \lVert X \rVert \geq t] &= \EE[\lvert g(X) - g_\theta(\theta(X)) \rvert \;\middle|\; \lVert X \rVert \geq t]\\
        &\leq \sup_{\lVert x \rVert \geq t} \lvert g(x) - g_\theta(\theta(x)) \rvert \tto 0.
    \end{align*}
\end{proof}

\begin{remark}
    In \cite{huet2023regression}, a more general result is provided, for a target $Y$ of the form $Y = g(X,\varepsilon)$ for some continuous function $g$. Unfortunately, in the context of quantile regression, we do not obtain such a general result. Indeed, for least-squares regression, the convenient conditional expectation expression $\EE[Y \given X=x]$ of the Bayes predictor greatly simplifies the theoretical analysis. The absence of an analogous closed-form expression for quantiles complicates the proofs.
\end{remark}
\begin{proof}[Proof of Proposition \ref{thm:ex_missing_component}] This proof requires several lemmas. We begin by showing that Assumption \ref{assumption regular variation} holds, that is, the pair $(X,Y)$ is regularly varying.

\begin{lemma} \label{ex:lem_regular_variation}
    In the framework described in \ref{example:prediction_missing_component}, the pair $(X,Y)$ has joint density $p(x,y) = \lVert x \rVert \pi(x,\lVert x \rVert y)$. This density is first component-regularly varying with limit density $p_\infty(x,y) = \lVert x \rVert \pi_\infty(x,\lVert x \rVert y)$ and same scaling function $b(t)$, \ie
    \begin{equation}
    \sup_{\lVert x \rVert \geq 1, y \in \R} \lvert b(t) t^d p(tx,y) - p_\infty(x,y) \rvert \tto 0.
    \end{equation}
\end{lemma}

\begin{proof}[Proof of Lemma \ref{ex:lem_regular_variation}]
    The steps for the proof are provided in the proof of Proposition 4.1 in \cite{clemenccon2026weak}. Their result could not be applied directly here because we do not assume that the target $\frac{Z}{\lVert X \rVert}$ is bounded, which is a needed assumption in the statements of Proposition 4.1. Nevertheless, it turns out that this boundedness assumption is not needed for the statement of Lemma \ref{ex:lem_regular_variation}.
    Regular variation for the pair $(X,Z)$ implies regular variation in the classical sense. More precisely, regular variation of densities implies regular variation of the distribution \citep[see \emph{e.g.}][\cite{cai2011estimation}]{}, that is, there exists a limit pair $(\widetilde{X}_\infty,\widetilde{Z}_\infty)$ such that 
        \begin{equation*}
            \left(t^{-1}(X,Z) \given \lVert (X,Z) \rVert \geq t \right) \xrightarrow[t \to \infty]{\law} (\widetilde{X}_\infty,\widetilde{Z}_\infty).
        \end{equation*}
        Where $(\widetilde{X}_\infty,\widetilde{Z}_\infty)$ has density $\pi_\infty$, which is homogeneous of order $-\alpha -d - 1$ because of the convergence \eqref{ex:regular_variation_density}. This implies that the marginal density of $\widetilde{X}$, defined by $\pi_\infty(x) = \int_\R \pi_\infty(x,z) \ud z$ is also homogeneous, of order $-\alpha -d$. Thus, the random variable $\widetilde{X}$ satisfies:
        \begin{align*}
            \PP(\lVert \widetilde{X} \rVert > 1) 
            &= \int_{\lVert x \rVert > 1} \pi_\infty(x) \ud x\\
            &= \int_{\lVert x \rVert > 1} \lVert x \rVert^{-\alpha - d} \pi_\infty(\theta(x)) \ud x\\
            &\geq c \int_{\lVert x \rVert > 1} \lVert x \rVert^{-\alpha-d} \ud x > 0
        \end{align*}
        where the last line follows from the lower bound of the marginal of the sphere, given by the condition \eqref{ex:lower_bound_marginal}. We can now apply Proposition 4.1, item 1 of \cite{clemenccon2026weak} which ensures that we also have the convergence in distribution when conditioning on $\|X\|$ large,
        \begin{equation*}
            \left(t^{-1}(X,Z) \given \lVert X \rVert \geq t \right) \xrightarrow[t \to \infty]{w} (X_\infty, Z_\infty),
        \end{equation*}
        where        \begin{equation}\label{ex:def_X_inty_Zinfty}
            (X_\infty,Z_\infty) \overset{(d)}{=} \left((\widetilde{X}_\infty, \widetilde{Z}_\infty) \given \lVert \widetilde{X}_\infty \rVert \geq 1\right).
        \end{equation}
        and we also have
        \begin{equation*}
            \left((X/t,Y) \given \lVert X \rVert \geq t \right) \xrightarrow[t \to \infty]{w} (X_\infty,Y_\infty)
        \end{equation*}
        with $(X_\infty,Y_\infty) \overset{(d)}{=} (X_\infty, \frac{Z_\infty}{\lVert X_\infty \rVert})$.
        The formula $p(x,y) = \lVert x \rVert \pi(x,\lVert x \rVert y)$ follows from a simple change of variable formula. To ensure that the convergence condition  \eqref{ex:regular_variation_density} holds, it suffices to show the uniform convergence on $\sphere \times \R$ (see \cite{de1987regular}, proof of Theorem 2.1), and we have
        \begin{align*}
            \sup_{\omega \in \sphere, y \in \R} \lvert b(t) t^d p(t\omega,y) - p_\infty(\omega,y) \rvert
            &= \sup_{\omega \in \sphere, y \in \R} \lvert b(t) t^{d+1} \pi(t \omega,ty) - \pi_\infty(\omega,y) \rvert\\
            &\leq \sup_{\lVert (x,z) \rVert \geq 1} \lvert b(t) t^{d+1} \pi(tx,tz) - \pi_\infty(x,z) \rvert \tto 0
        \end{align*}
        where the last inequality follows from the fact that $\lVert (\omega,y) \rVert \geq 1$ for any $\omega \in \sphere$. 
        \end{proof}
\noindent In the following, we will denote by $p_{Y\lvert X}(y\lvert x)$ the conditional density of $Y$ given $X=x$ and $p_{Y_\infty\lvert X_\infty}(y\lvert x)$ the conditional density of $Y_\infty$ given $X_\infty=x$. To check that Assumption \ref{assumption regularity} also holds, we begin by showing that the function $\omega \mapsto q_{\Plim}(\omega)$ is continuous on the sphere in a separate lemma.
\begin{lemma} \label{lem:ex_continuity_quantilefunc}
        Recall that $q_{\Plim}(\omega)$ is the $\tau$-quantile of $Y_\infty$ given $\theta(X_\infty) = \omega$. Assume that the limit marginal density $\pi_\infty(x)$ is lower bounded on the sphere \ie there exists $c> 0$ such that $\pi_\infty(\omega) \geq c$ for any $\omega \in \sphere$. Assume also that $\lVert \cdot \rVert = \lVert \cdot \rVert_p$ for some $p \in [1,+\infty]$. Then, the function $q_{\Plim}$ is continuous.
    \end{lemma}

    \begin{proof}[Proof of Lemma \ref{lem:ex_continuity_quantilefunc}]
        Let $\omega \in \sphere$ and consider the extreme conditional cumulative distribution function 
    \begin{equation*}
        F_\infty(s\lvert \omega) = \PP\left(Y_\infty \leq s \given X_\infty = \omega\right).
    \end{equation*}
    We first prove that  for any $s \in \R$, the mapping $\omega \in \sphere \mapsto F_\infty(s\lvert \omega) = \int_{-\infty}^s p_{\infty, Y \lvert X}(\omega,y) \ud y$ is continuous. Note first that the convergence 
        \begin{equation*}
    \sup_{\lVert(x,z) \rVert > 1} \lvert t^{d+1} b(t) \pi(tx,tz) - \pi_\infty(x,z) \rvert \tto 0.
\end{equation*}
    implies that $\pi_\infty$ is homogeneous of order $-\alpha-d-1$. Now, the function $(\omega,y) \in \sphere \times \R \mapsto p_{\infty, Y \lvert X}(\omega,y)$ is continuous (by continuity of $p_\infty$) and, for any $\omega \in \sphere$, we have
    \begin{align*}
        p_{Y_\infty \lvert X_\infty}(y \lvert \omega)
        &=  \frac{\pi_\infty(\omega,y)}{\pi_\infty(\omega)}\\
        &\leq \frac{\lVert (\omega,y) \rVert^{-\alpha - d - 1} \pi_\infty(\frac{(\omega,y)}{\lVert (\omega,y) \rVert})}{c}\\
        &\leq \frac{\lVert (\omega,y) \rVert^{-\alpha - d - 1} \sup_{x \in \sphere_{d+1}} \pi_\infty(x)}{c}
    \end{align*}
    where we used both the homogeneity of $\pi_\infty$ and the lower bound on its marginal on the sphere. Now, for any $p \in [1,\infty)$, we have
    \begin{equation*}
        \lVert(\omega,y) \rVert_p^p = \lVert \omega \rVert_p^p + \lvert y \rvert^p = 1 +  \lvert y \rvert^p
    \end{equation*}
    where we have used the fact that $\omega \in \sphere$ for the last equality. We obtain
    \begin{equation*}
         p_{Y_\infty \lvert X_\infty}(y \lvert \omega)
        \leq \frac{(1+\lvert y \rvert^p)^{-\frac{\alpha + d + 1}{p}} \sup_{\omega \in \sphere_{d+1}} \pi_\infty(\omega)}{c}
    \end{equation*}
    which is a bound independent from $\omega$, and defines an integrable function of $y$. Indeed,
    \begin{equation*}
    (1+\lvert y \rvert^p)^{-\frac{\alpha + d + 1}{p}} \underset{\lvert y \rvert \to \infty}{\sim} \lvert y \rvert^{-\alpha -d - 1} 
    \end{equation*}
    which is integrable at $\pm \infty$. If the case where $p=\infty$, we simply write
    \begin{equation}
        \lVert (\omega, y) \rVert_\infty^{-\alpha -d -1} \leq (1 + \lvert y \rvert)^{-\alpha -d-1}
    \end{equation}
    hence the same argument ensuring the domination of the density by an integrable function of $y$ also holds.
    By dominated convergence, we deduce that the function $\omega \mapsto F_\infty(s\lvert\omega)$ is continuous. Then Proposition 0.1 in \cite{resnick2008extreme} ensures that the weak convergence of non decreasing functions implies that of their generalized inverse. In particular, if $\omega_n \to \omega$, we have $q_{\Plim}(\omega_n) \to q_{\Plim}(\omega)$, hence $q_{\Plim}$ is continuous.
    \end{proof}
    
    \noindent We now check that the stability condition \eqref{eq:assumption_2} holds, and begin with an explicit expression for $p_{Y\lvert X}$ and $p_{Y_\infty \lvert X_\infty}$ as functions of $\pi$ and $\pi_\infty$ respectively. First, by definition of the conditional density, for any $x \in \R^d \setminus \{0\}$ and $y \in \R$,
    \begin{align*}
        p_{Y\lvert X}(y\lvert x) &= \frac{p(x,y)}{p(x)} = \frac{\lVert x \rVert \pi(x,\lVert x \rVert y)}{\int_\R \lVert x \rVert \pi(x,\lVert x \rVert y) \ud y} = \frac{\lVert x \rVert \pi(x,\lVert x \rVert y)}{\int_\R  \pi(x,u) \ud u} = \dfrac{\lVert x \rVert \pi(x, \lVert x \rVert y)}{\pi(x)}.
    \end{align*}
    For the conditional density $p_{Y_\infty \lvert X_\infty}$, note first the homogeneity of $\pi_\infty$ of order $-\alpha-d-1$ implies that $p_\infty$ is homogeneous of order $-\alpha-d$ with respect to the first component, \ie for any $x \in \R^d\setminus \{0\}$ and $t > 0$:
    \begin{equation*}
        p_\infty(tx,y) = t^{-\alpha-d} p_\infty(x,y).
    \end{equation*}
    The joint density of $(X_\infty, Z_\infty)$ is proportional to $\pi_\infty(x,z)$ on $\{x \in \R^d \colon \lVert x \rVert \geq 1\} \times \R$. It follows that the joint density of $(X_\infty, Y_\infty)$ is proportional to $p_\infty$ on $\{x \in \R^d \colon \lVert x \rVert \geq 1\} \times \R$, thus
    \begin{align*}
        p_{Y_\infty\lvert X_\infty}(y\lvert x) = \frac{p_\infty(x,y)}{\int_\R p_\infty(x,y) \ud y},
    \end{align*}
    so, by homogeneity of $p_\infty$,
    \begin{align*}
        p_{Y_\infty\lvert X_\infty}(y\lvert x) &= \frac{\lVert x \rVert^{-\alpha - d} p_\infty(\theta(x),y)}{\lVert x \rVert^{-\alpha - d} \int_\R  p_\infty(\theta(x),y) \ud y} = \frac{\pi_\infty(\theta(x),y)}{\pi_\infty(\theta(x))}.
    \end{align*}
    The following lemma then shows the uniform convergence of the conditional probability densities and their associated conditional distribution functions.

    \begin{lemma} \label{lem:ex_unif_cv}~
    \begin{enumerate}
        \item The conditional density of $Y$ given $X = t \omega$ converges to the conditional density of $Y_\infty$ given $X_\infty = \omega$ uniformly on the sphere, \ie
        \begin{equation*}
            \sup_{\omega \in \sphere} \lvert p_{Y\lvert X}(y\lvert t\omega) - p_{Y_\infty\lvert X_\infty}(y,\omega) \rvert \tto 0 
        \end{equation*}
        \item Denote by $F_t(s\lvert\omega) = \PP\left(Y\leq t \given X = t \omega \right)$ the conditional cumulative distribution function of $Y$ given $X = t \omega$. Also denote $F_\infty(s\lvert \omega) = \PP\left(Y_\infty \leq s \given X_\infty = \omega\right)$ the extreme conditional cumulative distribution function. The conditional cumulative function $F_t$ converges to the extreme conditional cumulative distribution function $F_\infty$, uniformly on the sphere, as $t \to \infty$, \ie
    \begin{equation*}
        \sup_{s \in \R, \omega \in \sphere} \lvert F_t(s \lvert \omega) - F_\infty(s \lvert \omega) \rvert  \tto 0.
    \end{equation*}
    \end{enumerate}
\end{lemma}
\begin{proof}[Proof of Lemma \ref{lem:ex_unif_cv}]
    We use several results from \cite{huet2023regression}. More precisely, we use Lemma B.1 and the computation done in the proof of Proposition 2.2 (item $(c)$). Note that both results are exploitable here because they do not require the boundedness of the target $Y$, and the statements can be proven only working with the densities $\pi, \pi_\infty$. This ensures that Lemma B.1 in \cite{huet2023regression} also applies in our setting, as well as the computation in the proof of Proposition 2.2.
    \begin{enumerate}
        \item From Lemma B.1 in \cite{huet2023regression}, we have the convergence of the marginals uniformly on the sphere,
        \begin{equation*}
            \sup_{\omega \in \sphere} \lvert t^d b(t) \pi(t\omega) - \pi_\infty(\omega) \rvert \tto 0.
        \end{equation*}
        Thus, we have the convergence
        \begin{equation*}
            \frac{t \pi (t \omega,ty)}{\pi(\omega)} = \frac{t^{d+1} b(t) \pi(t\omega,ty)}{t^d b(t) \pi(t\omega)} \tto \frac{\pi_\infty(\omega,y)}{\pi_\infty(\omega)}.
        \end{equation*}
        Moreover, this convergence is uniform on the sphere because both the numerator and the denominator converge uniformly on the sphere and the denominator is greater than $\frac{\inf_{\omega \in \sphere} \pi_\infty(\omega)}{2} > 0$ for $t$ large enough.
        \item Let $t > 0$. For any $s \in \R$, We have:
    \begin{align*}
        \sup_{\omega \in \sphere} ~~\lvert F_t(s \lvert \omega) - F_\infty(s \lvert \omega) \rvert &= \sup_{\omega \in \sphere} \Big\lvert \int_{-\infty}^s p_{Y\lvert X}(y\lvert t \omega) \ud y - \int_{-\infty}^s p_{Y_\infty \lvert X_\infty}(y\lvert \omega) \ud y\Big\rvert\\
        &= \sup_{\omega \in \sphere}\Big \lvert \int_{-\infty}^s \big(\frac{t \pi(t \omega,ty)}{\pi(t\omega)} - \frac{\pi_\infty(\omega,y)}{\pi_\infty(\omega)} \big) \ud y \Big\rvert \\
        &= \sup_{\omega \in \sphere} \Big \lvert \int_{-\infty}^s \big(\frac{t^{d+1} b(t) \pi(t \omega,ty)}{t^d b(t)\pi(t\omega)} - \frac{\pi_\infty(\omega,y)}{\pi_\infty(\omega)} \big) \ud y \Big\rvert\\
        &\leq \int_\R \sup_{\omega \in \sphere} \Big \lvert \frac{t^{d+1} b(t) \pi(t \omega,ty)}{t^d b(t)\pi(t\omega)} - \frac{\pi_\infty(\omega,y)}{\pi_\infty(\omega)} \Big\rvert \ud y 
    \end{align*}
    which is independent from $s$. It has been shown in \cite{huet2023regression} (proof of Proposition 2.2, equation 2.19) that this integral converges to $0$ as $t \to \infty$, which ends the proof.
    \end{enumerate}
    \end{proof}  
\noindent We now have all the necessary tools to show that the condition \eqref{eq:assumption_2} in Assumption \ref{assumption regularity} holds. We actually show the stronger condition
        \begin{equation} \label{ex:cond_ass2_sup}
            \sup_{\lVert x \rVert \geq t} \lvert q(x) - q_{\Plim}(x) \rvert \tto 0.
        \end{equation}
        \textbf{Step 1:} Let $x = t \omega$. We first establish a lower bound on the conditional density $p_{Y\lvert X}(y,t\omega)$ which holds on a specified domain. Lemma \ref{lem:ex_continuity_quantilefunc} ensures that by $q_{\Plim}$ is continuous on the sphere, hence it is bounded. Thus, there exists $M>0$ such that, for any $\omega \in \sphere$, $\lvert q_{\Plim}(\omega) \rvert \leq M$. since $\sphere \times [-M-1,M+1]$ is compact, by the continuity of the extreme conditional density, we can define
        \begin{equation*}
            c= \min_{(\omega,y) \in \sphere \times [-M-1,M+1]} p_{Y_{\infty}\lvert X_{\infty}}(y\lvert \omega) 
        \end{equation*}
        and $c>0$ by positivity of $p_{Y_{\infty}\lvert X_{\infty}}$. Now, by the uniform convergence of the conditional densities given by Lemma \ref{lem:ex_unif_cv}, it holds that, for $y \in [-M-1,M+1]$ and any $\omega \in \sphere$,
        \begin{align*}
            p_{Y\lvert X}(y,t\omega) &= p_{Y_\infty \lvert X_\infty}(y,\omega) + (p_{Y\lvert X}(y,t\omega) - p_{Y_\infty \lvert X_\infty}(y,\omega))\\
            &\geq p_{Y_\infty \lvert X_\infty}(y,\omega) - \sup_{\omega \in \sphere} \big\lvert p_{Y\lvert X}(y,t\omega) - p_{Y_\infty \lvert X_\infty}(y,\omega) \big\rvert
        \end{align*}
        if $t$ is large enough so that $\sup_{\omega \in \sphere} \lvert p_{Y\lvert X}(y,t\omega) - p_{Y_\infty \lvert X_\infty}(y,\omega) \rvert \leq \frac{c}{2}$, we obtain, for any $\omega \in \sphere$ and $y \in [-M-1,M+1]$:
        \begin{equation}\label{ex:lowerbound_density}
            p_{Y\lvert X}(y,t\omega) \geq \frac{c}{2}
        \end{equation}
        \textbf{Step 2:} We now use the uniform convergence  of the conditional cumulative distribution function. We have, from Lemma \ref{lem:ex_unif_cv},
        \begin{equation*}
            \sup_{\omega \in \sphere} \lvert F_t(q_{\Plim}(\omega)\lvert\omega) - F_\infty(q_{\Plim}(\omega) \lvert \omega) \rvert \leq \sup_{s \in \R, \omega \in \sphere} ~~\lvert F_t(s \lvert \omega) - F_\infty(s \lvert \omega) \rvert \tto 0
        \end{equation*}
        Moreover, by positivity of the conditional densities, we have $F_\infty(q_{\Plim}(\omega) \lvert \omega) = \tau = F_t(q(t\omega) \lvert \omega)$. We obtain
        \begin{equation}\label{ex:proximity_cdfs}
            \sup_{\omega \in \sphere} \lvert F_t(q_{\Plim}(\omega)\lvert\omega) - F_t(q(t\omega) \lvert \omega) \rvert \tto 0.
        \end{equation}
        \textbf{Step 3:} We now show that \eqref{ex:lowerbound_density} along with \eqref{ex:proximity_cdfs} imply that $q_{\Plim}(\omega)$ and $q(t\omega)$ are close for $t$ large enough. We first show that $q(t\omega) \in [-M-1,M+1]$ if $t$ is large enough. Assume by contradiction that, for any $T>0$ there exists $t\geq T$  and $\omega \in \sphere$ such that $\lvert q(t\omega) \rvert > M+1$. Assume for simplicity that $q(t\omega) > M+1$, and take $T$ large enough so that the lower bound on the conditional density $\eqref{ex:lowerbound_density}$ holds. Then, 
        \begin{align*}
            F_t(q(t\omega)) - F_t(q_{\Plim}(\omega)) &\geq F_t(M+1) - F_t(q_{\Plim}(\omega))\\
            &= \int_{q_{\Plim}(\omega)}^{M+1} p_{Y\lvert X}(y, t\omega) \ud y\\
            &\geq \frac{c}{2}(M+1-q_{\Plim}(\omega))\\
            &\geq \frac{c}{2}
        \end{align*}
        which contradicts the convergence \eqref{ex:proximity_cdfs}. Thus, there exists $T>0$ such that, for $t \geq T$ and any $\omega \in \sphere$, $q(t\omega) \in [-M-1,M+1]$. We can now leverage the lower bound on the conditional density $p_{Y\lvert X}$ provided by \eqref{ex:lowerbound_density}, which yields
        \begin{align*}
            \lvert F_t(q(t\omega)) - F_t(q_{\Plim}(\omega)) \rvert &=\Big \lvert \int_{q_{\Plim}(\omega)}^{q(t\omega)} p_{Y\lvert X}(y,t\omega) \ud y\Big \rvert\\
            &\geq \frac{c}{2} \lvert q(t\omega) - q_{\Plim}(\omega)\rvert
        \end{align*}
        and we obtain
        \begin{equation*}
            \sup_{\omega \in \sphere} \big \lvert q(t \omega) - q_{\Plim}(\omega) \rvert \leq \frac{2}{c} \sup_{\omega \in \sphere} \lvert F_t(q_{\Plim}(\omega)\lvert\omega) - F_t(q(t\omega) \lvert \omega) \big \rvert \tto 0.
        \end{equation*}
        Thus, it also holds that $\sup_{s \geq t} \sup_{\omega \in \sphere} \lvert q(t\omega) - q_{\Plim}(t \omega) \rvert \tto 0$, and 
        \begin{align*}
            \sup_{\lVert x \rVert \geq t} \lvert q(x) - q_{\Plim}(x) \rvert &= \sup_{\lVert x \rVert \geq 1} \lvert q(tx) - q_{\Plim}(tx) \rvert\\
            &= \sup_{s \geq t} \sup_{\lVert x \rVert = 1} \lvert q(sx) - q_{\Plim}(sx) \rvert \tto 0
        \end{align*}
        which shows that the condition \eqref{ex:cond_ass2_sup} is satisfied. 
    \end{proof}

\section{Proofs of the results in Section \ref{sec:stat_ERM}}
\subsection{Proof of Proposition \ref{prop:thershold_bias_vanishes}}
\label{appendix:proof_threshold_bias}
This is an adaptation of Proposition 3.2 in \cite{huet2023regression} to the case of quantile regression with the modified pinball loss. For any $h_1,h_2 \in \mathcal{F}$, we have, because the modified pinball loss is $\max(\tau,1-\tau)$ Lipschitz in each variable,
    \begin{equation*}
        \lvert \risk_t(h_1 \circ \theta) - \risk_t(h_2\circ \theta) \rvert \leq \max(\tau,1-\tau) \lVert h_1 - h_2 \rVert_\infty
    \end{equation*}
    and the same holds for $\risk_{\Plim}$,
    \begin{equation*}
        \lvert \risk_{\Plim}(h_1 \circ \theta) - \risk_{\Plim}(h_2\circ \theta) \rvert \leq \max(\tau,1-\tau) \lVert h_1 - h_2 \rVert_\infty.
    \end{equation*}
    By total boundedness of $\mathcal{F}$, for any $\varepsilon$, there exists $h_1,\dots, h_n$ such that $\mathcal{F} \subset \bigcup_{i=1}^n B(h_i,\varepsilon, \lVert \cdot \rVert_\infty)$ where $B(h_i,\varepsilon, \lVert \cdot \rVert_\infty)$ is the ball centered in $h_i$ and of radius $\varepsilon$ in the space $(\mathcal{C}(\sphere), \lVert \cdot \rVert_\infty)$. Because Assumption \ref{assumption regular variation} holds, we can apply Theorem \ref{Main theorem proba}, item 1, which ensures that, for any $i \in \{1,\dots,n\}$, 
    \begin{equation*}
        \risk_t(h_i \circ \theta) \tto \risk_{\Plim}(h_i \circ \theta).
    \end{equation*}
    Thus, for $t>T$ large enough, we have $\lvert \risk_t(h_i \circ \theta) - \risk_{\Plim}(h_i \circ \theta) \rvert \leq \varepsilon$ for any $i \in \{1,\dots,n\}$. Now, given $h \in \mathcal{F}$, there exists $i \in \{1,\dots,n\}$ such that $\lVert h - h_i \rVert_\infty \leq \varepsilon$. Thus, for $t \geq T$,

    \begin{align*}
        \lvert \risk_t(h \circ \theta) - \risk_{\Plim}(h \circ \theta) \rvert 
        &\leq \lvert \risk_t(h \circ \theta) - \risk_t(h_i \circ \theta) \rvert 
        + \lvert \risk_t(h_i \circ \theta) - \risk_{\Plim} (h_i \circ \theta) \rvert 
        + \lvert \risk_{\Plim}(h_i \circ \theta) - \risk_{\Plim}(h \circ \theta) \rvert\\
        &\leq (2\max(\tau,1-\tau) + 1) \varepsilon 
    \end{align*}
    This holds for any $h \in \mathcal{F}$, hence, for $t>T$,
    \begin{equation*}
        \sup_{h \in \mathcal{F}}\lvert \risk_t(h \circ \theta) - \risk_{\Plim}(h \circ \theta) \rvert \leq (2\max(\tau,1-\tau) + 1) \varepsilon.
    \end{equation*}
    Since $\varepsilon>0$ is arbitrary, we deduce the uniform convergence of the quantity above as $t \to \infty$. Note finally that $\risk_{\Plim}(h\circ \theta)$ can be replaced by $\risk_\infty(h\circ \theta)$ using item 1 of Theorem \ref{Main theorem proba}.

\subsection{Proof of Proposition \ref{prop:statistical_guarantee_ERM}} \label{sec:appendix_proof_stat_guarantee_ERM}

First, we split by the triangular inequality the supremum, introducing the intermediate risk $\riskinter(f \circ \theta)$, defined by equation \eqref{def:intermediate_risk},
    \begin{equation} \label{appendix_proof_ERM:splitting_risk}
        \sup_{f \in \ERMfamily} \lvert \risk_{\tnk}(f\circ \theta) - \emprisk_k(f \circ \theta) \rvert \leq \sup_{f \in \ERMfamily} \lvert \risk_{\tnk}(f\circ \theta) - \riskinter(f \circ \theta) \rvert + \sup_{f \in \ERMfamily} \lvert \riskinter(f\circ \theta) - \emprisk_k(f \circ \theta) \rvert.
    \end{equation}
We now provide a high probability bound on
\begin{equation*}
    \sup_{f \in \ERMfamily} \big\lvert \riskinter(f\circ \theta) - \emprisk_k(f \circ \theta) \big \rvert.
\end{equation*}
Notice first that by the uniform boundedness assumption (Assumption \ref{assumptionERM:uniform_bound}) on the class $\ERMfamily$, we have, for any $i \in \{1,\dots, n\}$,
\begin{equation*}
\lvert \pinbl'(Y_i,f(X_i)) \rvert \leq \max(\tau,1-\tau)\lvert f(X_i) \rvert \leq \max(\tau,1-\tau)M.
\end{equation*}
Applying the triangular inequality yields
\begin{equation}\label{eq:proof_ERM_riskhat_riskinter}
\begin{aligned}
        \sup_{f \in \mathcal{F}} \big \lvert \emprisk[k](f \circ \theta) - \riskinter(f \circ \theta) \big \rvert 
        &= \sup_{f \in \mathcal{F}} \frac{1}{k} \left\lvert \sum_{i=1}^n \pinbl'(Y_i,f(X_i))(\un_{\lVert X_i\rVert \geq \lVert X_{(k)}\rVert}-\un_{\lVert X_i\rVert \geq \tnk}) \right\rvert\\
        &\leq \frac{M \max(\tau,1-\tau)}{k} \sum_{i=1}^n \left\lvert \un_{\lVert X_i\rVert \geq\lVert X_{(k)}\rVert}-\un_{\lVert X_i\rVert \geq \tnk} \right\rvert.
    \end{aligned}
\end{equation}
Now, if $\tnk > \lVert X_{(k)}\rVert$, the previous sum counts all indices $i$ for which $\lVert X_i \rVert$ lies between $\lVert X_{k} \rVert$ (included) and $t$ (excluded), and there are exactly $\left\lvert\sum_{i=1}^n\un_{\lVert X_i \rVert \geq \tnk} -k \right\rvert$ such indices. 
Applying the same reasoning when $\tnk \leq \lVert X_{(k)}\rVert$ leads to the same result, thus,
\begin{equation}\label{eq:comparison_quantile_statistic_order}
        \sum_{i=1}^n \big\lvert \un_{\lVert X_{i} \rVert \geq \lVert X_{(k)} \rVert}  - \un_{\lVert X_i \rVert \geq \tnk} \big\rvert = \Big\lvert\sum_{i=1}^n\un_{\lVert X_i \rVert \geq \tnk} -k \Big\rvert.
    \end{equation}
The sum on the right-hand side of the equation can be bounded with high probability with Bernstein's inequality (see for instance \cite{bercu2015concentration}, Theorem 2.1 with Remark 2.4). Indeed, $\un_{\lVert X_i \rVert \geq \tnk}$ follows a binomial distribution with parameter $(n,\frac{k}{n}$), with $\EE[\un_{\lVert X_i \rVert \geq  \tnk}^2] = \frac{k}{n}$. By applying Bernstein's inequality along with a union bound argument, we obtain that, with probability no less than $1-2\delta$,
\[\left\lvert\sum_{i=1}^n\un_{\lVert X_i \rVert \geq \tnk} -k \right\rvert \leq \sqrt{2k\log(1/\delta)} + \frac{\log(1/\delta)}{3}.\]
Multiplying this inequality by $\frac{M\max(\tau,1-\tau)}{k}$, and replacing in \eqref{eq:proof_ERM_riskhat_riskinter}, we obtain that, with probability no less than $1-2\delta$,
\begin{equation} \label{appendix_proof:ERM_riskk_riskinter}
    \sup_{f \in \mathcal{F}} \big \lvert \emprisk[k](f \circ \theta) - \riskinter(f \circ \theta) \big \rvert \leq \frac{M \max(\tau,1-\tau)}{\sqrt{k}} \Big(\sqrt{2\log(1/\delta)} + \frac{\log(1/\delta)}{3\sqrt{k}}\Big).
\end{equation}
    To deal with the first term in \eqref{appendix_proof_ERM:splitting_risk}, we use a Talagrand type inequality (see \cite{bousquet2002bennett}, or the version in \cite{christmann2008support}, Theorem 7.5, with $\gamma = \frac{1}{2}$). We have, for any $f \in \ERMfamily$,
    \begin{align*}
        \big \lvert \riskinter(f \circ \theta) \big \rvert - \risk_{\tnk}(f\circ \theta) \rvert &= \frac{n}{k} \Big\lvert \frac{1}{n} \sum_{i=1}^n \pinbl'(Y_i,f(\theta(X_i))) \un\{\lVert X_i \rVert \geq \tnk\} - \EE[\pinbl'(Y_i,f(\theta(X_i))) \un\{\lVert X_i \rVert \geq \tnk\}] \Big\rvert\\
        &= \frac{n}{k} \Big\lvert \frac{1}{n} \sum_{i=1}^n g(X_i,Y_i) \Big\rvert
    \end{align*}
    where 
    \begin{equation}
        g(X_i,Y_i) = \pinbl'(Y_i,f(\theta(X_i))) \un\{\lVert X_i \rVert \geq \tnk\} - \EE[\pinbl'(Y_i,f(\theta(X_i))) \un\{\lVert X_i \rVert \geq \tnk\}].
    \end{equation}
    The function $g$ satisfies $\EE[g(X_1,Y_1)] = 0$. Moreover, for any $f \in \ERMfamily$, we have
    \begin{equation*}
        \lvert \pinbl'(Y,f(X)) \rvert 
        \leq \max(\tau,1-\tau) \lVert f\rVert_\infty \leq \max(\tau,1-\tau)M.
    \end{equation*}
    The second moment of $g(X_i,Y_i)$ can be bounded with
    \begin{align*}
        \EE[g(X_i,Y_i)^2] &= \Var(\pinbl'(Y_i,f(\theta(X_i))) \un_{\lVert X_i \rVert \geq \tnk})\\
        &\leq \EE[\pinbl'(Y_i,f(\theta(X_i)))^2 \un_{\lVert X_i \rVert \geq \tnk}]\\
        &\leq \max(\tau,1-\tau)^2 M^2 \frac{k}{n} = v.
    \end{align*}
    Note finally that the family $\ERMfamily$ is pointwise measurable, thus, we can restrict the supremum to $\ERMfamily_0$ which is at most countable, hence the supremum is measurable. We denote by $\psi_{n,k}$ the expectation of the supremum
    \begin{equation*}
        \psi_{n,k} = \frac{n}{k}\EE[\sup_{f \in \ERMfamily} \Big\lvert \frac{1}{n} \sum_{i=1}^n g((X_i,Y_i)) \Big\rvert ].
    \end{equation*}
    Talagrand's inequality yields that, with probability no less than $1-\delta$,
    \begin{align*}
        \sup_{f \in \ERMfamily} \Big\lvert \frac{1}{n} \sum_{i=1}^n g((X_i,Y_i)) \Big\rvert 
        &\leq \frac{3k}{2n}\psi_{n,k} + \sqrt{\frac{2v}{n}\log(1/\delta)} + \frac{8 M\max(\tau,1-\tau)\log(1/\delta)}{3n}.
    \end{align*}
    Multiplying both sides by $\frac{n}{k}$ and replacing $v$ by its value, we obtain
    \begin{equation*}
        \sup_{f \in \ERMfamily} \lvert \risk_{\tnk}(f\circ \theta) - \riskinter(f \circ \theta) \rvert \leq \frac{3}{2}\psi_{n,k} + M\max(\tau,1-\tau) \sqrt{\frac{2\log(1/\delta)}{k}}+ \frac{8M\max(\tau,1-\tau)\log(1/\delta)}{3k}.
    \end{equation*}
    We now provide a bound on the expectation $\psi_{n,k}$. We do exactly as in \cite{huet2023regression}, and introduce the Rademacher average
    \begin{equation*} \label{proof:ERM_riskt-riskinter_withpsi}
        \Rad(\ERMfamily,k) = \sup_{f \in \ERMfamily} \Big\lvert\frac{1}{k} \sum_{i=1}^n \sigma_i\pinbl'(Y_i,f(\theta(X_i))) \un_{\lVert X_i \rVert \geq \tnk}\Big\rvert
    \end{equation*}
    where the $\sigma_i$ are independent Rademacher variables. A standard symmetrization argument yields
    \begin{equation*}
        \EE[\sup_{f \in \ERMfamily} \lvert \risk_{\tnk}(f\circ \theta) - \riskinter(f \circ \theta) \rvert] \leq 2 \EE[\Rad(\ERMfamily,k)].
    \end{equation*}
    To deal with the Rademacher average, we use a conditioning lemma (see \cite{lhaut2022uniform}, Lemma 2.1), which ensures that
    \begin{equation*}
         \sum_{i=1}^n \sigma_i\pinbl'(Y_i,f(\theta(X_i))) \un_{\lVert X_i \rVert \geq \tnk} \overset{(d)}{=} \sum_{i=1}^{\mathcal{K}} \sigma_i\pinbl'(Y_i^k,f(\theta(X_i^k)))
    \end{equation*}
    where $(X_i^k,Y_i^k)$ are independent variables, independent from the variables $(X_i,Y_i)$ and from $\mathcal{K} \sim Bin(n,\frac{k}{n})$, and follow the same distribution as $\left(X,Y \given \lVert X \rVert \geq \tnk \right)$. Letting $\mathcal{G} = \{g(x,y) = \pinbl'(y,f\circ\theta(x)) \colon f \in \ERMfamily\}$, and considering any probability $Q$ on $\R^d \times \R$, we have
    \begin{align*}
        \lVert g_1-g_2 \rVert_{L^2(Q)} &= \sqrt{\EE_Q[(\pinbl'(Y,f_1\circ \theta(X)) - \pinbl'(Y,f_2\circ\theta(X)))^2]}\\
        &\leq \max(\tau,1-\tau)\sqrt{\EE_Q[(f_1(X)-f_2(X))^2]}\\
        &= \max(\tau,1-\tau) \lVert f_1 - f_2 \rVert_{L^2(Q_X \circ \theta^{-1})}
    \end{align*}
    where $Q_X$ is the marginal distribution of $X$. Denoting $\tilde{Q} = Q_X \circ \theta^{-1}$, we obtain that the covering number $\mathcal{N}(\mathcal{G},L_2(Q),\varepsilon)$ is bounded by $\mathcal{N}(\ERMfamily,L^2(\tilde{Q}), \frac{\varepsilon}{\max(\tau,1-\tau)})$. Since $\ERMfamily$ is of finite VC-dimension, with envelope function $M\un_{\sphere}$, Theorem 2.6.7 in \cite{vaart1997weak} yields
    \begin{equation*}
        \mathcal{N}(\ERMfamily, L^2(\tilde{Q}),M\varepsilon ) \leq \Big(\frac{A}{\varepsilon}\Big)^{2V_\ERMfamily}
    \end{equation*}
    where $V_\ERMfamily$ is the VC dimension of $\ERMfamily$, and $A$ is some universal constant. We deduce
    \begin{align*}
        \mathcal{N}(\mathcal{G}, L^2(Q), \varepsilon) &\leq \mathcal{N}\big(\ERMfamily, L^2(\tilde{Q}),\frac{\varepsilon}{\max(\tau,1-\tau)} \big)\\
        &\leq \Big(\frac{M\max(\tau,1-\tau)}{\varepsilon}\Big)^{2V_{\ERMfamily}}.
    \end{align*}
    Since the functions in $\mathcal{G}$ are uniformly bounded by $\max(\tau,1-\tau)M$, we obtain that $G = \max(\tau,1-\tau)M$ is an envelope function of $\mathcal{G}$ and
    \begin{equation*}
        \mathcal{N}(\mathcal{G}, L^2(Q), \varepsilon \lVert G \rVert_{L^2(Q)}) \leq \Big(\frac{A}{\varepsilon}\Big)^{2V_{\ERMfamily}}.
    \end{equation*}
    Now, applying Proposition 2.1 of \cite{gine2001consistency} along with $\sup_{g \in \mathcal{G}} \EE[g^2] \leq \max(\tau,1-\tau)^2 M^2$ and $\sup_{g\in \mathcal{G}} \lVert g \rVert_\infty \leq \max(\tau,1-\tau)M$, we obtain for any $m \geq 1$,
    \begin{equation*}
        \EE[\sup_{f \in \ERMfamily} \sum_{i=1}^m \sigma_i \pinbl'(Y_i^k,f\circ \theta(X_i^k))] \leq C'M\max(\tau,1-\tau)(V_\ERMfamily + \sqrt{mV_\ERMfamily}) 
    \end{equation*}
    for some universal constant $C'$. Thus,
    \begin{equation*}
        \EE\Big[\sup_{f \in \ERMfamily} \sum_{i=1}^{\mathcal{K}} \sigma_i\pinbl'(Y_i^k,f(\theta(X_i^k))) \; \big| \; \condK\Big]\leq  C'M\max(\tau,1-\tau)(V_\ERMfamily + \sqrt{\mathcal{K}V_\ERMfamily}).
    \end{equation*}
    We deduce
    \begin{align*}
        \EE[\Rad(\ERMfamily,k)] &= \EE\Big[\frac{1}{k}\EE\Big[\sup_{f \in \ERMfamily} \sum_{i=1}^m \sigma_i \pinbl'(Y_i^k,f\circ \theta(X_i^k)) \; \big| \; \condK \Big]\Big]\\
        &\leq \frac{C'M\max(\tau,1-\tau)}{k}\EE[V_\ERMfamily + \sqrt{\condK V_\ERMfamily}]\\
        &\leq \frac{C'M\max(\tau,1-\tau) V_\ERMfamily}{k} + \frac{C'M\max(\tau,1-\tau) \sqrt{V_\ERMfamily}}{\sqrt{k}}.
    \end{align*}
    Thus, for some universal constant $C^{''}$,
    \begin{equation}
        \psi_{n,k} \leq C^{''}M\max(\tau,1-\tau)\Big(\frac{V_\ERMfamily}{k} + \sqrt{\frac{V_\ERMfamily}{k}}\Big).
    \end{equation}
    Replacing $\psi_{n,k}$ by the previous bound in \eqref{proof:ERM_riskt-riskinter_withpsi}, we obtain that with probability no less than $1-\delta$, and some universal constant $C$,
    \begin{equation}\label{proof:ERM_riskt_riskinter}
    \begin{aligned}
        \sup_{f \in \ERMfamily} \lvert \risk_{\tnk}(f\circ \theta) - \riskinter(f \circ \theta) \rvert \leq \frac{M\max(\tau,1-\tau)}{\sqrt{k}}\Big(C\sqrt{V_\ERMfamily} + \sqrt{2\log(1/\delta)} + \frac{8 \log(1/\delta)}{3\sqrt{k}} + \frac{C V_\ERMfamily}{\sqrt{k}}\Big).
    \end{aligned}
    \end{equation}
    Combining  \eqref{appendix_proof:ERM_riskk_riskinter} with \eqref{proof:ERM_riskt_riskinter} and replacing in \eqref{appendix_proof_ERM:splitting_risk}, we obtain that, with probability no less than $1-3\delta$,
    \begin{equation*}
    \begin{aligned}
        \sup_{f \in \ERMfamily} \lvert \risk_{\tnk}(f\circ \theta) - \emprisk[k](f \circ \theta) \rvert \leq \frac{M\max(\tau,1-\tau)}{\sqrt{k}} \Big(C\sqrt{V_\ERMfamily} + 2\sqrt{2\log(1/\delta)}
        +\frac{3 \log(1/\delta)}{\sqrt{k}} + \frac{C V_\ERMfamily}{\sqrt{k}}\big).
    \end{aligned}
    \end{equation*}

\section{Proofs of the results in Section \ref{sec:stat_SVM}}
\subsection{Proof of the results in Section \ref{sec:SVM_preliminary}}
\label{appendix:proof_Existence_uniqueness_minimizer}
\begin{proof}[Proof of Lemma \ref{lem:Existence-uniqueness-minimizer}]
    \emph{Uniqueness}: By convexity of the pinball loss with respect to its first variable, we have, for any $h_1, h_2 \in \rkhs$ and $x,y \in \R^d$
    \begin{equation*}
        \pinbl'\big(\frac{h_1+ h_2}{2} \circ \theta(x),y\big) \leq \frac{\pinbl'(h_1 \circ \theta(x),y) + \pinbl'(h_1\circ \theta(x),y)}{2}
    \end{equation*}
    Integrating both sides with respect to the distribution $P$, we obtain
    \begin{equation*}
        \risk_P\big(\frac{h_1 + h_2}{2} \circ \theta\big) \leq \frac{1}{2}(\risk_P(h_1 \circ \theta) + \risk_P(h_2 \circ \theta))
    \end{equation*}
    that is, the function $h \mapsto \risk_P(h \circ \theta)$ is convex. Moreover, for any $\lambda>0$, the function $h \mapsto \lambda \normrkhs[h]^2$ is strictly convex. Hence, the map $h \mapsto \risk_P(h \circ \theta) + \lambda \normrkhs[h]^2$ is strictly convex, and the minimum is unique.

    \noindent \emph{Existence}: For any $h_1, h_2 \in \rkhs$, since the risk is $\max(\tau,1-\tau)$-Lipschitz, 
    \begin{equation}\label{eq:continuity-risk}
        \big \lvert \risk_P(h_1 \circ \theta) - \risk_P(h_2 \circ \theta) \big \rvert \leq \max(\tau,1-\tau)\lVert h_1 - h_2 \rVert_\infty
    \end{equation}
    Furthermore, since the kernel $K$ is bounded with $\sup K(x,x) \leq 1$, we have
    \begin{equation}\label{eq:inequality-normesup-normeH}
    \lVert h \rVert_\infty \leq \normrkhs[h].
    \end{equation}
    Replacing in \eqref{eq:inequality-normesup-normeH} in \eqref{eq:continuity-risk}, we obtain
    \begin{equation*}
        \big \lvert \risk_P(h_1 \circ \theta) - \risk_P(h_2 \circ \theta) \big \rvert \leq \max(\tau,1-\tau)\normrkhs[h_1-h_2]
    \end{equation*}
    that is to say, the function $h \in \rkhs \mapsto \risk_P(h \circ \theta)$ is $\max(\tau,1-\tau)$-Lipschitz. In particular, it is continuous, and therefore the map $h \in \rkhs \mapsto \lambda \normrkhs[h]^2 + \risk_P(h \circ \theta)$ is continuous.
    Now, consider the set 
    \begin{equation*}
        A = \{ h \in \rkhs \colon \lambda \normrkhs[h]^2 + \risk_P(h \circ \theta) \leq 0\}.
    \end{equation*}
    Since $\risk_P(0 \circ \theta) = 0$, it follows that $0 \in A$. Moreover, if $h \in A$, we have
    \begin{equation}\label{eq:estimee_norme_1}
    \begin{aligned}
        \lambda \normrkhs[h]^2 &\leq \lvert \risk_P(h \circ \theta) \rvert\\
        &\leq \max(\tau,1-\tau) \lVert h \rVert_\infty.
    \end{aligned}
    \end{equation}
    We obtain, replacing \eqref{eq:inequality-normesup-normeH} in \eqref{eq:estimee_norme_1}, 
    \begin{equation*}
        \lambda \normrkhs[h]^2 \leq \max(\tau,1-\tau) \lVert h \rVert_\rkhs.
    \end{equation*}
    Dividing both sides by $\lambda \lVert h \rVert_\rkhs$ yields
    \begin{equation*}
        \lVert h \rVert_\rkhs \leq \frac{\max(\tau,1-\tau)}{\lambda}.
    \end{equation*}
    Therefore, $A$ is a non empty bounded subset of the reflexive Banach space $\rkhs$. Then, (see Theorem A.6.9 in \cite{christmann2008support}) the map $h \in \rkhs \mapsto \lambda \normrkhs[h]^2 + \risk_P(h)$ has a minimum, achieved at $h^*_{\lambda,P}$. Moreover, this SVM solution satisfies, by the previous inequality,
    \begin{equation*}
        \normrkhs[h^*_{\lambda,P}] \leq \frac{\max(\tau,1-\tau)}{\lambda}.
    \end{equation*}
\end{proof}

\begin{proof}[Proof of Corollary \ref{cor:Bound_norm}]
       Define $\widehat{P}_k$ the empirical distribution of $(X_{(i)},Y_{(i)})_{1\leq i \leq k}$, $P_t$ the conditional law of $(X,Y)$ given $\lVert X \rVert \geq t$ and recall that $P_\infty$ is the joint law of $(X_\infty,Y_\infty)$ defined by \eqref{eq:Plim}. Recall that $((X_1, Y_1), \dots, (X_n,Y_n))$ are sampled from a distribution such that $\PP(X_i=0) = 0$ for any $i$, so that the empirical distribution $\widehat{P}_k$ satisfies the condition $\widehat{P}_k(0 \times \R) = 0$. This condition is also satisfied by $P_t$ since this distribution is supported on $\{(x,y) \in \R^d \times \R \colon \lVert x \rVert \geq t\}$, and the probability distribution $\Plim$ defined by $\eqref{eq:Plim}$ also satisfies $\Plim(0 \times \R) = 0$. Hence, Lemma \ref{lem:Existence-uniqueness-minimizer} applies to these 3 probability distributions, and yields the existence of $\hhlambda$, $h_{\lambda,t}^*$ and $h_{\lambda,\infty}^*$ respectively, along with the bound on their norm.

       \noindent For the intermediate SVM solution $\htildelambda$, denote by $N = \sum_{i=1}^n \un_{\lVert X_i \rVert \geq \tnk}$ the number of observations above the threshold $\tnk$. If $N=0$, the intermediate risk defined by \eqref{def:intermediate_risk} is identically equal to 0, and in this case $\htildelambda = 0$. Otherwise, denote by $\widetilde{P}_k$ the probability distribution associated to the observations $(X_i,Y_i)$ where $\lVert X_i \rVert \geq \tnk$,
       \begin{equation*}
           \widetilde{P}_k = \frac{1}{N} \sum_{i=1}^n \delta_{(X_i,Y_i)} \un_{\lVert X_i \rVert \geq \tnk}.
       \end{equation*}
       This defines a random probability distribution supported on $\{(x,y) \in \R^d \times \R \colon \lVert x \rVert \geq \tnk\}$. Lemma \ref{lem:Existence-uniqueness-minimizer} applies and the minimizer of 
       \begin{equation*}
           \frac{1}{N} \sum_{i=1}^n\pinbl'(h(\theta(X_i)),Y_i) + \lambda' \normrkhs[h]^2
       \end{equation*}
       is well-defined, for any $\lambda' > 0$. Now, multiply this penalized risk by $\frac{N}{k}$, and we obtain that the minimizer of 
       \begin{equation*}
           \riskinter(h \circ  \theta) + \lambda' \normrkhs[h]^2 \frac{N}{k}
       \end{equation*}
       is well defined for any $\lambda>0$. In particular, taking $\lambda' = \lambda \frac{k}{N}$ yields the existence and uniqueness of $\htildelambda$.
       
    \end{proof}
\begin{proof}[Proof of Proposition \ref{prop:restriction_RKHS_SVM_solution}]
        Notice first that the RKHS-norm of the SVM solution $h^*_{\lambda,P}$ is equal to the RKHS-norm of its restriction as an element of $\rkhs_\sphere$,
        \begin{equation} \label{eq:equality_norm_rkhs_rkhs_sphere}
            \normrkhs[h^*_{\lambda,P}] = \lVert h^*_{\lambda,P}\lvert_\sphere \rVert_{\rkhs_\sphere}.
        \end{equation}
        Indeed, recall that by definition of the norm $\lVert \cdot \rVert_{\rkhs_\sphere}$, 
        \begin{equation*}
             \lVert h^*_{\lambda,P}\lvert_\sphere \rVert_{\rkhs_\sphere} = \min \{ \normrkhs[h] \colon h\lvert_\sphere = h^*_{\lambda,P}\lvert_\sphere\} \leq \normrkhs[h^*_{\lambda,P}].
        \end{equation*}
        Assume by contradiction that $\lVert h^*_{\lambda,P}\lvert_\sphere \rVert_{\rkhs_\sphere} < \normrkhs[h^*_{\lambda,P}]$. Then there exists $h \in \rkhs$ such that $h\lvert_\sphere = h^*_{\lambda,P}\lvert_\sphere$ and $\normrkhs[h] < \normrkhs[h^*_{\lambda,P}]$. Since both functions coincide on the sphere, their conditional risk are equal, and we obtain
        \begin{align*}
            \riskPen[P](h \circ \theta) &= \risk[P](h \circ \theta) + \lambda \normrkhs[h]^2 \\
            &= \risk[P](h^*_{\lambda,\tnk} \circ \theta) + \lambda \normrkhs[h]^2\\
            &< \riskPen[P](h^*_{\lambda,\tnk} \circ \theta)
        \end{align*}
        which contradicts the definition of $h^*_{\lambda,P}$ as the minimizer of the penalized conditional risk. Now, take $h^*_\sphere$ the element of the RKHS $\rkhs$ such that $h^*_\sphere \lvert_\sphere = h^*_{\lambda,P,\sphere}$ and $\lVert h^*_\sphere \rVert_{\rkhs} = \lVert h^*_{\lambda,P,\sphere}\rVert_{\rkhs_\sphere}$. Proposition \ref{def:rkhs'} ensure that such an element exists and is unique. By definition of $h^*_{\sphere}$, we have
        \begin{align*}
            \risk_{P,\sphere}^{(\lambda)}(h^*_{\lambda,P,\sphere} \circ \theta) &= \risk[P](h^*_\sphere) + \lambda \normrkhs[h^*_\sphere]^2.
        \end{align*}
        By definition of $h^*_{\lambda,P,\sphere}$ as the minimizer of the penalized conditional risk over $\rkhs_\sphere$, we obtain
        \begin{align*}
             \risk[P](h^*_\sphere) + \lambda \normrkhs[h^*_\sphere]^2 &\leq \risk[P](h^*_{\lambda,P}\lvert_\sphere \circ \theta) + \lambda \lVert h^*_{\lambda,P} \lvert_\sphere \rVert_{\rkhs_\sphere}\\
             &= \risk[P](h^*_{\lambda,P} \circ \theta) + \lambda \lVert h^*_{\lambda,P}\rVert_{\rkhs}
        \end{align*}
        where the last line follows from the equality \eqref{eq:equality_norm_rkhs_rkhs_sphere}. Hence, $h^*_\sphere$ also minimizer the penalized conditional risk, that is, $h^*_\sphere = h^*_{\lambda,P}$. It follows, restricting to the sphere, that $h^*_{\lambda,P,\sphere} = h^*_{\lambda,P}\lvert_\sphere$.
    \end{proof}
\subsection{Proof the results in Section \ref{sec:stat_svm_debut}}
\begin{proof}[Proof of Theorem \ref{thm:SVM_bias_variance_dec}]
    We start with the  decomposition
    \[
    \risk_\infty(\hhlambda \circ \theta) + \lambda \normrkhs[\hhlambda]^2 - \risk_\infty^*=\mathrm{I}+\mathrm{II}+\mathrm{III}
    \]
    where
    \begin{equation*}
    \begin{aligned}
        \mathrm{I}&= \risk_{\tnk}(h^*_{\lambda,\tnk} \circ \theta) + \lambda \normrkhs[h^*_{\lambda,\tnk}]^2 - \risk_\infty^*,\\
        \mathrm{II}&=\risk_\infty(\hhlambda \circ \theta) - \risk_{\tnk}(\hhlambda \circ \theta),\\
        \mathrm{III}&=\Big(\risk_{\tnk}(\hhlambda \circ \theta) + \lambda \normrkhs[\hhlambda]^2\Big) - \Big( \risk_{\tnk}(h^*_{\lambda,\tnk} \circ \theta) + \lambda \normrkhs[h^*_{\lambda,\tnk}]^2\Big).
    \end{aligned}
    \end{equation*}
    For term $\mathrm{III}$, we notice that
    \begin{equation*}
        \mathrm{III}
        = \riskPen_{\tnk}(\hhlambda \circ \theta) - \riskPen_{\tnk}(h^*_{\lambda,\tnk} \circ \theta)
        = D(n,k,\lambda).
    \end{equation*}
    For term $\mathrm{II}$, we know from Corollary \ref{cor:Bound_norm} that $\lVert \hhlambda \rVert_{\rkhs_\sphere} \leq \frac{\max(\tau,1-\tau)}{\lambda}$, so that we can deduce
    \begin{equation*}
        \mathrm{II}=\risk_\infty(\hhlambda \circ \theta) - \risk_{\tnk}(\hhlambda \circ \theta) \leq \sup_{\normrkhs[h] \leq \max(\tau,1-\tau) \lambda^{-1}}\big \lvert \risk_{\tnk}(h \circ \theta) - \risk_\infty(h \circ \theta) \big \rvert.
    \end{equation*}
    Finally for term $\mathrm{I}$, we use the fact that $h^*_{\lambda,\tnk}$ minimizes the penalized conditional risk so that its penalized conditional risk is smaller than that of $h^*_{\lambda,\infty}$. Hence
    \begin{equation*}
    \begin{aligned}
        \mathrm{I}&=\risk_{\tnk}(h^*_{\lambda,\tnk} \circ \theta) + \lambda \normrkhs[h^*_{\lambda,\tnk}]^2 - \risk_\infty^*\\ 
        &\leq  \risk_{\tnk} (h^*_{\lambda,\infty} \circ \theta) + \lambda \normrkhs[h^*_{\lambda,\infty}]^2 - \risk_\infty^*\\
        &= \underbrace{\risk_{\infty}(h^*_{\lambda,\infty} \circ \theta) + \lambda \normrkhs[h^*_{\lambda,\infty}]^2 - \risk_\infty^*}_{\mathrm{I}_a} + \underbrace{\risk_{\tnk}(h^*_{\lambda,\infty} \circ \theta) - \risk_\infty(h^*_{\lambda,\infty} \circ \theta)}_{\mathrm{I}_b}
    \end{aligned}
    \end{equation*}
    For  term $\mathrm{I}_a$, we immediately recognize the asymptotic approximation error of the RKHS $\rkhs$:
    \begin{equation*}
        \mathrm{I}_a = \approxerror_2^\infty(\lambda).
    \end{equation*}
    Finally, we know from Corollary \ref{cor:Bound_norm} that $\lVert h^*_{\lambda,\infty} \rVert_{\rkhs_\sphere} \leq \frac{\max(\tau,1-\tau)}{\lambda}$, so that
    \begin{equation*}
        \mathrm{I}_b\leq \sup_{\normrkhs[h] \leq \max(\tau,1-\tau)\lambda^{-1}}\big \lvert \risk_{\tnk}(h \circ \theta) - \risk_\infty(h \circ \theta) \big \rvert.
    \end{equation*}
    Combining all of the previous inequalities yields the result.
\end{proof}
\begin{proof}[Proof of Proposition \ref{prop:BiasControl}]
    Let $R = \frac{\max(\tau,1-\tau)}{\lambda} > 0$. We want to prove
    \begin{equation*}
        \sup_{h \in R \ball_{\rkhs}} \Big \lvert \risk_{t}(h \circ \theta) - \risk_\infty(h \circ \theta) \Big \rvert \tto 0.
    \end{equation*}
    Since $h\circ \theta$ depends only on the restriction to the sphere  $h_{|\sphere}$, we can work in the RKHS $\rkhs_\sphere$. Using also $\lVert h\lvert_\sphere \rVert_{\rkhs_\sphere} \leq \normrkhs[h]$, we obtain
    \begin{equation*}
        \sup_{h \in R \ball_{\rkhs}} \Big \lvert \risk_{t}(h \circ \theta) - \risk_\infty(h \circ \theta) \Big \rvert  \leq \sup_{h \in R \ball_{\rkhs_\sphere}} \Big \lvert \risk_{t}(h \circ \theta) - \risk_\infty(h \circ \theta) \Big \rvert.
    \end{equation*}
    Since the sphere $\sphere \subset \R^d$ is compact, it follows from Corollary 4.31 in \cite{christmann2008support} that the inclusion $\rkhs_{\sphere} \to \mathcal{C}(\sphere)$ is compact. Thus, the $\lVert \cdot \rVert_\infty$-closure of $\ball_{\rkhs_\sphere}$ is a compact subset of $\mathcal{C}(\sphere)$. In particular, $R \ball_{\rkhs_\sphere}$ is totally bounded in the space $(\mathcal{C}(\sphere),\lVert \cdot \rVert_\infty)$. We conclude by applying Proposition \ref{prop:thershold_bias_vanishes}.
\end{proof}

\begin{proof}[Proof of Corollary \ref{cor:threshold_bias_vanishes_SVM}]
  Define a  function of two variables, 
  \begin{equation*}
  \beta(n, \lambda) := B(n, k(n), \lambda), ~~\lambda>0. 
  \end{equation*}
  From the definition of the bias term \eqref{eq:threshold_bias}, the function $\beta$ is clearly non increasing in its second argument, and from the previous proposition and the fact that $k(n)/n \to 0$, for any fixed $\lambda > 0$, $\beta(n,\lambda)$ converges to $0$ as $n \to \infty$. We need to construct a sequence $\lambda(n)>0$ such that $\lambda(n) \to 0$ and $\beta(n,\lambda(n)) \to 0$.
  Consider the subset
  \begin{equation*}
      \Lambda_n = \{\lambda>0 \colon  \beta(n,\lambda) \leq \lambda \}.
  \end{equation*}
First, we claim that $\Lambda_n$ is nonempty. If $\Lambda_n = (0,+\infty)$ then it is clearly nonempty, otherwise we can choose $\lambda > 0$ such that $\beta(n,\lambda) > \lambda$. Then for $x\ge \beta(n,\lambda)$, we have by
  monotonicity,
  $$
\beta(n,x) \le  \beta(n,\lambda) \le x, 
  $$
  thus  $x \in \Lambda_n$. In addition, the set $\Lambda_n$ is clearly lower bounded by 0. Consequently, the function
  \begin{equation*}
      \psi(n) = \inf(\Lambda_n)
  \end{equation*}
  is valued in $[0,+\infty)$. We now show that for all $n$,  the set $\Lambda_n$ is  an interval of the kind $(\psi(n),+\infty)$ or $[\psi(n),+\infty)$. Indeed since $\Lambda_n$ is nonempty, let $\lambda\in \Lambda_n$ and let  $\lambda'>\lambda$.  By monotonicity we have 
$$
\beta(n,\lambda') \leq \beta(n,\lambda) \leq \lambda < \lambda', 
$$
proving that $\lambda' \in \Lambda_n$ too and that $\Lambda_n$ is an interval. Now, define $\lambda(n) = \psi(n) + \frac{1}{n}$ so that, for all $n$, $\lambda(n) \in \Lambda_n$.  Then,
$$
\beta(n, \lambda(n)) \le \lambda(n) = \psi(n) +1/n. 
$$
It remains only to  prove 
that $\psi(n) \to 0$ as $n\to\infty$.
We proceed by contradiction: assume that $\psi(n)$ does not converge to $0$. Then  for some $\varepsilon>0$,  and some subsequence  $\phi(n)\to\infty$, we have  $\psi\circ\phi(n)>\varepsilon$, for all $n$. Thus $\varepsilon \notin \Lambda_{\phi(n)}$, which means
\begin{equation*}
    \beta(\phi(n),\varepsilon) > \varepsilon ,
\end{equation*}
contradicting the fact that $\beta(\phi(n), \varepsilon) \to 0$ as $n\to \infty$. 
 \end{proof} 
\subsection{Proof of the results in Section \ref{sec:stat_bound_distance}}
\label{appendix:proof_bernstein_hh_h^*}
\begin{proof}[Proof of Lemma \ref{lem:distance_hh_htilde}]
The proof shares some similarities with the one of Theorem 5.9 (see also Theorem 5.8) in \cite{christmann2008support}. Since our setting involves restricting the empirical distributions to regions above  order statistics and true quantiles, respectively, we find it simpler to propose a full proof rather than attempting to apply the above mentioned results. The intuition behind this result is linked with the stability of SVM solutions. Given a distribution $P$ on $\R^d \times \R$, consider $h_{\lambda,P}^*$ the SVM solution obtained by minimizing 
\begin{equation*}
    \empriskPen_P(h \circ \theta) = \emprisk_P(h \circ \theta) + \lambda \normrkhs[h]^2.
\end{equation*}
If $P$ and $P'$ be two probability distribution on $\R^d \times \R$, results concerning the stability of SVM solutions typically show that the distance between the two resulting SVM solutions $h^*_{\lambda,P}$ and $h^*_{\lambda,P'}$ can be bounded by  some term quantifying the distance between $P$ and $P'$. Here, we are considering instead the two SVM solutions $\htildelambda$ and $\hhlambda$. Recall that $\emprisk_k$ is associated with the empirical measure of the sample $(X_{(i)},Y_{(i)})_{1\leq i\leq k}$ but that  $\riskinter$ cannot be directly interpreted as a risk against a probability measure. Still, we can prove that some similar property holds, and provide on $\normrkhs[\hhlambda -\htildelambda]$ close depending on how well the quantile $t_{n,k}$ is estimated by $\lVert X_{(k)} \rVert$ the $k$-th statistic order of $\lVert X \rVert$.

First, note that the risk functions $\riskinter[k,\sphere] \colon h \in \rkhs \mapsto \riskinter(h \circ \theta)$ and $\emprisk[k,\sphere] \colon h \in \rkhs \mapsto \emprisk[k](h \circ \theta)$ are convex functions from $\rkhs$ to $\R$, by convexity of the pinball loss. Moreover, the function $h \in \rkhs \mapsto \normrkhs[h]^2$ is differentiable and convex. Therefore, the regularized risks $\riskinterPen[k] \colon h \in \rkhs \mapsto \lambda \normrkhs[h]^2 + \widetilde{\risk}_{t,\sphere}(h)$ and $\empriskPen_k \colon h \in \rkhs \mapsto \lambda \normrkhs[h]^2 + \emprisk[k,\sphere](h)$ are both subdifferentiable, and their subgradient verifies
    \begin{equation*}
        \partial \riskinterPen[k](h) = \partial \riskinter[k,\sphere](h) + 2\lambda \iota h,
    \end{equation*}
    and
    \begin{equation*}
        \partial \empriskPen[k](h) = \partial \emprisk[k,\sphere](h) + 2\lambda \iota h,
    \end{equation*}
    where $\iota \colon \rkhs \to \rkhs'$ is the Fréchet-Riesz isomorphism that maps $\rkhs$ to its dual $\rkhs'$.
    Now, by definition of $\htildelambda$ and $\hhlambda$, we have $ 0 \in \partial \riskinterPen[k](\htildelambda)$ and $ 0 \in \partial \empriskPen_{k}(\hhlambda)$. Thus, there exists $\widetilde{w} \in \partial \riskinter[k,\sphere](\htildelambda)$ and $\widehat{w} \in \partial \emprisk[k,\sphere](\hhlambda)$ such that:
    \begin{equation*}
        \widetilde{w} + 2\lambda \iota \htildelambda = 0
    \end{equation*}
    and
    \begin{equation*}
        \widehat{w} + 2\lambda \iota \hhlambda = 0.
    \end{equation*}
    Now, by definition of the subgradient, we have
    \begin{equation*}
        \langle \widehat{w}, \htildelambda-\hhlambda \rangle_{\rkhs',\rkhs} \leq \emprisk[k](\htildelambda \circ \theta) - \emprisk[k](\hhlambda \circ \theta),
    \end{equation*}
    which implies
    \begin{equation*}
        \langle -\widehat{w}, \hhlambda - \htildelambda \rangle_{\rkhs',\rkhs} \leq \emprisk[k](\htildelambda \circ \theta) - \emprisk[k](\hhlambda \circ \theta),
    \end{equation*}
    and
    \begin{equation*}
        \langle \widetilde{w}, \hhlambda - \htildelambda \rangle_{\rkhs',\rkhs} \leq \riskinter(\hhlambda \circ \theta) - \riskinter(\htildelambda \circ \theta).
    \end{equation*}
    Adding the two equations above yields
    \begin{equation*}
        \langle \widetilde{w} - \widehat{w}, \hhlambda - \htildelambda \rangle_{\rkhs',\rkhs} \leq \emprisk[k](\htildelambda \circ \theta)- \emprisk[k](\hhlambda \circ \theta) +
        \riskinter(\hhlambda \circ \theta) - \riskinter(\htildelambda \circ \theta).
    \end{equation*}
    Using the relations $\widetilde{w}= -2\lambda \iota  \htildelambda$ and $\widehat{w} =-2\lambda \iota \hhlambda $, we obtain
    \begin{equation*}
        2\lambda \langle \iota (\hhlambda - \htildelambda), \hhlambda - \htildelambda \rangle_{\rkhs',\rkhs} \leq \emprisk[k](\htildelambda \circ \theta)- \emprisk[k](\hhlambda \circ \theta) +
        \riskinter(\hhlambda \circ \theta) - \riskinter(\htildelambda \circ \theta).
    \end{equation*}
    By definition of $\iota$ and $\langle~ \cdot ~\rangle_{\rkhs',\rkhs}$, we have
    \begin{equation*}
        \langle \iota (\hhlambda - \htildelambda), \hhlambda - \htildelambda \rangle_{\rkhs',\rkhs}
        = \langle \hhlambda - \htildelambda, \hhlambda - \htildelambda \rangle_{\rkhs}
        = \lVert \hhlambda - \htildelambda \rVert_\rkhs^2.
    \end{equation*}
    Replacing in the previous inequality, we obtain the following bound:
    \begin{equation}\label{eq:distance htilde hh with risk}
        2\lambda \lVert \hhlambda - \htildelambda \rVert_\rkhs^2 \leq 
        \emprisk[k](\htildelambda\circ \theta)- \emprisk[k](\hhlambda \circ \theta) +
        \riskinter(\hhlambda \circ \theta) - \riskinter(\htildelambda \circ \theta).
    \end{equation}
    Now, using the definition of $\riskinter$ and $\emprisk[k]$,
    \begin{align*}
        &\big\lvert \emprisk[k](\htildelambda \circ \theta)- \emprisk[k](\hhlambda \circ \theta) +
        \riskinter(\hhlambda \circ \theta) - \riskinter(\htildelambda \circ \theta)\big\rvert \\
        &= \frac{1}{k} \Big\lvert \sum_{i=1}^n (\pinbl'(\htildelambda(\theta_{(i)}),Y_i) - \pinbl'(\hhlambda(\theta_{(i)}), Y_i) (\un_{\lVert X_{i} \rVert \geq \lVert X_{(k)} \rVert}  - \un_{\lVert X_i \rVert \geq \tnk})\Big\rvert\\
        &\leq \frac{\max(\tau,1-\tau) \lVert \htildelambda - \hhlambda\rVert_\infty}{k} \sum_{i=1}^n \big \lvert \un_{\lVert X_{i} \rVert \geq \lVert X_{(k)} \rVert}  - \un_{\lVert X_i \rVert \geq \tnk} \big \rvert,
    \end{align*}
    The sum can be rewritten as in \eqref{eq:comparison_quantile_statistic_order}. Moreover, since the kernel $K$ is bounded by $1$, we have 
    \begin{equation*}
        \lVert \htildelambda - \hhlambda\rVert_\infty \leq \normrkhs[\htildelambda - \hhlambda].
    \end{equation*}
    Replacing in \eqref{eq:distance htilde hh with risk}, we obtain
    \begin{equation*}
        2\lambda \lVert \hhlambda - \htildelambda \rVert_\rkhs^2 \leq 
        \frac{\max(\tau,1-\tau) \normrkhs[\hhlambda - \htildelambda]}{k} \Big\lvert\sum_{i=1}^n\un_{\lVert X_i \rVert \geq \tnk} -k \Big\rvert
    \end{equation*}
    and we obtain the result by dividing both sides by $2\lambda\normrkhs[\hhlambda - \htildelambda]$.
\end{proof}
\begin{proof}[Proof of Proposition \ref{prop:distance_htilde_hstar}]
Theorem 5.8 in \cite{christmann2008support} ensures that, for any probability measure $P$ on $\mathbb{R}^d \times \mathbb{R}$, the SVM solution
\begin{equation*}
    h_{\lambda,P} \in \argmin \risk_P(h) + \lambda \normrkhs[h]^2
\end{equation*}
can be written as an expectation involving a certain function $w$ and the feature map $\Phi$ of the RKHS $\rkhs$. More precisely, we have 
\begin{equation*}
    h_{\lambda,P} = - \frac{1}{2\lambda} \EE_P[w \Phi]
\end{equation*}
and the function $w$ satisfies $w(x,y) \in \partial \pinbl'(y,h_{\lambda,P}(x))$ for any  $(x,y) \in \mathbb{R}^d \times \mathbb{R}$.
This theorem covers also our setting because $h_{\lambda,t}$ minimizes
\[\EE[\pinbl(Y, h \circ \theta (X)) \given \lVert X \rVert \geq \tnk] + \lambda \normrkhs[h]^2\]
and we can take $P=P_{\tnk,\sphere}$ the joint distribution of $\left(\theta(X),Y \given \lVert X \rVert \geq \tnk \right)$. More precisely, Corollary 5.10 in  \cite{christmann2008support} ensures that, for any $(x, y) \in \R^d \times \R$, the function $\representer$ satisfies
\begin{equation*}
    \representer(\theta(x),y) \in \partial \pinbl'(y,h^*_{\lambda,\tnk}(\theta(x))).
\end{equation*}
Note that the non-negativity of the loss is not satisfied in our case. However, this is not an issue for quantile regression because the modified pinball loss is globally Lipschitz and the result from \cite{christmann2008support} can still be used.
The $\max(\tau,1-\tau)$-Lipschitzianity of the pinball loss also yields $\lVert \representer \rVert_\infty \leq \max(\tau,1-\tau)$. Moreover, for any distribution $P$ on $\R^d \times \R$, we obtain, from the propositions in \cite{christmann2008support},
\begin{equation} \label{eq:representer_with_proba_distrib}
    \normrkhs[h^*_{\lambda,\tnk} - h_{\lambda,P}^*] \leq \frac{1}{\lambda} \big\lVert \EE_{P_{\tnk,\sphere}}[w \Phi] - \EE_{P_{\sphere}}[w\Phi] \big\rVert_{\rkhs},
\end{equation}
where $P_\sphere$ is the distribution of $(\theta(X),Y)$ when $(X,Y) \sim P$. It turns out that this result can be extended to even discrete finite measures (not necessarily probability measures), with exactly the same proof as in \cite{christmann2008support} (see the end of the proof of Theorem 5.9 in particular). Taking $\mu = \frac{1}{k} \sum_{i=1}^n \delta_{(\theta(X_i),Y_i)} \un \{\lVert X_i \rVert \geq \tnk\}$ the measure associated to a sample $((X_1,Y_1) \dots, (X_n,Y_n))$, with associated SVM solution $\htildelambda$, \eqref{eq:representer_with_proba_distrib} extends in this case to
\begin{equation} \label{eq:representer_hstar_htilde}
    \normrkhs[h^*_{\lambda,t} - \htildelambda] 
    \leq \frac{1}{\lambda } \Big\lVert \EE_{P_{t,\sphere}}[w \Phi] - \frac{1}{k} \sum_{i=1}^n w(\theta(X_i),Y_i) \Phi(\theta(X_i)) \un\{\lVert X_i \rVert \geq \tnk\} \Big\rVert_{\rkhs}.
\end{equation}
An application of a Bernstein's type inequality adapted to Hilbert-spaces valued random variables allows to control the term on the right-hand side.
We will use the Bernstein's inequality stated below.
\begin{proposition}
    [Bernstein's inequality in Hilbert spaces, \cite{christmann2008support} Theorem 6.14] \label{Bernsteinhilbert}
Let $\rkhs$ be a separable RKHS, $B>0$ and $\sigma > 0$. Let $\xi_1,\dots,\xi_n$ be random variables taking values in $\rkhs$ and satisfying, for all $i \in \{1,\dots,n\}$, $\EE[\xi_i] = 0$, $\EE [\lVert \xi_i \rVert_\rkhs^2] \leq \sigma^2$ and such that, with probability $1$, $\normrkhs[\xi_i] \leq B$ for some $B>0$ . Then, with probability no less than $1-\delta$,
\begin{equation*}
    \normrkhs[\frac{1}{n}\sum_{i=1}^n \xi_i] 
    \leq \sqrt{\frac{2\sigma^2\log(1/\delta)}{n}} + \sqrt{\frac{\sigma^2}{n}} + \frac{2B\log(1/\delta)}{3n}.
\end{equation*}
\end{proposition}
\noindent Applying this high probability bound in our context leads to the following proposition:
\begin{proposition}\label{prop:bernstein+representer}
    Let $\tnk$ be the $1-\frac{k}{n}$-quantile of $\lVert X \rVert$. Suppose the kernel $K$ is bounded by $1$, and take $w \colon \R^d \times \R$ any bounded function. With probability no less than $1-\delta$, the following holds:
    \begin{align*}
        &\Big\lVert \frac{1}{k} \sum_{i=1}^n \representer \Phi(X_i,Y_i) \un_{\lVert X_i \rVert \geq \tnk} - \EE[ \representer \Phi(Y,\theta) \; \middle| \; \lVert X \rVert \geq \tnk] \Big\rVert_\rkhs\\
        &\leq  \lVert w \rVert_\infty \left(\sqrt{\frac{2 \log(1/\delta)}{k}} + \frac{1}{\sqrt{k}} + \frac{4\log(1/\delta)}{3k}\right).
    \end{align*}
\end{proposition}

\begin{proof}
    We apply Bernstein's inequality (Proposition \ref{Bernsteinhilbert}) with 
    \[
    \xi_i = (\representer \Phi)(Y_i,\theta(X_i)) \un_{\lVert X_i \rVert \geq \tnk} - \EE[ \representer \Phi(Y,\theta) \un_{\lVert X \rVert \geq \tnk}].
    \]
    Note first that these variables are $\rkhs$-valued. Indeed, for any $\omega \in \Omega$, we have
    \begin{equation*}
        (\representer \Phi)(Y_i(\omega), \theta(X_i(\omega)) \un_{\lVert X_i(\omega) \rVert \geq \tnk} = \representer(\theta(X_i(\omega),Y_i(\omega))) \un_{\lVert X_i(\omega) \rVert \geq \tnk}\Phi(\theta(X_i(\omega)))
    \end{equation*}
    which is an element of $\rkhs$ as product of a real number and an element of $\rkhs$. Moreover, for any $\omega \in \Omega$, 
    \begin{equation*}
        \normrkhs[\representer(\theta(X_i(\omega),Y_i(\omega))) \un_{\lVert X_i(\omega) \rVert \geq \tnk} \Phi(\theta(X_i(\omega)))] \leq \lVert \representer \rVert_\infty \normrkhs[\Phi(X_i(\omega)] \leq \lVert \representer \rVert_\infty,
    \end{equation*}
    where the last inequality holds because the kernel is bounded by $1$. Hence, the random variable $\representer \Phi(Y,\theta(X)) \un_{\lVert X \rVert \geq \tnk}$ defines a Bochner-$P$ integrable function, thus $\EE[ \representer \Phi(Y,\theta) \un_{\lVert X \rVert \geq \tnk}]\in \rkhs$. We have $\EE[\xi_i]=0$, and by the triangular inequality, 
    \begin{align*}
        \lVert \xi_i \rVert_\rkhs &\leq \lVert (\representer \Phi)(Y_i,\theta(X_i)) \un_{\lVert X_i \rVert \geq \tnk} \rVert_\rkhs + \lVert \EE[ \representer \Phi(Y,\theta) \un_{\lVert X \rVert \geq \tnk}] \rVert_\rkhs\\
        &\leq \lVert (\representer \Phi)(Y_i,\theta(X_i)) \un_{\lVert X_i \rVert \geq \tnk} \rVert_\rkhs + \EE[ \lVert (\representer \Phi)(Y_i,\theta(X_i)) \un_{\lVert X_i \rVert \geq \tnk} \rVert_\rkhs] \\
        &\leq 2 \lVert w \rVert_\infty,
    \end{align*}
    where we used the fact that, for any $\theta \in \sphere$, $\normrkhs[\Phi(\theta)] \leq 1$. Finally, we have,
    \begin{align*}
    \EE[\normrkhs[\xi_i]^2] 
    &\leq \EE\left[\normrkhs[\representer \Phi(Y \theta(X)) \un_{\lVert X \rVert \geq \tnk}]^2\right]\\ 
    &\leq \lVert \representer \rVert_\infty^2 \EE[\un_{\lVert X_i \rVert \geq \tnk}]\\
    &= \lVert \representer \rVert_\infty^2   \frac{k}{n}.
    \end{align*}
    We can now apply Theorem \ref{Bernsteinhilbert} which yields, with probability no less than $1-\delta$,
    \begin{equation}
        \lVert \frac{1}{n} \sum_{i=1}^n \xi_i \rVert_\rkhs \leq \frac{\sqrt{2 \log(1/\delta) k}}{n}\lVert \representer \rVert_\infty + \frac{\sqrt{k }}{n} \lVert \representer \rVert_\infty + \frac{4 \lVert \representer \rVert_\infty \log(1/\delta)}{3n}.
    \end{equation}
    Multiplying this inequality by $\frac{n}{k}$, we get the result.
\end{proof}
Combining the high probability bound obtained in Proposition \ref{prop:bernstein+representer} with equation \eqref{eq:representer_hstar_htilde}, and using the estimate $\lVert \representer \rVert_\infty \leq \max(\tau,1-\tau)$, we obtain the result stated in Proposition \ref{prop:distance_htilde_hstar}.
\end{proof}
\subsection{Proof of Theorem \ref{thm:variance_term_bound}}
\label{appendix:proof_variance_bound}
The first part of the proof consists in replacing the empirical risk $\emprisk_k$ by the intermediate risk $\riskinter$ and $\hhlambda$ by $\widetilde{h}_\lambda$ the minimizer of $\riskinter(h) + \lambda \normrkhs[h]^2$. We have
    \begin{align*}
        \riskPen_{\tnk}(\hhlambda) &= \lambda \normrkhs[\hhlambda]^2 + \risk_{\tnk}(\hhlambda) \\ 
        &= \empriskPen_k(\hhlambda) + \left(\risk_{\tnk}(\hhlambda) - \emprisk_k(\hhlambda)\right)\\
        &\leq \empriskPen_k(\widetilde{h}_\lambda) + 
        \left(\risk_{\tnk}(\hhlambda) - \emprisk_k(\hhlambda)\right)\\
        &= \riskinterPen_k(\widetilde{h}_\lambda)
        + \left(\risk_{\tnk}(\hhlambda) - \riskinter(\hhlambda)\right)
        + \underbrace{\left(\emprisk_k(\widetilde{h}_\lambda) - \riskinter(\widetilde{h}_\lambda)\right)-\left( \emprisk_k(\hhlambda) - \riskinter(\hhlambda) \right)}_{d_1}.
    \end{align*}
    The term $d_1$ has its absolute value controlled by
    \begin{align*}
        |d_1|
        &=\frac{1}{k} \Big\lvert \sum_{i=1}^n (\pinbl'(\hhlambda(\theta_{(i)}),Y_i) - \pinbl'(\widetilde{h}_\lambda(\theta_{(i)}), Y_i)) (\un_{\lVert X_{i} \rVert \geq \lVert X_{(k)} \rVert}  - \un_{\lVert X_i \rVert \geq \tnk})\Big\rvert \\
        &\leq \frac{\max(\tau,1-\tau) \lVert \hhlambda - \widetilde{h}_\lambda \rVert_\infty}{k} \sum_{i=1}^n \big\lvert \un_{\lVert X_{i} \rVert \geq \lVert X_{(k)} \rVert}  - \un_{\lVert X_i \rVert \geq \tnk} \big\rvert.
    \end{align*}
    From Lemma \ref{lem:distance_hh_htilde}, we have
    \begin{equation*}
    \lVert \hhlambda - \htildelambda \rVert_\infty 
    \leq \lVert \hhlambda - \htildelambda \rVert_\rkhs
    \leq \frac{\max(\tau,1-\tau)}{2\lambda k} \Big\lvert\sum_{i=1}^n\un_{\lVert X_i \rVert \geq \tnk} -k \Big\rvert.
    \end{equation*}
    and, the equality \eqref{eq:comparison_quantile_statistic_order} still holds, that is,
    \begin{equation*}
        \sum_{i=1}^n \lvert \un_{\lVert X_{i} \rVert \geq \lVert X_{(k)} \rVert}  - \un_{\lVert X_i \rVert \geq \tnk} \rvert = \Big\lvert\sum_{i=1}^n\un_{\lVert X_i \rVert \geq \tnk} -k \Big\rvert.
    \end{equation*}
    Replacing in the inequality for $\lvert d_1 \rvert$, we obtain
    \begin{equation} \label{eq:bound_d1}
        |d_1| \leq  \frac{\max(\tau,1-\tau)^2}{2 \lambda k^2} \Big\lvert\sum_{i=1}^n\un_{\lVert X_i \rVert \geq \tnk} -k \Big\rvert^2.
    \end{equation}
    Now that we have bounded $d_1$, we can proceed with the decomposition of the penalized risk and replace the intermediate SVM solution $\htildelambda$ by the true SVM solution $h^*_{\lambda,\tnk}$. Using the fact that $\htildelambda$ minimizes the penalized pseudo-empirical risk yields
    \begin{align*}
        & \riskinterPen_k(\widetilde{h}_\lambda)
        + \left(\risk_{\tnk}(\hhlambda) - \riskinter(\hhlambda)\right) \\
        & \leq \riskinterPen_k(h^*_{\lambda,\tnk}) 
        + \big(\risk_{\tnk}(\hhlambda) - \riskinter(\hhlambda)\big)\\
        &=\riskPen_{\tnk}(h^*_{\lambda,\tnk}) + \underbrace{\big(\riskinter(h^*_{\lambda,\tnk}) - \risk_{\tnk}(h^*_{\lambda,\tnk})\big) - \big( \riskinter(\hhlambda) - \risk_{\tnk}(\hhlambda)\big)}_{d_2}
    \end{align*}  
    We now bound with high probability both terms $d_1$ and $d_2$. Combining Lemma \ref{lem:distance_hh_htilde} and Proposition \ref{prop:distance_htilde_hstar} we obtain that, with probability no less than $1-3\delta$, we have simultaneously
    \begin{equation} \label{eq:high_proba_bound_estimate_quantile}
        \Big\lvert\sum_{i=1}^n\un_{\lVert X_i \rVert \geq \tnk} -k \Big\rvert \leq \sqrt{2k\log(1/\delta)} + \frac{\log(1/\delta)}{3}
    \end{equation}
    and
    \begin{equation}\label{eq:def_rho}
        \lVert \hhlambda - h^*_{\lambda,\tnk} \rVert_\rkhs \leq  
        \frac{\max(\tau,1-\tau)}{\lambda}\left(\frac{1}{\sqrt{k}} + 3\sqrt{\frac{\log(1/\delta)}{2k}} + \frac{3\log(1/\delta)}{2k}\right) = \frac{\max(\tau,1-\tau)}{\lambda \sqrt{k}}C(k,\delta) = \rho(\lambda,k,\delta).
    \end{equation}
    Combining \eqref{eq:bound_d1} with \eqref{eq:high_proba_bound_estimate_quantile}, we obtain that, on this high probability event,
    \begin{equation*}
        d_1 \leq \frac{\max(\tau,1-\tau)^2}{\lambda k} \Big(\sqrt{\frac{\log(1/\delta)}{2}} + \frac{\log(1/\delta)}{3\sqrt{k}}\Big)^2 = D_1(\lambda,k,\delta).
    \end{equation*}
    Leveraging the bound on the distance $\lVert \hhlambda - h^*_{\lambda,\tnk} \rVert_\rkhs$ provided by \eqref{eq:def_rho} also yields
    \begin{equation}\label{eq:C11}
        \begin{aligned}
            d_2 &= \big(\riskinter(h^*_{\lambda,\tnk}) - \risk_{\tnk}(h^*_{\lambda,\tnk})\big) - \big( \riskinter(\hhlambda) - \risk_{\tnk}(\hhlambda)\big)\\
        &\leq \sup_{h \in \ball_\rkhs(h^*_{\lambda,\tnk},\rho(\lambda,k,\delta))} \big\lvert \big(\riskinter(h^*_{\lambda,\tnk}) - \risk_{\tnk}(h^*_{\lambda,\tnk})\big) - \big( \riskinter(h) - \risk_{\tnk}(h)\big) \big\rvert.
        \end{aligned}
    \end{equation}
    It remains to provide a high probability bound for the right-hand side of the last inequality. Denote, for $h \in \ball_\rkhs(h^*_{\lambda,\tnk},\rho(\lambda,k,\delta))$,
    \begin{equation*}
        Z_h = \big(\pinbl'(Y_i,h^*_{\lambda,\tnk}(X_i)) - \pinbl'(Y_i,h(X_i))\big) \un \{\lVert X_i \rVert \geq \tnk\},
    \end{equation*}
 Equation~\eqref{eq:C11} is equivalent to
    \begin{equation*}
d_2\leq  \sup_{h \in \ball_\rkhs(h^*_{\lambda,\tnk},\rho(\lambda,k,\delta))} \Big\lvert \frac{1}{k}\sum_{i=1}^n \big(Z_h - \EE[Z_h]\big) \Big\rvert.
    \end{equation*}
    Define
    \begin{equation*}
        \psi_{n,k,\delta} = \EE\left[ \sup_{h \in \ball_\rkhs(h^*_{\lambda,\tnk},\rho(\lambda,k,\delta))} \Big\lvert \frac{1}{k}\sum_{i=1}^n \big(Z_h - \EE[Z_h]\big) \Big\rvert\right].
    \end{equation*}
    We have, for any $h \in \ball_\rkhs(h^*_{\lambda,\tnk},\rho(\lambda,k,\delta))$,
    \begin{equation*}
        \big \lvert Z_h - \EE[Z_h] \big \rvert \leq 2\max(\tau,1-\tau) \rho(\lambda,k,\delta)
    \end{equation*}
    and
    \begin{align*}
        \EE\left[(Z_h - \EE[Z_h])^2\right] &= 
        \EE\big((\pinbl'(Y_i,h^*_{\lambda,\tnk}(X_i)) - \pinbl'(Y_i,h(X_i)))^2 \un \{\lVert X_i \rVert \geq \tnk\}\big)\\
        &\leq \max(\tau,1-\tau)^2\rho(\lambda,k,\delta)^2 \EE[\un \{\lVert X_i \rVert \geq \tnk\}]\\
        &= \max(\tau,1-\tau)^2\rho(\lambda,k,\delta)^2 \frac{k}{n}.
    \end{align*}
    Note that the set of functions we are considering is $\ball_{\rkhs}(h^*_{\lambda,\tnk},\rho(\lambda,k,\delta))$, which is separable by separability of the RKHS $\rkhs$ (see for instance Lemma 4.33 in \cite{steinwart2011estimating} with the continuity of the kernel $k$) and the metric induced by the norm $\lVert \cdot \rVert_\rkhs$ dominates the pointwise convergence by the reproducing property of $\rkhs$. Hence, the set $(h^*_{\lambda,\tnk},\rho(\lambda,k,\delta))$ is pointwise measurable, and Talagrand's inequality may be applied. It follows that, from Talagrand's inequality (see Theorem 7.5 in \cite{christmann2008support} with $\gamma = \frac{1}{2}$), with probability $1-\delta$,
    \begin{equation*}
    \begin{aligned}
        \sup_{h \in \ball_\rkhs(h^*_{\lambda,\tnk},\rho(\lambda,k,\delta))} \Big\lvert \sum_{i=1}^n Z_h - \EE[Z_h] \Big\rvert 
        \leq \frac{3}{2} k \psi_{n,k,\delta}
        + \max(\tau,1-\tau) \rho(\lambda,k,\delta) \left(\sqrt{2k\log(1/\delta)} 
        + \frac{16 \log(1/\delta)}{3}\right)
    \end{aligned}
    \end{equation*}
    Multiplying both sides by $\frac{1}{k}$, we obtain that, on the event with probability $1-4\delta$ where the previous inequality, \eqref{eq:high_proba_bound_estimate_quantile}, and \eqref{eq:def_rho}  hold simultaneously, we have
    \begin{equation}\label{Talagrand_bound_with_psi}
        d_2 \leq \frac{3}{2}\psi_{n,k,\delta}
        + \max(\tau,1-\tau) \rho(\lambda,k,\delta) \left(\sqrt{\frac{2\log(1/\delta)} {k}}
        + \frac{16 \log(1/\delta)}{3k}\right).
    \end{equation}
We now bound the expectation $\psi_{n,k,\delta}$.
To this end, let $\sigma = (\sigma_1, \dots, \sigma_n)$ be a sequence of independent Rademacher variables. We introduce the Rademacher average
\begin{equation*}
\begin{aligned}
    \risk_k^{\sigma} &= \frac{1}{k} \EE_\sigma \Big[\sup_{\normrkhs[h - h^*_{\lambda,\tnk}] \leq \rho(\lambda,k,\delta)}\sum_{i=1}^n\sigma_i\big(\pinbl'(Y_i,h^*_{\lambda,\tnk}(X_i)) - \pinbl'(Y_i,h(X_i))\big) \un_{\lVert X_i \rVert \geq \tnk} \Big]\\
    &= \frac{1}{k} \EE_\sigma \Big[\sup_{\normrkhs[h - h^*_{\lambda,\tnk}] \leq \rho(\lambda,k,\delta)}\sum_{i=1}^n\sigma_i\big(\pinbl(Y_i,h^*_{\lambda,\tnk}(X_i)) - \pinbl(Y_i,h(X_i))\big) \un_{\lVert X_i \rVert \geq \tnk} \Big]
\end{aligned}
\end{equation*}
where we recall that $\pinbl$ is the non-modified pinball loss.
Following the classical symmetrization argument, we have
\begin{equation} \label{eq:symmetrization_bound_psi}
    \psi_{n,k,\delta} = \EE \Big[\sup_{\normrkhs[h - h^*_{\lambda,\tnk}] \leq \rho(\lambda,k,\delta)} \Big\lvert\big(\riskinter(h^*_{\lambda,\tnk}) - \risk_{\tnk}(h^*_{\lambda,\tnk})\big) - \big( \riskinter(h) - \risk_{\tnk}(h)\big)\Big\rvert \Big] \leq 2 \EE[\risk_k^{\sigma} ].
\end{equation} 
Now, by symmetry of the Rademacher variables  $\sigma_i$,
\begin{equation*}
    \EE_\sigma \Big[\sum_{i=1}^n \sigma_i\pinbl(Y_i,h^*_{\lambda,\tnk}(X_i)) \Big] = 0
\end{equation*}
thus
\begin{equation*}
    \risk_k^{\sigma} =  \frac{1}{k} \EE_\sigma \Big[\sup_{\normrkhs[h - h^*_{\lambda,\tnk}] \leq \rho(\lambda,k,\delta)} \sum_{i=1}^n \sigma_i \pinbl(Y_i,h(X_i))\big) \un_{\lVert X_i \rVert \geq \tnk } \Big]
\end{equation*}
Now, by the conditioning lemma from \cite{lhaut2022uniform} (see Lemma 2.1), we have the equality in distribution
\begin{equation*}
    \sum_{i=1}^n\sigma_i \pinbl(Y_i,h(X_i)) \un \{\lVert X_i \rVert \geq \tnk \} 
    \overset{d}{=} \sum_{i=1}^{\mathcal{K}} \sigma_i \pinbl(Y_i^k,h(X_i^k))
\end{equation*}
where $((X_1^k,Y_1^k),\dots,(X_n^k,Y_n^k))$ are independent variables with distribution $\law\left((X_i,Y_i) \given \lVert X_i \rVert \geq \tnk\right)$ and $\mathcal{K} \sim Bin(n,\frac{k}{n})$ is independent from the $X_i^k,Y_i^k,\sigma_i$. We obtain
\begin{equation} \label{thm_proof:conditioning_in_expectation}
    \begin{aligned}
    \EE [\risk_k^{\sigma} ] &= \frac{1}{k}\EE \Big[ \EE_\sigma \Big[\sup_{\normrkhs[h - h^*_{\lambda,\tnk}] \leq \rho(\lambda,k,\delta)} \sum_{i=1}^{\mathcal{K}} \sigma_i  \pinbl(Y_i^k,h(X_i^k)) \Big] \Big]\\
    &= \frac{1}{k}\EE \left[ \condK \EE \left[\EE_\sigma \left[ \frac{1}{\condK}\sup_{\normrkhs[h - h^*_{\lambda,\tnk}] \leq \rho(\lambda,k,\delta)} \sum_{i=1}^{\mathcal{K}} \sigma_i \pinbl(Y_i^k,h(X_i^k)) \given \mathcal{K}\right] \right]\right].
    \end{aligned}
\end{equation}
We now provide a bound for 
\begin{equation*}
    \EE_\sigma \left[ \frac{1}{\condK} \sup_{\normrkhs[h - h^*_{\lambda,\tnk}] \leq \rho(\lambda,k,\delta)} \sum_{i=1}^{\mathcal{K}} \sigma_i \pinbl(Y_i^k,h(X_i^k)) \given \mathcal{K}\right].
\end{equation*}
Now, notice that the pinball loss is distance based. More precisely, we have
\begin{equation*}
    \pinbl(Y_i^k,h(X_i^k)) = g_\tau(h(X_i^k) - Y_i^k)
\end{equation*}
where $g_\tau$ is defined by $g_\tau(x) = (1-\tau) \un_{x \geq 0} - \tau x \un_{x<0}$. $g_\tau$ is $\max(\tau,1-\tau)$ Lipschitz, hence, by the contraction lemma (see \cite{ledoux2013probability}, see also Theorem 5.7 in \cite{mohri2018foundations}),
\begin{align*}
    &\EE_\sigma \left[ \frac{1}{\condK} \sup_{\normrkhs[h - h^*_{\lambda,\tnk}] \leq \rho(\lambda,k,\delta)} \sum_{i=1}^{\mathcal{K}} \sigma_i \pinbl(Y_i^k,h(X_i^k)) \given \mathcal{K}\right]\\
    &\leq \max(\tau,1-\tau) \EE_\sigma \left[ \frac{1}{\condK} \sup_{\normrkhs[h - h^*_{\lambda,\tnk}] \leq \rho(\lambda,k,\delta)} \sum_{i=1}^{\mathcal{K}} \sigma_i (h(X_i^k) - Y_i^k) \given \mathcal{K}\right]\\
    &= \max(\tau,1-\tau) \EE_\sigma \left[ \frac{1}{\condK} \sup_{\normrkhs[h - h^*_{\lambda,\tnk}] \leq \rho(\lambda,k,\delta)} \sum_{i=1}^{\mathcal{K}} \sigma_i (h(X_i^k) - h^*_{\lambda,\tnk}(X_i^k))\given \mathcal{K}\right]\\
    &= \max(\tau,1-\tau) \EE_\sigma \left[\frac{1}{\condK}\sup_{h \in \ball_{\rkhs}(\rho(\lambda,k,\delta))} \sum_{i=1}^{\mathcal{K}} \sigma_i h(X_i^k) \given \mathcal{K} \right]
\end{align*}
where the second equality holds once again because of the symmetry of the Rademacher variables. Now, it holds from Proposition 6.12 in \cite{mohri2018foundations} that
\begin{equation*}
    \EE[\frac{1}{\condK}\sup_{h \in \ball_{\rkhs}(\rho(\lambda,k,\delta))} \sum_{i=1}^{\mathcal{K}} \sigma_i h(X_i^k) \given \mathcal{K}] \leq \frac{\rho(\lambda,k,\delta)}{\sqrt{\condK}}.
\end{equation*}
Thus, we obtain, from \ref{thm_proof:conditioning_in_expectation} and then applying Jensen's inequality,
\begin{align*}
    \EE[\risk_k^{\sigma}] &\leq \frac{\max(\tau,1-\tau)}{k} \rho(\lambda,k,\delta) \EE[\sqrt{\condK}] \\
    &\leq \frac{\max(\tau,1-\tau)}{k} \rho(\lambda,k,\delta) \sqrt{\EE[\condK]}\\
    &= \frac{\max(\tau,1-\tau)}{\sqrt{k}} \rho(\lambda,k,\delta).
\end{align*}
Finally, replacing the bound we obtain for the Rademacher average in \eqref{eq:symmetrization_bound_psi}, we get
\begin{equation}
    \psi_{n,k,\delta} \leq  \frac{2 \max(\tau,1-\tau)}{\sqrt{k}} \rho(\lambda,k,\delta).
\end{equation}
Combining this inequality with \eqref{Talagrand_bound_with_psi}, we obtain that, on an event with probability no less than $1- 4\delta$,

\begin{align*}
    d_2 &\leq \frac{\max(\tau,1-\tau) \rho(\lambda,k,\delta)}{\sqrt{k}} \left( 3 + \sqrt{\frac{2\log(1/\delta)} {k}}
        + \frac{16 \log(1/\delta)}{3k}\right)\\
        &= \frac{\max(\tau,1-\tau)^2}{\lambda k} C(k,\delta) \Big( 3 + \sqrt{\frac{2\log(1/\delta)} {k}}
        + \frac{16 \log(1/\delta)}{3k}\Big).
\end{align*}

\bibliographystyle{apalike}
\bibliography{biblio.bib}

\end{document}